\documentclass[MAL,biber]{nowfnt} 

\usepackage{amsmath,amsfonts,amssymb}
\usepackage{xcolor}
\usepackage{mathtools}
\usepackage{enumerate}
\usepackage{textcomp}
\usepackage{caption}
\usepackage{subcaption}
\usepackage{epigraph}
\usepackage{etoolbox}

\newtheorem{conjecture}{Conjecture}[section]

\def\environment{\mathcal{E}}

\def\proxy{\tilde{\environment}}
\def\target{\chi}

\def\actions{\mathcal{A}}

\def\observations{\mathcal{O}}
\def\histories{\mathcal{H}}

\def\states{\mathcal{S}}

\def\Regret{{\rm Regret}}
\def\KL{\mathbf{d}_{\mathrm{KL}}}

\def\E{\mathbb{E}}
\def\F{\mathbb{F}}
\def\H{\mathbb{H}}
\def\I{\mathbb{I}}
\def\Pr{\mathbb{P}}
\def\R{\mathbb{R}}
\def\1{\mathbf{1}}

\def\Var{\mathrm{Var}}

\DeclareMathOperator*{\argmax}{arg\,max}
\DeclareMathOperator*{\argmin}{arg\,min}

\newcommand{\Ic}{\mathcal{I}}
\newcommand{\Xc}{\mathcal{X}}
\newcommand{\Yc}{\mathcal{Y}}

\newcommand{\Qc}{\mathcal{Q}}
\newcommand{\Ec}{\mathcal{E}}

\newcommand{\Ac}{\mathcal{A}}

\newcommand{\Bc}{\mathcal{B}}

\usepackage{color}
\definecolor{darkred}{rgb}{0.7,0,0}
\definecolor{umblue}{rgb}{0.0,0.153,0.215}
\definecolor{ummaize}{rgb}{1.0,0.79,0.02}
\definecolor{darkblue}{rgb}{0.0, 0.0, 0.55}

\newcount\Comments
\newcommand{\kibitz}[2]{\ifnum\Comments=1{\textcolor{#1}{\textsf{\footnotesize #2}}}\fi}

\usepackage{caption}
\usepackage{changepage}

\newcommand{\bsuite}{\texttt{bsuite}}
\newcommand{\bsuiteversion}{\texttt{bsuite2019}}
\newcommand{\bsuitegithub}{\texttt{github.com/deepmind/bsuite}}

\newcommand{\bsuitetitle}[1]{

\section{\bsuite\ Report}

}

\newcommand{\bsuiteabstract}{
The \textit{Behaviour Suite for Core Reinforcement Learning}, or \bsuite\ for short, is a collection of carefully-designed experiments that investigate core capabilities of a reinforcement learning (RL) agent.
The aim of the \bsuite\ project is to collect clear, informative and scalable problems that capture key issues in the design of efficient and general learning algorithms and study agent behaviour through their performance on these shared benchmarks.
This report provides a snapshot of agent performance on \bsuiteversion, obtained by running the experiments from \bsuitegithub\ \citep{Osband2020Behaviour}.
}

\newcommand{\bsuiteradarplot}{figures/bsuite_radar}  
\newcommand{\bsuitebarplot}{figures/bsuite_bar}  

\title{Reinforcement Learning, Bit by Bit}

\booltrue{authortwocolumn} 
\maintitleauthorlist{
Xiuyuan Lu \\
DeepMind \\
lxlu@deepmind.com
\and
Benjamin Van Roy \\
DeepMind \\
benvanroy@deepmind.com
\and
Vikranth Dwaracherla \\
DeepMind \\
vikranthd@deepmind.com
\and
Morteza Ibrahimi \\
DeepMind \\
mibrahimi@deepmind.com
\and
Ian Osband \\
DeepMind \\
iosband@deepmind.com
\and
Zheng Wen \\
DeepMind \\
zhengwen@deepmind.com
}

\issuesetup
{%
 copyrightowner={A.~Heezemans and M.~Casey},
 volume        = xx,
 issue         = xx,
 pubyear       = 2018,
 isbn          = xxx-x-xxxxx-xxx-x,
 eisbn         = xxx-x-xxxxx-xxx-x,
 doi           = 10.1561/XXXXXXXXX,
 firstpage     = 1, 
 lastpage      = 18
 }

\usepackage{mwe}

\author[1]{Lu,Xiuyuan}
\author[2]{Van Roy,Benjamin}
\author[3]{Dwaracherla,Vikranth}
\author[4]{Ibrahimi,Morteza}
\author[5]{Osband,Ian}
\author[6]{Wen,Zheng}

\affil[1]{DeepMind; lxlu@deepmind.com}
\affil[2]{DeepMind; benvanroy@deepmind.com}
\affil[3]{DeepMind; vikranthd@deepmind.com}
\affil[4]{DeepMind; mibrahimi@deepmind.com}
\affil[5]{DeepMind; iosband@deepmind.com}
\affil[6]{DeepMind; zhengwen@deepmind.com}

\articledatabox{\nowfntstandardcitation}

\begin{document}

\makeabstracttitle

\begin{abstract}
Reinforcement learning agents have demonstrated remarkable achievements in simulated environments.  Data efficiency poses an impediment to carrying this success over to real environments.  The design of data-efficient agents calls for a deeper understanding of information acquisition and representation.  We discuss concepts and regret analysis that together offer principled guidance.  This line of thinking sheds light on questions of {\it what information to seek}, {\it how to seek that information}, and {\it what information to retain}.  To illustrate concepts, we design simple agents that build on them and present computational results that highlight data efficiency.
\end{abstract}

\chapter{Introduction}

{``\it Other learning paradigms are about minimization; reinforcement learning is about maximization.''}

The statement quoted above has been attributed to Harry Klopf, though it might only be accurate in sentiment. The statement may sound vacuous, since minimization can be converted to maximization simply via negation of an objective.  However, further reflection reveals a deeper observation.  Many learning algorithms aim to mimic observed patterns, minimizing differences between model and data.  Reinforcement learning is distinguished by its open-ended view.  A reinforcement learning agent learns to improve its behavior over time, without a prescription for eventual dynamics or the limits of performance.  If the objective takes nonnegative values, {\it minimization} suggests a well-defined desired outcome while {\it maximization} conjures pursuit of the unknown.  Indeed, \citeauthor{klopf1982} \citeyear{klopf1982} argued that, by focusing on {\it minimization} of deviations from a desired operating point, then-prevailing theories of homeostasis were too limiting to explain intelligence, while a theory centered around heterostasis {\it could} by allowing for {\it maximization} of open-ended objectives.

\section{Data Efficiency}

In reinforcement learning, the nature of data depends on the agent's behavior.  This bears important implications on the need for data efficiency.  In supervised and unsupervised learning, data is typically viewed as static or evolving slowly.  If data is abundant, as is the case in many modern application areas, the performance bottleneck often lies in model capacity and computational infrastructure.  This holds also when reinforcement learning is applied to simulated environments; while data generated in the course of learning does evolve, a slow rate can be maintained, in which case model capacity and computation remain bottlenecks, though data efficiency can be helpful in reducing simulation time.  On the other hand, in a real environment, data efficiency often becomes the gating factor.

Data efficiency depends on what information the agent seeks, how it seeks that information, and what it retains.  This tutorial offers a framework that can guide associated agent design decisions.  This framework is inspired in part by concepts from another field that has grappled with data efficiency.  In communication, the goal is typically to transmit data through a channel in a way that maximizes throughput, measured in bits per second.  In reinforcement learning, an agent interacts with an unknown environment with an aim to maximize reward.  An important difference that emerges is that bits of information serve as means to maximizing reward and not ends.  As such, an important factor arising in reinforcement learning concerns weighing costs and benefits of acquiring particular bits of information.  Despite this distinction, some concepts from communication can guide our thinking about information in reinforcement learning.

\section{Information Versus Computation}

Communication was a particularly active area of research at the turn of the twentieth century, with an emphasis on scaling up power generation to enable transmission of analog signals over increasing distances.  At the time, encoding and decoding was handled heuristically.  In the 1940s, following Shannon's maxim of ``information first, then computation,'' the focus shifted to understanding what is possible or impossible.  This initiative introduced the {\it bit}\footnote{Originally termed the {\it binary digit}, then the {\it binit}, before the {\it bit}.} as a unit of information and established fundamental limits of communication.  The maxim encouraged understanding possibilities and deferred the study of computation.  Design of encoding and decoding algorithms that attain fundamental limits arrived in the 1960s, with practical implementations emerging in the 1990s.  It is fair to say that this thread of research formed a cornerstone for today's connected world \citep
{jha_2016}.

Reinforcement learning seems to have followed an opposite maxim: ``computation first, then information.''  Beginning with heuristic evolution of algorithmic ideas such as temporal-difference learning \citep{Witten1976,witten1977,sutton1988learning,watkins1989learning} and actor-critic architectures \citep{witten1977,barto1983neuronlike,sutton1984temporal}, followed by demonstrated promise \citep{tesauro1992practical,tesauro1994td}, over the last decades of the twentieth century, much effort was directed toward computational methods, with little regard to data-efficiency \citep{bertsekas1996neuro,sutton2018reinforcement,bertsekas2019reinforcement}.  The past decade has experienced a great deal of further innovation, with an emphasis on scaling up computations and environments, leading to reinforcement learning agents that have produced impressive results in simulated environments and attracted enormous interest \citep{mnih2015human,schrittwieser2020mastering}.  However, data efficiency presents an impediment to the transfer of this success to real environments.  Unlike communication, {\it information} has not been the focus in these lines of research. Questions that are central to data efficiency, such as what information an agent should acquire and the cost of gathering that information, have mostly been ignored.

While much of the focus has been on developing heuristics and scaling up computation, there is a few notable exceptions.  The work of \citet{hutter2007universal} aims to design a ``universal'' agent, building on ideas such as Solomonoff's universal prior while putting aside any computational consideration. It remains unclear, though, how this line of thinking may offer a path towards designing practical, data-efficient agents. There is also a body of work that aims to address data efficiency and derive sample complexity bounds in stylized environments including bandits and Markov decision processes \citep{kearns2002near,brafman2003rmax,jaksch2010near,azar2017minimax,jiang2017contextual,jin2020provably}. However, methods considered in this line of work are not sufficiently scalable to address real, complex environments. The generality of our theoretical framework for thinking about information and data efficiency accommodates reasoning about scalable agent designs. This serves our ultimate goal of designing practical, data-efficient agents for real applications.


\section{Preview}

In this tutorial, we present a framework for studying costs and benefits associated with information.  As we will explain, this can guide how agents represent knowledge and how they seek and retain new information.  In particular, the framework sheds light on the questions of {\it what information to seek}, {\it how to seek that information}, and {\it what information to retain}.

We begin in Chapter \ref{se:agents} with a formalism for studying agents and environments.  We present a simplified version of the DQN agent \citep{mnih-atari-2013,mnih2015human} and an ensemble-DQN agent \citep{osband2016deep,osband2019deep} as examples. Then, in Chapter \ref{se:agent-design}, we discuss conceptual elements arising in the design of practical agents that can operate effectively in complex environments, with particular emphasis on informational considerations.  By interpreting the DQN and ensemble-DQN agents through this lens, we illustrate abstract concepts and highlight sources of inefficiency.  In Chapter \ref{se:cost-benefit}, we study a regret bound that applies to all agents and provides insight into design trade-offs.  We also illustrate insights offered by the bound when used to study particular classes of environments and agents.  As discussed in Chapters \ref{se:retaining-information} and \ref{se:seeking-information}, this bound can be used to think about how to design agents that seek and retain the right information.  In Chapter \ref{se:computation}, we present scalable agent designs.  Computational results reported in Chapter \ref{se:computation} serve to illustrate concepts covered in the tutorial and demonstrate their practical applicability.

\chapter{Environments and Agents}
\label{se:agents}

A reinforcement learning agent interacts with an environment $\environment$ through an interface of the sort illustrated in Figure \ref{fig:RLInterface}.  At each time $t$, the agent executes an action $A_t$, and the environment produces an observation $O_{t+1}$ in response.   The following coin tossing interface serves as an example.
\begin{example}
\label{ex:coin-tossing}
{\bf (coin tossing)} Consider an environment with $M$ possibly biased coins, with probabilities $p_1,\ldots,p_M$ of landing heads.  Each action $A_t \in \{1, \ldots, M\}$ selects and tosses a coin, and the resulting observation $O_{t+1} \in \{0,1\}$ indicates a heads or tails outcome, encoded as $0$ or $1$, respectively.
\end{example}

This coin tossing environment is particularly simple.  There are two possible observations and actions impose no delayed consequences.  In particular, the next observation depends only on the current action, regardless of previous actions.  Dialogue systems offer a context in which delayed consequences play a central role, calling for more sophisticated agent design.

\begin{figure}[htb]
\begin{center}
\includegraphics[scale=0.35]{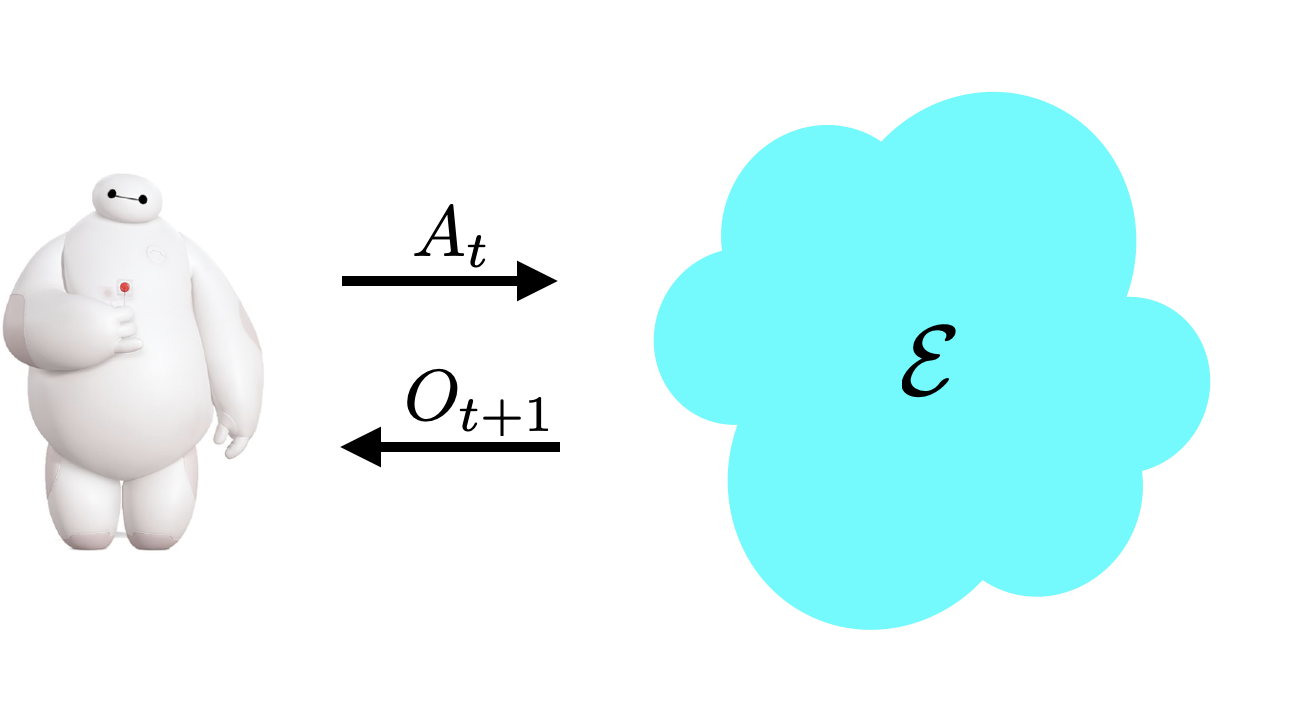}
\caption{The agent-environment interface.}
\label{fig:RLInterface}
\end{center}
\end{figure}

\begin{example}
{\bf (dialogue)}
\label{ex:dialogue}
Consider an agent that engages in sequential correspondence, with information conveyed via text messages exchanged with another party.  The agent first transmits a message $A_0$ and receives a response $O_1$.  Such exchanges continue.  For example, the first message $A_0$ could be ``How can I help you?'', with a response $O_1$ of ``How does one replace a light bulb?''.  Subsequent messages transmitted by the agent could seek clarification regarding the task at hand and guide the process.
\end{example}

Such an agent is designed to achieve goals while interacting with an environment.  In this section, we introduce a framework for modeling such interactions and framing goals.  We also describe a prototypical example of an agent.

\section{Agent-Environment Interface}

We consider a mathematical formulation in which the agent and environment interface through finite\footnote{The formulation and our results can be extended to accommodate infinite action and observation sets (under suitable measure-theoretic conditions), but the mathematics required would complicate our exposition.} action and observation sets $\actions$ and $\observations$.  Interactions make up a history $H_t = (A_0, O_1, A_1, O_2, \ldots, A_{t-1}, O_t)$.  The agent selects action $A_t$ after experiencing $H_t$.  Let $\histories$ denote the set of possible histories of any duration.

From the designer's perspective, an environment can be characterized by a tuple $\environment = (\actions, \observations, \rho)$, identified by a finite set of actions $\actions$, a finite set of observations $\observations$, and a function $\rho$ that for each history $h \in \histories$, action $a \in \actions$, and observation $o \in \observations$, prescribes an observation probability $\rho(o|h,a)$.  Though the internal workings of an environment can be arbitrarily complex, since the agent only interfaces through actions and observations, the designer can think of the environment as simply sampling observation $O_{t+1}$ from $\rho(\cdot|H_t, A_t)$.

From the environment's perspective, an agent samples each action $A_t$ from a probability mass function $\pi(\cdot|H_t)$ that depends on the history $H_t$.  We refer to such a function $\pi$ as a {\it policy}.  We will often use $\pi$ as a dummy variable and denote by $\pi_{\rm agent}$ the specific policy executed by the agent.  However, occasionally, when it is clear from context that we are referring to $\pi_{\rm agent}$, we will drop the subscript and simply use $\pi$ to refer to the agent's policy.

With the coin tossing environment of Example \ref{ex:coin-tossing}, $\actions = \{1,\ldots,M\}$, $\observations = \{0,1\}$, and $\rho(1 | H_t, A_t) = p_{A_t} = 1 - \rho(0 | H_t, A_t)$.  The fact that $\rho(\cdot | H_t, A_t)$ does not depend on $H_t$ indicates the absence of delayed consequences.  Note that $\rho$ identifies the coin biases.  As we will discuss in the next section, we will typically consider the designer to be uncertain about $\rho$ and thus the environment $\environment$.  In the coin tossing context, this motivates designing the agent to learn about coin biases through interactions and to leverage what is learned to improve performance over time.

A dialogue of the kind presented in Example \ref{ex:dialogue} can also be framed in these terms.  With a constrained number of tokens per text message, the set $\actions = \observations$ of possible text messages is finite.  The function $\rho$ assigns probabilities $\rho(\cdot|H_t,A_t)$ conditioned on the history of past messages, representing the manner in which previous exchanges influence what may come next.  For example, if $A_t$ is ``what is the wattage,'' then $\rho(o|H_t,A_t)$ ought to be relatively large for messages $o$ that communicate common wattage ratings.

It is worth noting that our formulation of agent-environment interactions is very general, involving a single stream of experience, without restrictive assumptions commonly made in the literature such as episodicity or that observations are of environment state.  While this formulation bears close resemblance to those studied by \citet{mccallum1995instance,hutter2007universal,daswani2013q,daswani2014feature}, such formulations have not been a focus of work on provably efficient reinforcement learning.   Our work extends regret analysis tools to this setting.

\section{Probabilistic Framework and Notation}

The reinforcement learning literature addresses uncertainty about observations and actions using the tools of probability theory.  However, traditional frameworks of reinforcement learning do not extend this use of probability to model uncertainty about the environment.  We will work with a more general framework that supports coherent reasoning about such uncertainty and how it is shaped by actions and observations.  In this section, we introduce this probabilistic framework and associated notation.  

We build on the foundations of probability, based on the Kolmogorov axioms, defining all random quantities with respect to a probability space $(\Omega, \F, \Pr)$.   Statements and arguments we present have precise meaning within the framework.  However, we often leave out measure-theoretic formalities for the sake of readability.  A mathematically-oriented reader ought to be able fill in these gaps.

The probability of an event $F \in \F$ is denoted by $\Pr(F)$.  For any events $F, G \in \F$ with $\Pr(G) > 0$, the probability of $F$ conditioned on $G$ is denoted by $\Pr(F | G)$.  When $Z$ takes values in $\Re^K$ and has a density $p_Z$ with respect to the Lebesgue measure, though $\Pr(Z=z)=0$ for all $z$, conditional probabilities $\Pr(F| Z=z)$ are well-defined and denoted by $\Pr(F | Z = z)$.  Consider a function $f(z) = \Pr(F | Z=z)$.  Given another random variable $Y$ with the same range as $Z$, we use the assignment symbol $\leftarrow$ to denote $f(Y)$ by $\Pr(F | Z \leftarrow Y)$.  Note that, in general, $\Pr(F | Z \leftarrow Y)$ differs from $\Pr(F | Z = Y)$.  The latter conditions on the event that $Z=Y$, while the former represents a change of measure for the variable $Z$.  Note that $\Pr(F | Z) = \Pr(F | Z \leftarrow Z)$.

For each possible realization $z$, the probability $\Pr(Z=z)$ that $Z = z$ is a function of $z$.  We denote the value of this function evaluated at $Z$ by $\Pr(Z)$.  Note that $\Pr(Z)$ is itself a random variable because it depends on $Z$.  For random variables $Y$ and $Z$ and possible realizations $y$ and $z$, the probability $\Pr(Y=y|Z=z)$ that $Y=y$ conditioned on $Z=z$ is a function of $(y, z)$.  Evaluating this function at $(Y,Z)$ yields a random variable, which we denote by $\Pr(Y|Z)$.

Particular random variables appear routinely throughout the paper.  One is the environment $\environment = (\actions, \observations, \rho)$.  While $\actions$ and $\observations$ are deterministic sets that define the agent-environment interface, the observation probability function $\rho$ is a random variable.  This randomness reflects the agent designer's epistemic uncertainty about the environment.  The probability measure $\Pr(\environment \in \cdot)$ can be thought of as assigning {\it prior probabilities} to sets of possible environments.  We often consider probabilities $\Pr(F|\environment)$ of events $F$ conditioned on the environment $\environment$.

For each policy $\pi$, random variables $A_0^\pi, O_1^\pi, A_1^\pi, O_2^\pi, \ldots$ denote the sequence of interactions generated by selecting actions according to $\pi$.  In particular, with $H_t^\pi = (A_0^\pi, O_1^\pi, \ldots, O_t^\pi)$ denoting the history of interactions through time $t$, we have $\Pr(A^\pi_t|H^\pi_t) = \pi(A^\pi_t|H^\pi_t)$ and $\Pr(O^\pi_{t+1}|H^\pi_t, A^\pi_t, \environment) = \rho(O^\pi_{t+1}|H^\pi_t,A^\pi_t)$ almost surely.  As shorthand, we generally suppress the superscript $\pi$ and instead indicate the policy through a subscript of $\Pr$.  For example,
$$\Pr_\pi(A_t|H_t) =  \Pr(A^\pi_t|H^\pi_t) = \pi(A^\pi_t|H^\pi_t),$$
and
$$\Pr_\pi(O_{t+1}|H_t, A_t, \environment) =  \Pr(O^\pi_{t+1}|H^\pi_t, A^\pi_t, \environment) = \rho(O^\pi_{t+1}|H^\pi_t, A^\pi_t).$$

When expressing expectations, we use the same subscripting notation as with probabilities.  For example, consider a reward function $r$ that maps history, action, and observation to a scalar reward $R^\pi_{t+1} = r(H^\pi_t, A^\pi_t, O^\pi_{t+1})$.  The expected reward is written as $\E[R^\pi_{t+1}] = \E_\pi[R_{t+1}]$, and the expected reward conditioned on the history $H_t$ and action $A_t$ is $\E[R^\pi_{t+1}|H^\pi_t, A^\pi_t] = \E_\pi[R_{t+1}|H_t, A_t]$.

\section{Rewards}

We consider the design of an agent to produce desirable outcomes.  The agent's preferences can be represented by a function $r:\histories \times\actions\times\observations\mapsto \Re$ that maps histories to rewards, incentivizing preferred outcomes.  After executing action $A_t$, the agent observes $O_{t+1}$ and enjoys reward 
$$R_{t+1} = r(H_t, A_t, O_{t+1}).$$
These rewards accumulate to produce, for any horizon $T$, a return $\sum_{t=0}^{T-1} R_{t+1}$.  Note that, unlike the treatment of \citet{sutton2018reinforcement}, we take reward to be a function of actions and observations rather than a designated signal provided by the environment at time $t+1$.  This does not rule out the possibility that a designated reward signal makes up part of the observation $O_{t+1}$ and that the reward function simply takes that to be $R_{t+1}$.

While the reward function expresses preferences for interactions over a single timestep, the agent's preferences may depend on its entire experience.  A natural way of assessing the desirability of a policy is through its {\it value} -- that is, the expected {\it return} -- over a long time horizon $T$:
$$\overline{V}_\pi = \E_\pi\left[\sum_{t=0}^{T-1} R_{t+1} \Big| \environment \right].$$
It is also useful to define notation for the optimal value: $\overline{V}_* = \sup_\pi \overline{V}_\pi$.

\section{Prototypical Examples} \label{se:dqn}

As we introduce abstract concepts to guide the design of data-efficient agents, it will be useful to anchor discussions around interpretation of simple examples.  For this purpose, we will consider simplified versions of the deep Q-network (DQN) agent \citep{mnih-atari-2013,mnih2015human} and the ensemble-DQN agent \citep{osband2016deep,osband2019deep}.  We will revisit these prototypical examples in later sections to illustrate abstract concepts.

The original DQN agent \citep{mnih-atari-2013,mnih2015human} was designed to interface with environments through an arcade game console.  In that context, each action $A_t$ represents a combination of joystick position and activation, and each observation $O_{t+1}$ is an image captured from the video display.  While the DQN agent of \citet{mnih-atari-2013,mnih2015human} was designed for episodic games, which occasionally end and restart, let us not assume any notion of termination and instead consider a perpetual stream of experience.  This could be generated through interaction with a never-ending game or through repeatedly playing a terminating game, continuing with a restart each time an episode ends.

At each time $t$, our simplified version of DQN maintains an action value function $Q_{\theta_t}$ using a neural network with weights $\theta_t$.  For each action $a \in \actions$, this function maps some number $M$ of recent observations, say $S_t = (O_{t-M+1}, \ldots, O_t)$, to a scalar value $Q_{\theta_t}(S_t, a)$.  While DQN can be applied with any reward function, to illustrate one possibility, the reward $R_{t+1}$ could indicate the change in game score as reflected by the score displayed in $O_{t+1}$ versus $O_t$, so long as $O_{t+1}$ does not indicate termination and restart.  Actions are selected via an $\epsilon$-greedy policy, meaning that the next action $A_t$ is sampled uniformly from $\mathcal{A}$, with probability $\epsilon$, and otherwise from among actions that maximize $Q_{\theta_t}(S_t, \cdot)$.  The value $\epsilon$ represents probability that the agent executes a random exploratory action, while $1-\epsilon$ is the probability with which the agent selects an action that maximizes its current value estimate.

Between observing $O_t$ and selecting $A_t$, the agent adjusts neural network weights, transitioning them from $\theta_{t-1}$ to $\theta_t$, via a training algorithm, details of which we will not cover here.  This training makes use of data cached in a {\it replay buffer}.  We consider a simple version, in which the replay buffer $B_t = (S_{t-N}, A_{t-N}, R_{t-N+1}, \ldots, S_{t-1}, A_{t-1}, R_t, S_t)$, for some fixed $N \gg M$, is made up of most recent neural network inputs, actions, and rewards.

Our simplified ensemble-DQN agent is similar to the based DQN agent we have described, except that it maintains an ensemble $Q_{\theta_{t, 1}}$, $\dots$, $Q_{\theta_{t, K}}$ of $K$ action value functions rather than a single point estimate. Each network in the ensemble is trained separately, using the same data but randomly perturbed to diversify the ensemble. Together, the ensemble represents the range of statistically plausible estimates of the optimal action value function. The agent selects actions in a manner inspired by Thompson sampling \citep{DBLP:conf_icml_Strens00,NIPS2013_6a5889bb}. Every $\tau$ timesteps, the agent samples an index $k_t$ uniformly from $\{1, \dots, K\}$ for use over the subsequent $\tau$ timesteps, that is, $k_t = k_{t+1} = \cdots = k_{t + \tau - 1}$. At timestep $t$, the ensemble-DQN agent selects an action that is greedy with respect to the $k_t^{\rm th}$ action value function in the ensemble, that is, $A_t \in \arg\max_{a\in\actions} Q_{\theta_{t, k_t}}(S_t, a)$.

\chapter{Elements of Agent Design}
\label{se:agent-design}

An agent is designed with respect to an environment interface, a reward function, uncertainty about the environment, and computational constraints.  In this chapter, we discuss these design considerations, with an emphasis on the role of information in balancing reward, uncertainty, and computation.  

We will introduce a notion of {\it agent state}, which represents data maintained by the agent in order to select actions.  In particular, the agent state evolves according to
$$X_{t+1} = f_{\rm agent}(X_t, A_t, O_{t+1}, U_{t+1}),$$
where $f_{\rm agent}$ is an agent state update function.  The update can be randomized, with $U_{t+1}$ representing an independent random draw sampled by the agent to allow for that.  This update function represents an important element of the agent design.  The agent's actions depend on the history $H_t$ only through the agent state $X_t$.  This allows the agent to operate within limits of memory rather than store and repeatedly process an ever growing history $H_t$.

An agent can be designed around any choice of agent state update function $f_{\rm agent}$.  In order to structure our thinking about the agent state, we will focus on representations comprising a triple
\[ X_t = \Big(\text{algorithmic state } Z_t, \text{situational state } S_t, \text{epistemic state } P_t\Big), \]
where
\begin{enumerate}
\item The {\bf algorithmic state} $Z_t$ is data cached by the agent that is not intended to represent information about the environment or past observations, but which the agent intends to use in its subsequent computations.
\item The {\bf situational state} $S_t$ represents the agent's summary of its current situation in the environment.
\item The {\bf epistemic state} $P_t$ represents the agent's current knowledge about the environment.
\end{enumerate}
The epistemic state encodes information about the environment extracted from history.  Because the agent must work with limited memory and per timestep computation, when interacting with a complex environment, it typically cannot seek or retain all relevant information.  Rather, the agent must prioritize, and we introduce two constructs that characterize this prioritization:
\begin{enumerate}
\setcounter{enumi}{3}
\item The {\bf environment proxy} $\proxy$ prioritizes information retained by the epistemic state.
\item The {\bf learning target} $\target$ prioritizes information sought by the agent for inclusion in epistemic state.
\end{enumerate}
In this section, we will elaborate on the nature and role of these five constructs, illustrating them by viewing the simplified DQN and ensemble-DQN agents described in Section \ref{se:dqn} through this lens.  Since the three components of agent state address three sources of uncertainty, we begin with a discussion of these sources.

\section{Sources of Uncertainty}

The agent should be designed to operate effectively in the face of uncertainty.  It is useful to distinguish three potential sources of uncertainty:
\begin{itemize}
\item {\bf Algorithmic uncertainty} may be introduced through computations carried out by the agent.  For example, the agent could apply a randomized algorithm to update parameters or select actions in a manner that depends on internally generated random numbers.
\item {\bf Aleatoric uncertainty} is associated with unpredictability of observations that persists even when $\rho$ is known.  In particular, given a history $h$ and action $a$, while  $\rho(\cdot|h,a)$ assigns probabilities to possible immediate observations, the realization is randomly drawn.
\item {\bf Epistemic uncertainty} is due to not knowing the environment -- this amounts to uncertainty about the observation probability function $\rho$, since the action and observation sets are inherent to the agent design.
\end{itemize}


To illustrate, consider design of an agent that interfaces with arcade games, as described in Section \ref{se:dqn}.  In that context, the agent can be applied to any arcade game that shares the prescribed interface.  Epistemic uncertainty arises from the fact that the designer does not know in advance the dynamics of the particular arcade game to which the agent will be applied.  Aleatoric uncertainty, on the other hand, is associated with random outcomes produced by the arcade game.  If the game were Tetris, for example, a source of aleatoric uncertainty is the random shape of each new tetromino.  Finally, algorithmic uncertainty arises in the application of a DQN agent, for example, through randomized selection of exploratory actions, and an ensemble-DQN agent through random draws of ensemble members to guide action selection.

In the parlance of \citet{anscombe1963definition}, aleatoric uncertainty is due to ``roullete lotteries,'' and can be thought of in terms of physical phenomena that generate outcomes for which probabilities can be determined empirically.  Epistemic uncertainty, on the other hand, is due to ``horse lotteries,'' and reflects the designer's beliefs.
Probabilities quantifying such uncertainties, together with Bayes' rule, represent a coherent logic for forming beliefs and making decisions.

\section{Agent State}

The agent state $X_t = (Z_t, S_t, P_t)$ encodes all data that the agent can use to select action $A_t$.  In particular, the agent's behavior can be expressed in terms of an {\bf agent policy} $\pi_{\rm agent}$, which selects each action $A_t$ according to probabilities $\pi_{\rm agent}(\cdot|X_t)$, so that
$$\Pr(A_t| X_t) = \pi_{\rm agent}(A_t | X_t).$$ 
Note that our original definition of {\it policy} indicated dependence on history.  However, with some abuse of notation, when a policy $\pi$ depends on history only through another statistic $\Psi_t$ we write, $\pi(A_t | \Psi_t) \equiv \pi(A_t | H_t)$.  Our notation $\pi_{\rm agent}(A_t|X_t)$ is an example of this usage.

\subsection{Algorithmic State}

The agent executes an algorithm in order to select actions.  The algorithm can be randomized, in which case a random draw $U_{t+1}$ is used in the computations carried out between observing $O_{t+1}$ and selecting $A_{t+1}$.  Some algorithms maintain an internal notion of state that is not intended to encode information about the environment or past observations, but rather, represent a combination of past computations and random draws.  We represent dynamics of the algorithm state using an {\it algorithmic state update function} $f_{\rm algo}$, with
\begin{equation}
\label{eq:algorithmic-state-dynamics}
Z_{t+1} = f_{\rm algo}(X_t, A_t, O_{t+1}, U_{t+1}).
\end{equation}

In the ensemble-DQN example from Section \ref{se:dqn}, we can view the algorithmic state as involving the latest random draw of a member of the ensemble. The algorithmic state $Z_{t+1}$ is updated periodically by sampling uniformly from the set of possible ensemble indices $1, \dots, K$, and $Z_{t+1} = Z_t$ between updates.

As another example of algorithmic state, let us imagine a modification of the DQN agent from Section \ref{se:dqn}.  Suppose that the agent wants to ``mix things up'' by increasing the probability of sampling exploratory actions that have not been selected recently.  In this case, the algorithm might maintain an algorithmic state in $\Re^{\actions}$, initialized with $Z_0 = \1$ and updated according to $Z_{t+1} = 0.9 Z_t + \1_{A_t}$, where $\1_{A_t}$ is a one-hot vector on $A_t$.  Given this algorithmic state, the agent could, as before, select an exploratory action with probability $\epsilon$, but now sample $A_t$ according to probabilities inversely proportional to $Z_{t, A_t}$, the $A_t^{\mathrm{th}}$ component of $Z_t$.  A large value of $Z_{t, a}$ indicates that the action $a$ has recently been executed, making it less likely to be sampled as the next exploratory action.


\subsection{Situational state}

If the environment $\environment = (\actions, \observations, \rho)$ is known, only algorithmic and aleatoric uncertainty remain, and the agent can select actions that depend on $\environment$ and $H_t$.  One might think of this in terms of a two-step process: identify a policy $\pi$ that depends on $\environment$ and generates desirable expected return $\E_\pi[\sum_{t=0}^{T-1} R_{t+1} | \environment]$ and then execute actions $A_t$ by sampling from $\pi(\cdot|H_t)$.  However, general dependence of rewards $R_{t+1} = r(H_t, A_t, O_{t+1})$ and action probabilities $\pi(A_t|H_t)$ on the history $H_t$ is problematic.  Even if compressed, its memory requirements grow unbounded, as does per-timestep computation, if the agent accesses the entire history to assess rewards or select actions.  This necessitates use of a bounded summary that can be updated incrementally, which we refer to as the {\it situational state}.

The design of any practical agent that can operate in a known complex environment over a long duration entails specification of situational state dynamics.
We denote by $\states$ the set of possible situational states.  Initialized with a distinguished element $S_0 \in \states$, the situational state is incrementally updated in response to actions, observations, and possibly, algorithmic randomness.  We express the dynamics using an {\it situational state update function} $f_{\rm situ}$, with state evolving according to
\begin{equation}
\label{eq:aleatoric-state-dynamics}
S_{t+1} = f_{\rm situ}(X_t, A_t, O_{t+1}, U_{t+1}).
\end{equation}
Recall that $U_{t+1}$ represents a random draw that that allows for algorithmic randomization.

In the event that the environment is known, the situational state serves as a summary of history for the purposes of reward assessment and action selection.  With this understanding, from here on we will treat rewards as functions of situational state instead of history, denoting realized reward by $R_{t+1} = r(S_t, A_t, O_{t+1})$.  Similarly, we will often consider policies that select actions based on situational state instead of history, in which case we denote action probabilities by $\pi(A_t|S_t)$.  It is important to note that such restriction can prevent the agent from achieving levels of performance that are possible with unbounded memory and computation.

\begin{figure}[htb]
\begin{center}
\includegraphics[scale=0.3]{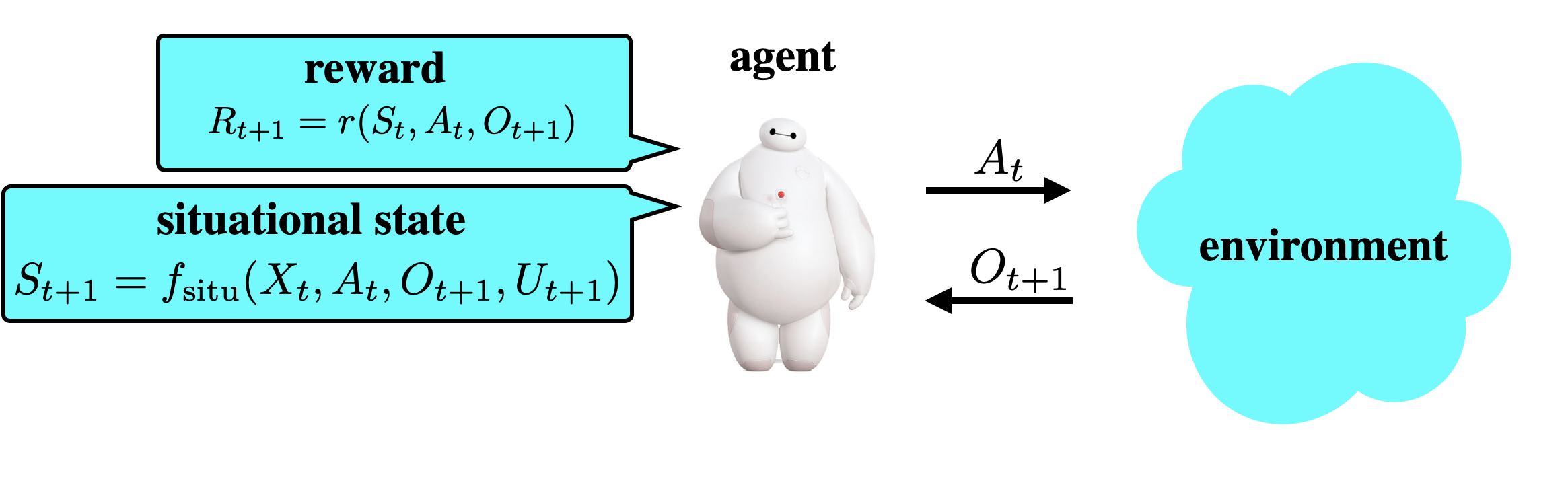}
\caption{The agent maintains a situational state and uses that to compute rewards.}
\label{fig:AleatoricState}
\end{center}
\end{figure}

For example, the situational state of the DQN and ensemble-DQN agents described in Section \ref{se:dqn} is made up of the $M$ most recent observations, $S_t = (O_{t-M+1}, \ldots, O_t)$.  While history grows unbounded, this situational state is bounded and can be updated incrementally by appending the most recent observation and ejecting the least recent one. Further, reward can be assessed from $(S_t, A_t, O_{t+1})$, and given weights $\theta_t$ for the action value function $Q_{\theta_t}$, the agent's action depends on history only through the situational state. As such, per-timestep computation does not grow with time.

\subsection{Epistemic State}

As the agent learns, it needs to represent the knowledge it retains about the unknown environment $\environment = (\actions, \observations, \rho)$.  In combination with the prior, the history $H_t$ serves as a comprehensive representation of this knowledge.  However, a practical agent must operate with bounded memory and per-timestep computation, and thus, cannot store and access an ever growing history.  Rather, the agent must use a bounded knowledge representation that can be updated incrementally, which we refer to as the agent's {\it epistemic state}. The epistemic state evolves as the agent interacts with the environment, at each time encoding knowledge retained by the agent.  The epistemic state is updated after each observation according to an {\it epistemic state update function}
\begin{equation}
\label{eq:epistemic-state-dynamics}
P_{t+1} = f_{\rm epis}(X_t, A_t, O_{t+1}, U_{t+1}).
\end{equation}

In the DQN example from Section \ref{se:dqn}, the weights $\theta_t$ of the action value function together with the replay buffer $B_t$ make up the epistemic state $P_t = (\theta_t, B_t)$.  When observations are registered, new information is incorporated by updating the epistemic state.  For example, the weights of the action value function may be revised by sampling a minibatch from the replay buffer and applying an associated temporal-difference gradient update, as done in \citet{mnih-atari-2013,mnih2015human}, and the replay buffer may be updated by ejecting the least recent and adding the most recent data. 
Similarly, in the ensemble-DQN example from Section \ref{se:dqn}, the epistemic state can be thought of as comprising of an ensemble of network weights together with the replay buffer, and the network weights maybe updated using techniques similar to temporal-difference learning \citep{osband2016deep,osband2019deep}, though randomly perturbed to diversify the ensemble.

Technically, the epistemic state can be any bounded object that is updated incrementally.  Some other examples include estimates of generative models, policies, general value functions, and belief distributions over aforementioned objects.

\section{Information}

We will quantify uncertainty, or equivalently, information, using the tools of information theory \citep{cover1991information}.  In this section, we define information-theoretic concepts and notation that we will use through the remainder of our exposition.

\subsection{Entropy}

A central concept is the entropy $\H(X)$, which quantifies uncertainty about a random variable $X$.  If $X$ takes on values in a countable set $\mathcal{X}$, this is defined by
$$\H(X) = - \sum_{x \in \mathcal{X}} \Pr(X = x) \log \Pr(X=x) = - \E[\log \Pr(X)],$$
with a convention that $0 \log 0 = 0$.  Entropy can alternatively be interpreted as the expected number of bits required to identify $X$, or the information content of $X$.  The realized conditional entropy $\H(X | Y = y)$ quantifies uncertainty remaining after observing $Y=y$.  If $Y$ takes on values in a countable set $\mathcal{Y}$ and $\Pr(Y=y)>0$,
\begin{align*}
\H(X | Y = y) &= - \sum_{x \in \mathcal{X}} \Pr(X = x | Y = y) \log \Pr(X=x | Y = y) \\
&= - \E[\log \Pr(X | Y) | Y = y].
\end{align*}
This can be viewed as a function $f(y)$ of $y$, and we write the random variable $f(Y)$ as $\H(X|Y \leftarrow Y)$.  The conditional entropy $\H(X | Y)$ is its expectation
$$\H(X | Y) = \sum_{y \in \mathcal{Y}} \Pr(Y = y) \H(X|Y=y) = \E[\H(X|Y \leftarrow Y)].$$
Note that $\H(X|X) = \H(X|X \leftarrow X) = 0$, since after observing $X$, all uncertainty about $X$ is resolved.  

\subsection{Mutual Information}

Mutual information quantifies information common to random variables $X$ and $Y$.  If $X$ and $Y$ take on values in countable sets then their mutual information is defined by
$$\I(X; Y) = \H(X) - \H(X|Y) = \H(Y) - \H(Y|X).$$
Note that $\H(X) = \I(X; X)$, since the information $X$ has in common with itself is exactly its information content.  Further, mutual information is always nonnegative, and if $X \perp Y$ then $\I(X; Y) = 0$.  If $Z$ is a random variable taking on values in a countable set $\mathcal{Z}$ and $\Pr(Z=z) > 0$, then the realized conditional mutual information $\I(X; Y| Z=z)$ quantifies remaining common information after observing $Z$, defined by
$$\I(X; Y| Z=z) = \H(X|Z=z) - \H(X|Y, Z=z).$$
The conditional mutual information $\I(X; Y | Z)$ is its expectation
$$\I(X; Y | Z) = \sum_{z \in \mathcal{Z}} \Pr(Z=z) \I(X; Y| Z=z) = \E[\I(X; Y | Z \leftarrow Z)].$$

\subsection{Continuous Variables}

We have defined entropy and mutual information, as well as their conditional counterparts, for discrete random variables.  The set of possible environments can be continuous, and to accommodate that, we generalize these definitions.

For random variables $X$ and $Y$ taking on values in (possibly uncountable) sets $\mathcal{X}$ and $\mathcal{Y}$, mutual information is defined by
$$\I(X; Y) = \sup_{f \in \mathcal{F}_{\text{finite}}, g \in \mathcal{G}_{\text{finite}}} \I(f(X); g(Y)),$$
where $\mathcal{F}_{\text{finite}}$ and $\mathcal{G}_{\text{finite}}$ are the sets of functions mapping $\mathcal{X}$ and $\mathcal{Y}$ to finite ranges.  Specializing to the case where $\mathcal{X}$ and $\mathcal{Y}$ are countable recovers the previous definition.  The generalized notion of entropy is then given by $\H(X) = \I(X; X)$.  Conditional counterparts to mutual information and entropy can be defined in a manner similar to the countable case.

It is worth noting that, when the set of possible environments is continuous, entropy is typically infinite -- that is, $\H(\environment) = \infty$.  However, as we will discuss in Section \ref{se:target-proxy}, an agent can restrict attention to learning a target $\target$ for which only a finite number $\I(\target; \environment)$ of bits must be acquired from the environment.

\subsection{Chain Rules and the Data-Processing Inequality}

The chain rule of entropy decomposes the entropy of a vector-valued random variable into component-wise conditional entropies, according to
$$\H(X_1, \ldots, X_N) = \H(X_1) + \H(X_2|X_1) + \cdots + \H(X_N | X_1,\ldots, X_{N-1}).$$
Mutual information also obeys a chain rule, which takes the form
$$\I(X; Y_1, \ldots, Y_N) = \I(X; Y_1) + \I(X; Y_2|Y_1) + \cdots + \I(X; Y_N | Y_1,\ldots, Y_{N-1}).$$

The data-processing inequality asserts that if $X$ and $Z$ are conditionally independent given $Y$, then $\I(X; Y) \geq \I(X; Z)$. As a special case, if $Z$ is a function of $Y$, then, by the data-processing inequality, $\I(X; Y) \geq \I(X; Z)$.  This relation is intuitive: $Z$ cannot provide more information about $X$ than does $Y$.  If $Z$ also determines $Y$, then both provide the same information, and $\I(X; Y) = \I(X; Z)$.

\subsection{KL-Divergence}

For any pair of probability measures $P$ and $P'$ defined with respect to the same $\sigma$-algebra, we denote KL-divergence by
$$\KL(P \| P') = \begin{cases}
\int P(dx) \log\frac{dP}{dP'}(x) & \text{ if} P \ll P', \\
\infty & \text{ otherwise.}
\end{cases}$$
Gibbs' inequality asserts that $\KL(P\|P') \geq 0$, with equality if and only if $P = P'$ almost everywhere.

Mutual information and KL-divergence are intimately related.  For any probability measure $P(\cdot) = \Pr((X,Y) \in \cdot)$ over a product space $\mathcal{X} \times \mathcal{Y}$ and probability measure $P'$ generated via a product of marginals $P'(dx \times dy) = P(dx) P(dy)$, mutual information can be written in terms of KL-divergence:
\begin{equation}
\label{eq:mutual-information-marginal-distribution}
\I(X; Y) = \KL(P \| P').
\end{equation}
Further, for any random variables $X$ and $Y$,
\begin{equation}
\label{eq:mutual-information-conditional-distribution}
\I(X; Y) = \E[\KL(\Pr(Y \in \cdot|X) \| \Pr(Y \in \cdot))].
\end{equation}
In other words, the mutual information between $X$ and $Y$ is the expected KL-divergence between the distribution of $Y$ with and without conditioning on $X$.  

Pinsker's inequality provides a lower bound on KL-divergence:
\begin{equation}
\label{eq:pinskers-inequality}
\sup_F |P(F) - P'(F)| \leq \sqrt{\frac{\ln 2}{2} \KL(P\|P')}.
\end{equation}
An immediate implication is that, if $P$ and $P'$ share support $[0,1]$, 
\begin{equation}
\int x P(dx) - \int x P'(dx) \leq \sqrt{\frac{\ln 2}{2} \KL(P\|P')}.
\end{equation}

\section{Learning and Prioritization}
\label{se:target-proxy}

We refer to uncertainty about the environment as {\it epistemic}.  Learning is the process of resolving epistemic uncertainty.  In this section, we present an approach to quantifying epistemic uncertainty and its reduction.  We then explain the need for prioritization of information that the agent ought to retain and information the agent ought to seek.  Environment proxies and learning targets are introduced as mechanisms for expressing prioritization.

\subsection{Quantifying Epistemic Uncertainty}

To model epistemic uncertainty, we treat the environment $\environment$ as a random variable.  The entropy $\H(\environment)$ of the environment quantifies the degree of epistemic uncertainty.  This represents the expected number of bits required to identify $\environment$.  If an agent digests all relevant information presented in its history, its remaining uncertainty at time $t$ is expressed by the conditional entropy $\H(\environment | H_t)$.  This is the expected number of bits still to be learned.  Conditioned on the specific realization $H_t$, the number becomes $\H(\environment | H_t \leftarrow H_t)$.

Learning reduces epistemic uncertainty.  The mutual information $\I(\environment; H_t)$ quantifies the extent to which observing the history $H_t$ is expected to reduce uncertainty. In the case of finite $\H(\environment)$ , the mutual information $\I(\environment; H_t) = \H(\environment) - \H(\environment | H_t)$. Note that this is the difference between the number $\H(\environment)$ of bits required to identify the environment and the expected number remaining at time $t$.

\subsection{The Curse of Knowledge}

The expected number of bits $\H(\environment)$ required to identify a complex environment is typically infinite or exceedingly large.  Consider, for example, the DQN agent described in Section \ref{se:dqn}.  If the designer knows in advance that the agent will interface with one of $K$ known arcade games, then $\environment$ is a random variable that takes on $K$ possible values and $\H(\environment) \leq \log K$.  However, the agent may alternatively engage with a complex range of environments. For example, an agent may control a robot that operates in any physical context. Then, the number of possible variations, and thus $\H(\environment)$, becomes intractable.  As $t$ grows, the same tends to be true for the expected number $\I(\environment; H_t)$ of these bits revealed by history.

Consider an epistemic state that encodes in $P_{t+1}$ all information revealed by $(P_t, S_t, A_t, O_{t+1})$ that is relevant to identifying the environment.  The expected number of bits incorporated is given by the mutual information $\I(\environment; S_t, A_t, O_{t+1} | P_t)$.  If the agent does this at every time step, $P_t$ retains all environment-relevant information presented by the history $H_t$, which means that $\I(\environment; P_t) = \I(\environment; H_t)$ for all $t$.  As such, the expected number of bits $\H(P_t) \geq \I(\environment; P_t) = \I(\environment; H_t)$ encoded in $P_t$ can grow too large to be retained by a bounded agent.

\subsection{Environment Proxies}

Since a bounded agent cannot generally retain all environment-relevant information, it must prioritize.  The choice of information retained by an agent is a critical design decision.  One possibility is to design the epistemic state to prioritize knowledge about an {\it environment proxy} $\proxy$.  By allowing the epistemic state to discard bits that are not proxy-relevant, the agent can dramatically reduce memory and computational requirements.

\begin{figure}[htb]
\begin{center}
\includegraphics[scale=0.25]{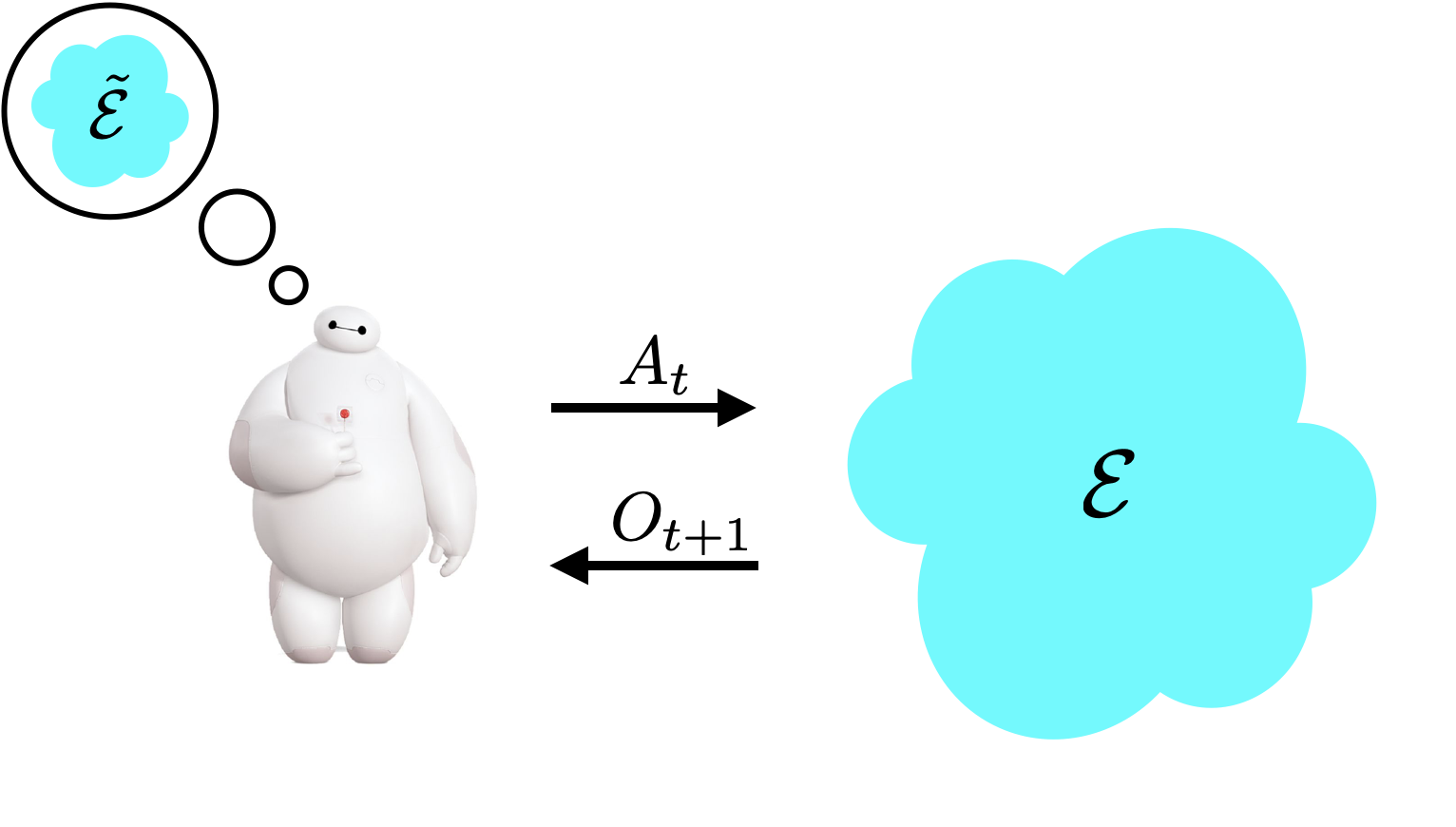}
\caption{The agent aims to retain knowledge about the environment proxy $\proxy$.}
\label{fig:proxy}
\end{center}
\end{figure}

With the DQN and ensemble-DQN agents described in Section \ref{se:dqn}, for example, a desired action value function $\tilde{Q} = Q_{\tilde{\theta}}$, generated by some neural network weights $\tilde{\theta}$, can be thought of as the environment proxy. The epistemic state encodes learned weights and the data buffer. 
Though observed video images can reveal enormous amounts of environment-relevant information, the proxy leads the agents to prioritize retaining information about $\tilde{Q}$ while ignoring other information revealed by observations.

Technically, a proxy could be any random variable $\proxy$.  However, a practical proxy should be designed to encode essential features of the environment $\environment$ with a manageable number $\H(\proxy) \ll \H(\environment)$ of bits.  Further, an effective proxy design should expedite accumulation of useful information.  Aside from action value functions, as used by DQN and ensemble-DQN, other proxy designs commonly used in reinforcement learning include general value functions, simplified models of the environment, and policies.  We will further discuss in Section \ref{se:retaining-information} factors that drive design of an effective proxy.

\subsection{The Curse of Curiosity}

To identify the environment, an agent must uncover $\H(\environment)$ bits.  An agent that indiscriminately gathers all this information is termed {\it curiosity-driven}.  Since $\H(\environment)$ is typically intractably large, or even infinite, a curiosity-driven agent may spend its lifetime gathering irrelevant information.  The potentially catastrophic consequences are crystallized by the {\it noisy TV} study \citep{pathak18largescale}, which demonstrates how irrelevant yet complex patterns can draw the attention of a curious agent and dramatically impede its acquisition of useful information.


As aptly noted by \citet{howard1966information}, not all information an agent can acquire is equally valuable: ``if losing all your assets in the stock market and having a whale steak for supper have the same probability, then the information associated with the occurrence of either event is the same.''  We need a mechanism for prioritizing information the agent can acquire.  One approach is to prioritize proxy-relevant information.  While identifying the proxy requires fewer bits than the environment, since $\H(\proxy) \ll \H(\environment)$, the number can still be too large. Moreover, this tends to be wasteful as much of the information required to identify the proxy may be irrelevant to achieving high expected return.  To understand this, consider using a simplified model of the environment as a proxy and suppose that for all possible realizations, a large fraction of situational states yield large negative rewards.  Then, the agent should avoid those states rather than making sacrifices to learn about their associated dynamics.  A new concept beyond that of a proxy is needed to prioritize information seeking.

\subsection{Learning Targets} \label{se:learning-target}

We will consider designing an agent that seeks knowledge about an alternative object, which we refer to as the {\it learning target}. The learning target $\target$ is a function of the environment proxy $\proxy$. As such, the number of environment-relevant bits encoded by $\target$ is no larger than that encoded by $\proxy$.
Knowledge retained about the proxy serves to inform estimates of the learning target.  The learning target should be designed to simultaneously limit two quantities:
\begin{itemize}
\item {\it information:} Environment-relevant information encoded in $\target$, measured by the mutual information $\I(\target; \environment)$, should make up a modest number of bits.
\item {\it regret:} 
Given $\target$, the agent should be able to execute a target policy $\pi_\target$ that incurs modest regret $\E[\overline{V}_* - \overline{V}_{\pi_\target}]$.
\end{itemize}
These requirements are intuitive -- in a complex environment, the agent should prioritize acquiring a modest amount of information that can be used to produce an effective policy.

\begin{figure}[htb]
\begin{center}
\includegraphics[scale=0.25]{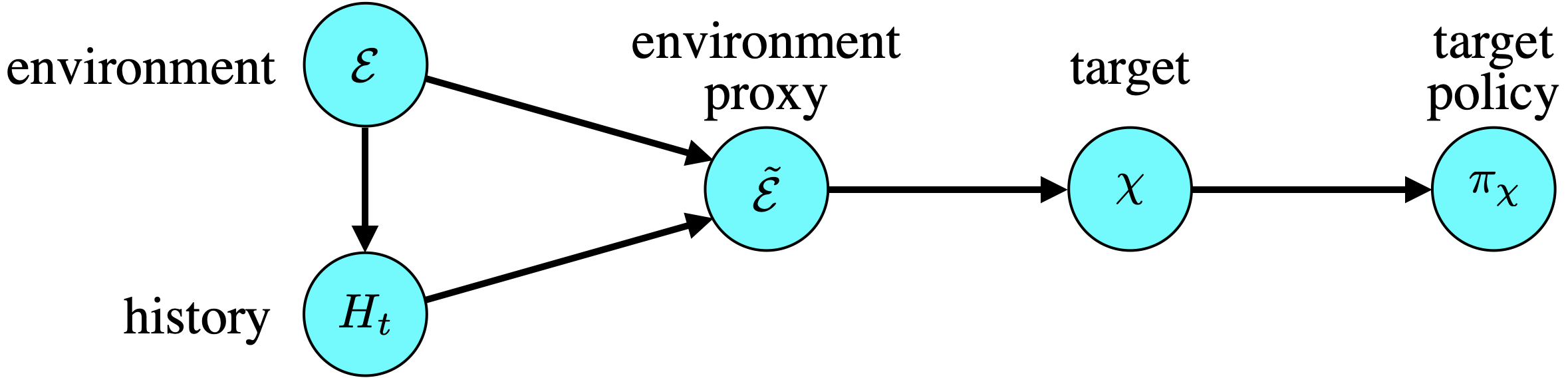}
\caption{A Bayesian network illustrating dependencies between the environment, history, proxy, target, and target policy.}
\label{fig:target-notation}
\end{center}
\end{figure}

It is worth noting that environment proxies and learning targets are abstract concepts that can guide agent design even if they do not explicitly appear in an agent's algorithm.  They reflect a designer's intent with regards to what information the agent should retain and what information the agent should seek.  For the DQN and ensemble-DQN agents described in Section \ref{se:dqn}, it is natural to think of the proxy as an action value function.  However, it is less clear what learning target, if any, motivated the design.  One possibility is the greedy policy with respect to the action value function.  The fact that these agents usually select a greedy action with respect to an action value network might be motivated by this target.


Action value functions and general value functions \citep{sutton2011horde} have served as proxies for reinforcement learning agents, and they can simultaneously serve as learning targets, with target policies taken to be their respective greedy policies.  Alternatively, a designer could take the greedy policies themselves to simultaneously serve as learning targets and target policies.  As another example, with a simplified model of the environment serving as a proxy, the learning target could be a policy generated by a planning algorithm. MuZero \citep{schrittwieser2020mastering}, for example, operates in this manner, with Monte Carlo tree search used for planning.  Our computational studies presented in Section \ref{se:computation} will illustrate more concretely a few specific choices of proxies, learning targets, and target policies.

\chapter{Cost-Benefit Analysis}
\label{se:cost-benefit}

We have highlighted a number of design decisions.  These determine the components of agent state, the environment proxy, the learning target, and how actions are selected to balance between exploiting current knowledge and acquiring new information.  Choices are constrained by memory and per-timestep computation, and they influence expected return in complex ways.  In this chapter, we formalize the design problem and establish a regret bound that can facilitate cost-benefit analysis.

\begin{figure}[htb]
\begin{center}
\includegraphics[scale=0.4]{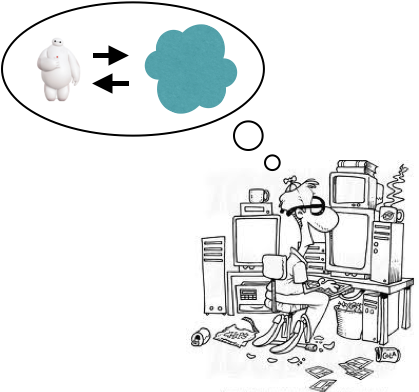}
\caption{Agent design.}
\label{fig:AgentDesign}
\end{center}
\end{figure}

\section{Agent Policy and Regret}

The design problem entails specifying the agent policy $\pi_{\rm agent}$, which at each time produces agent-state-contingent action probabilities $\pi_{\rm agent}(\cdot |X_t)$ from which action $A_t$ will be sampled.  The objective is to maximize the expected return $\E\left[\overline{V}_{\pi_{\rm agent}}\right]$
subject to memory and per-timestep computation constraints.  This expectation is with respect to all uncertainty, while
$$\overline{V}_{\pi_{\rm agent}} = \E\left[\sum_{t=0}^{T-1} R_{t+1} \Big| \environment \right]$$
represents an expectation with respect to aleatoric and algorithmic but not epistemic uncertainty.  It is worth noting that the agent policy $\pi_{\rm agent}$ generally differs from the target policy $\pi_\target$.  The former is the policy executed in the process of learning the latter.

This characterization of agent design in terms of maximizing the expected value $\E\left[\overline{V}_{\pi_{\rm agent}}\right]$ is not new to the reinforcement learning literature.  For example, \citet{duff2003dissertation} considered this characterization and developed computational methods that aim to approximately solve this problem.  However, it is not clear to what extent these methods are scalable and reliable.  Our emphasis is distinct in aiming to offer a general {\it way of thinking} about agents and their relation to this objective.  The intent is to provide a framework through which one can reason about data efficiency of agents that are practical and scalable.

It is useful to define history transition matrices.  For all $h, h' \in \histories$ and $a \in \actions$, let
$$P_{ahh'} = \left\{\begin{array}{ll}
\rho(o|h,a) \qquad & \text{if } h' = (h,a,o) \\
0 \qquad & \text{otherwise.}
\end{array}\right.$$
Here, $(h,a,o)$ denotes the history obtained by concatenating action $a$ and observation $o$ to the history $h$.  We also define, for each policy $\pi$, transition probabilities
$$P_{\pi h h'} = \sum_{a \in \actions} \pi(a|h) P_{a h h'}.$$
We consider $P_a$ and $P_\pi$ to be stochastic matrices.
Further, for each $h \in \histories$ and $a \in \actions$, we define a mean reward 
$$\overline{r}_{a h} = \sum_{o \in \observations} \rho(o|h,a) r(h,a,o),$$
and for each policy $\pi$, $\overline{r}_{\pi h} = \sum_{a \in \actions} \pi(a|h) \overline{r}_{ah}$.  We view $\overline{r}_a$ and $\overline{r}_\pi$ as infinite-dimensional vectors.

The value function of a policy $\pi$ is defined by
$$V_\pi(h) = \sum_{t=|h|}^{T-1} (P_\pi^{t-|h|} \overline{r}_\pi) (h),$$
where for $h = (a_0,o_1,\ldots,a_{t-1},o_t)$, $|h|=t$ is the duration of $h$, and $P_\pi^{t-|h|} \overline{r}_\pi$ is interpreted as an infinite-dimensional matrix-vector product.
The action value function is defined by
$$Q_\pi(h,a) = \overline{r}_{ah} + (P_a V_\pi)(h).$$
The value $V_\pi(h)$ represents the future expected return beginning at history $h$ and executing $\pi$ thereafter, while $Q_\pi(h,a)$ represents the future expected return if action $a$ is executed at $h$ and $\pi$ selects actions thereafter.  Note that $\overline{V}_\pi = V_\pi(H_0)$.  We define the optimal value function by $V_*(h) = \sup_\pi V_\pi(h)$ and the optimal action value function by $Q_*(h,a) = \sup_\pi Q_\pi(h,a)$, and we assume that they are finite.

Klopf's sentiment notwithstanding, maximizing expected return is equivalent to minimizing regret
$$\Regret(T|\pi) = \E[\overline{V}_* - \overline{V}_\pi].$$
Though this frames the problem literally as one of {\it minimization}, it retains the spirit of {\it maximization} in the sense of pursuing open-ended goals, as $\overline{V}_*$ is unknown.  This term is often referred to as {\it Bayesian regret}, though we will omit the {\it Bayesian} designation as we will not be using any other versions of regret.  It is the shortfall in expected return relative to an optimal policy.  We treat regret, rather than expected return, as the design objective.  The advantage is that per-step regret decreases as the agent learns, making bounds on regret easier to interpret than bounds on expected return.

It is worth noting that, by definition, $\Regret(T|\pi)$ is a deterministic quantity.  This is true even if the policy $\pi$ is random.  For example, the target policy $\pi_\target$ is random, but $\Regret(T|\pi_\target)$ is not itself a function of $\pi_\target$, but rather, the expectation of $\overline{V}_* - \overline{V}_{\pi_\target}$, which integrates over $\pi_\target$.

A more general notion is the regret $\E[\overline{V}_{\overline{\pi}} - \overline{V}_\pi]$ relative to a baseline policy $\overline{\pi}$.  The following result decomposes regret across time, offering a useful interpretation of regret as a sum of terms, each of which represents the shortfall $V_{\overline{\pi}}(H_t) - Q_{\overline{\pi}}(H_t, A_t)$ due to executing policy $\pi$ instead of $\overline{\pi}$.  While similar results have appeared in the literature over many decades, such as in \citet{Kakade+Langford:2002}, and an elegant version is presented in \citet{sutton2018reinforcement}, we provide a proof for completeness. 
\begin{theorem}
\label{th:shortfall}
{\bf (shortfall decomposition)}
For all policies $\overline{\pi}$ and $\pi$,
$$\overline{V}_{\overline{\pi}} - \overline{V}_\pi 
= \E_\pi\left[\sum_{t=0}^{T-1} \left(V_{\overline{\pi}}(H_t) - Q_{\overline{\pi}}(H_t, A_t)\right) \Big| \environment\right].$$
\end{theorem}
\begin{proof}
Since $V_{\overline{\pi}}(H_0) = \overline{V}_{\overline{\pi}}$ and $V_{\overline{\pi}}(H_T) = 0$, we have
\begin{align*}
\overline{V}_{\overline{\pi}} - \overline{V}_\pi 
=& \E_\pi\left[V_{\overline{\pi}}(H_0) - \sum_{t=0}^{T-1} R_{t+1} \Big| \environment\right] \\
=& \E_\pi\left[V_{\overline{\pi}}(H_0) + \sum_{t=0}^{T-1} V_{\overline{\pi}}(H_{t+1}) - \sum_{t=0}^{T-1} (R_{t+1} + V_{\overline{\pi}}(H_{t+1}))  \Big| \environment\right] \\
=& \E_\pi\left[\sum_{t=0}^{T-1} (V_{\overline{\pi}}(H_t) - (R_{t+1} + V_{\overline{\pi}}(H_{t+1})) \Big| \environment\right] \\
=& \E_\pi\left[\sum_{t=0}^{T-1} (V_{\overline{\pi}}(H_t) - \E_\pi[R_{t+1} + V_{\overline{\pi}}(H_{t+1}) | \environment, H_t, A_t]) \Big| \environment\right] \\
=& \E_\pi\left[\sum_{t=0}^{T-1} (V_{\overline{\pi}}(H_t) - Q_{\overline{\pi}}(H_t, A_t)) \Big| \environment\right].
\end{align*}
\end{proof}
An immediate corollary bounds shortfall relative to maximal expected return.
\begin{corollary}
\label{co:shortfall}
For all policies $\pi$,
$$\overline{V}_* - \overline{V}_\pi 
= \E_\pi\left[\sum_{t=0}^{T-1} \left(V_*(H_t) - Q_*(H_t, A_t)\right) \Big| \environment\right].$$
\end{corollary}

\section{Information Gain}

The expected shortfall $\E[V_*(H_t) - Q_*(H_t, A_t)]$ represents a cost, which may be deliberately incurred by an agent as it seeks information.  To reason about benefits that offset this cost, we will devise a measure of information gain.

The conditional mutual information $\I(\target; S_t, A_t, O_{t+1}|P_t)$ offers one notion of information gain.  This value quantifies information about the learning target $\target$ that is revealed by $(S_t, A_t, O_{t+1})$ and absent from the epistemic state $P_t$.  However, an agent may be motivated by delayed rather than immediate information.

It can be in an agent's interest to incur a shortfall in order to position itself to acquire information at a future time.  For example, an effective agent may incur a large shortfall $\E[V_*(H_t) - Q_*(H_t, A_t)]$ at time $t$ that reveals no immediate information -- that is, $\I(\target; S_t, A_t, O_{t+1}|P_t) = 0$ -- if this subsequently leads over, say, $\tau$ timesteps to substantial information $\I(\target; S_t, H_{t:t+\tau}|P_t)$, where $H_{t:t+\tau} = (A_t, O_{t+1}, A_{t+1}, \ldots, A_{t+\tau-1}, O_{t+\tau})$.
The act of sacrificing immediate reward for delayed information is sometimes referred to as {\it deep exploration} \citep{osband2019deep}.

As we will discuss further in Section \ref{se:retaining-information}, the epistemic state $P_{t+\tau}$ does not necessarily retain all information about the learning target $\target$.  As such, instead of the mutual information $\I(\target; S_t, H_{t:t+\tau}|P_t)$, which quantifies new information revealed, we will measure information gain in terms of new information retained. In particular, we measure the decrease $\I(\target; \environment|P_t) - \I(\target; \environment|P_{t+\tau})$ in the conditional mutual information.  This quantifies the increase in environment-relevant information retained about the learning target as the epistemic state transitions from $P_t$ to $P_{t+\tau}$.  In the event that the environment $\environment$ determines the proxy $\proxy$, which in turn determines the learning target $\target$, all information about $\target$ is environment-relevant, and our expression of information gain simplifies, with
$$\I(\target; \environment|P_t) - \I(\target; \environment|P_{t+\tau}) = \H(\target|P_t) - \H(\target|P_{t+\tau}).$$
The use of mutual information $\I(\target; \environment|P_t)$ generalizes entropy $\H(\target|P_t)$, allowing $\target$ to be influenced by algorithmic randomness without counting that as part of the information gain.  Information gain cannot exceed information about the target revealed by observations, or, equivalently, 
\begin{align*}
\I(\target; \environment|P_t) - \I(\target; \environment|P_{t+\tau}) &\leq \I(\target; S_t, H_{t:t+\tau}|P_t) - \I(\target; S_t, H_{t:t+\tau} | P_t, \environment) \\
&\leq \I(\target; S_t, H_{t:t+\tau}|P_t).
\end{align*}
That is because, at best, the agent can reduce its uncertainty about the learning target $\target$ by retaining all $\I(\target; S_t, H_{t:t+\tau}|P_t)$ bits of new information.

\section{The Information Ratio}
\label{se:information_ratio}

As opposed to communication, where efficiency is typically framed in terms of maximizing information throughput, an important consideration in reinforcement learning is how the agent trades off between exploiting current knowledge and acquiring new information.  This can be viewed as a balance between expected immediate shortfall $\E[V_*(H_t) - Q_*(H_t, A_t)]$ and incremental information $\I(\target; \environment|P_t) - \I(\target; \environment|P_{t+\tau})$.

We will assume that uncertainty conditioned on agent beliefs is monotonically nonincreasing, in the sense that $\mathbb{I}(\target; \environment | P_t) \geq \mathbb{I}(\target; \environment | P_{t+\tau})$.  With this in mind, we consider quantifying the manner in which an agent balances immediate shortfall and information via the {\it $\tau$-information ratio}
$$\Gamma_{\tau,t} = \frac{\E[V_*(H_t) - Q_*(H_t, A_t)]^2}{(\I(\target; \environment|P_t) - \I(\target; \environment|P_{t+\tau})) / \tau}.$$
The numerator is squared expected shortfall, while the denominator represents information gain, normalized by duration.  As the numerator grows, actions sacrifice a larger amount of immediate reward.  As the denominator grows, actions more substantially inform the agent.  The ratio reflects the trade-off that the agent is striking.  When both numerator and denominator are zero, we take $\Gamma_{\tau,t}$ to be zero.  For the case of $\tau=1$, we simply refer to this as the information ratio and write $\Gamma_t \equiv \Gamma_{1,t}$.

When designing and analyzing agents, it is often helpful to consider shortfall relative to a suboptimal baseline.  This leads to the notion of a $(\tau,\epsilon)$-information ratio, defined by
$$\Gamma_{\tau,\epsilon,t} = \frac{\E[V_*(H_t) - Q_*(H_t, A_t) - \epsilon_t]_+^2}{(\I(\target; \environment|P_t) - \I(\target; \environment|P_{t+\tau}))/\tau},$$
for the sequence $\epsilon = (\epsilon_0,\ldots, \epsilon_{T-1})$.  We use the subscript plus sign to denote the positive part of a number; in other words, $x_+ = \max(x,0)$.
Note that, if $\epsilon = (0,\ldots,0)$ then $\Gamma_{\tau,\epsilon,t} = \Gamma_{\tau,t}$.

Given a policy $\overline{\pi}$ chosen as a baseline against which the agent is designed to compete, it is natural to consider a variant of the information ratio that depends on shortfall with respect to $\overline{\pi}$, which takes the form
$$\frac{\E[V_{\overline{\pi}}(H_t) - Q_{\overline{\pi}}(H_t, A_t)]^2_+}{(\I(\target; \environment|P_t) - \I(\target; \environment|P_{t+\tau})) / \tau} = \Gamma_{\tau,\epsilon,t},$$
with
$$\epsilon_t = \E[V_*(H_t) - Q_*(H_t, A_t)] - \E[V_{\overline{\pi}}(H_t) - Q_{\overline{\pi}}(H_t, A_t)].$$
When $\epsilon$ is chosen in this manner, we will alternately denote the $(\tau,\epsilon)$-information ratio by $\Gamma_{\tau,\overline{\pi}, t} = \Gamma_{\tau,\epsilon, t}$, in which case we will refer to it as the $(\tau,\overline{\pi})$-information ratio.

These definitions generalize the concept of an information ratio as originally proposed by \citet{russo2016information} as a tool for analyzing Thompson sampling \citep{thompson1933likelihood,thompson1935theory}.  A growing body of work has studied, applied, and extended this concept \citep{russo2014learning,bubeck2015bandit,bubeck2015bandit,bubeck2016multi,NEURIPS2018_f3e52c30,liu2018information,russo2018learning,dong2019performance,lu2019information-confidence,zimmert2019connections,lattimore2019information,bubeck2020first,xlu2020dissertation,russo2020satisficing}.  The definitions of this section unify and extend previous ones to accommodate general learning targets, baseline policies, and delayed information.

\section{A Regret Bound}

It is generally difficult to understand the exact impact of design choices like proxies and targets on regret.  These choices impact what uncertainty the agent aims to resolve, regret incurred to do that, and regret associated with the target policy.  However, we can establish a regret bound that simplifies and offers insight into the tradeoffs.  As a much simpler alternative to minimizing regret, a designer could aim to minimize this bound.  We use $\mathbb{Z}_{++}$ and $\Re_+$ to denote the positive integers and nonnegative reals.
\begin{theorem} \label{th:regret-bound}
If $\mathbb{I}(\target; \environment | P_t)$ is monotonically nonincreasing with $t$, then, for all $\tau \in \mathbb{Z}_{++}$ and $\epsilon \in \Re^T_+$,
\[\Regret(T|\pi_{\rm agent}) \leq \sqrt{\I(\target; \environment) \sum_{t=0}^{T-1} \Gamma_{\tau,\epsilon,t}} + \sum_{t=0}^{T-1} \epsilon_t.\]
\end{theorem}
\begin{proof}
By Corollary \ref{co:shortfall},  
\begin{align*}
\E[\overline{V}_* - \overline{V}_{\pi_{\mathrm{agent}}}]
=& \sum_{t=0}^{T-1} \E[V_*(H_t) - Q_*(H_t, A_t)] \\
=& \sum_{t=0}^{T-1} \E[V_*(H_t) - Q_*(H_t, A_t) - \epsilon_t] + \sum_{t=0}^{T-1} \epsilon_t \\
\leq& \sum_{t=0}^{T-1} \E[V_*(H_t) - Q_*(H_t, A_t) - \epsilon_t]_+ + \sum_{t=0}^{T-1} \epsilon_t \\
=& \sum_{t=0}^{T-1} \sqrt{\Gamma_{\tau,\epsilon,t} \frac{1}{\tau} \left(\I(\target; \environment|P_t) - \I(\target; \environment|P_{t+\tau})\right)} + \sum_{t=0}^{T-1} \epsilon_t \\
\overset{(a)}{\leq}& \sqrt{\frac{1}{\tau} \sum_{t=0}^{T-1} \left(\I(\target; \environment |P_t) - \I(\target; \environment |P_{t+\tau})\right)} \sqrt{\sum_{t=0}^{T-1} \Gamma_{\tau,\epsilon,t}} + \sum_{t=0}^{T-1} \epsilon_t \\
\overset{(b)}{\leq}& \sqrt{\I(\target; \environment|P_0)} \sqrt{\sum_{t=0}^{T-1} \Gamma_{\tau,\epsilon,t}} + \sum_{t=0}^{T-1} \epsilon_t \\
=& \sqrt{\I(\target; \environment) \sum_{t=0}^{T-1} \Gamma_{\tau,\epsilon,t}} + \sum_{t=0}^{T-1} \epsilon_t,
\end{align*}
where (a) follows from Cauchy-Bunyakovsky-Schwarz and (b) holds because 
\begin{align*}
& \frac{1}{\tau} \sum_{t=0}^{T-1} \left(\I(\target; \environment |P_t) - \I(\target; \environment|P_{t+\tau})\right) \\
=& \frac{1}{\tau} \sum_{t=0}^{T-1} \sum_{k=0}^{\tau-1} (\I(\target; \environment|P_{t+k}) - \I(\target; \environment|P_{t+k+1})) \\
=& \frac{1}{\tau} \sum_{k=0}^{\tau-1} \sum_{t=0}^{T-1} (\I(\target; \environment|P_{t+k}) - \I(\target; \environment|P_{t+k+1})) \\
=& \frac{1}{\tau} \sum_{k=0}^{\tau-1} (\I(\target; \environment|P_k) - \I(\target; \environment|P_{T+k}))\\
\overset{(c)}{\leq}& \frac{1}{\tau} \sum_{k=0}^{\tau-1} \I(\target; \environment|P_k) \\
\overset{(d)}{\leq}& \I(\target; \environment|P_0),
\end{align*}
where (c) holds because mutual information is nonnegative and (d) holds because $\mathbb{I}(\target; \environment | P_t)$ is monotonically nonincreasing by assumption.
\end{proof}

If an agent learns its target $\target$, it can execute the target policy $\pi_\target$.  As such, it is natural to consider $\pi_\target$ as a baseline for assessing the agent's performance.  With
\[\epsilon_t = \E[V_*(H_t) - Q_*(H_t, A_t)] - \E[V_{\pi_\target}(H_t) - Q_{\pi_\target}(H_t, A_t)],\]
Theorem \ref{th:regret-bound} and Theorem \ref{th:shortfall} yield the following result.
\begin{corollary} \label{co:regret-bound}
If $\I(\target; \environment | P_t)$ is monotonically nonincreasing with $t$, then, for all $\tau \in \mathbb{Z}_{++}$,
\[\Regret(T|\pi_{\rm agent}) \leq \sqrt{\I(\target; \environment) \sum_{t=0}^{T-1} \Gamma_{\tau,\pi_\target,t}} + \Regret(T|\pi_\target).\]
\end{corollary}
These results unify and generalize those of \citet{russo2016information,russo2020satisficing,NEURIPS2018_f3e52c30,lu2019information-confidence,xlu2020dissertation}.  Among other things, they address delayed information through an information ratio that is distinguished from prior work by its dependence on an {\it information horizon} $\tau$.  The bounds hold for any value of $\tau$.  Intuitively, $\tau$ should be chosen to cover the duration over which an action may substantially impact information subsequently revealed.

The bound of Corollary \ref{co:regret-bound} isolates three factors.  The mutual information $\I(\target; \environment)$ is the number of bits required to resolve environment-relevant uncertainty about learning target.  Another is the regret of the target policy, $\Regret(T|\pi_\target)$. Once the agent resolves all uncertainty about the learning target, $\Regret(T|\pi_\target)$ measures the performance shortfall of the target policy relative to the optimal policy. The role of the information ratio $\Gamma_{\tau,\pi_\target, t}$ deserves the most discussion.  As mentioned earlier, this quantifies the manner in which the agent trades off between regret and bits of information.  It is through $\Gamma_{\tau,\pi_\target, t}$ that the learning target and proxy influence rates at which relevant bits are acquired and retained in epistemic state.  Examples below and material of Sections \ref{se:retaining-information} and \ref{se:seeking-information} clarify the role that the information ratio can play in analyzing agents and designing ones that are effective at seeking and retaining useful information.

\section{Examples}

We next present several examples that illustrate implications of our regret bound and, in particular, how it can help in understanding an agent's data efficiency. While this regret bound can apply to any agent, for the purpose of illustration, we will focus mostly on Thompson sampling \citep{thompson1933likelihood,russo2018tutorial} and information-directed sampling agents \citep{russo2014learning}. We begin with multi-armed bandit environments, in which actions impact the immediate observation but do not induce delayed consequences.  Then, we consider in Sections \ref{se:episodic-mdp} and \ref{se:chain} examples for which an agent's handling of delayed consequences becomes essential.  

We find that in the mathematical analyses in Sections \ref{se:mab} and \ref{se:episodic-mdp}, it is more natural to measure information in nats rather than bits. For simplicity, we will use the same notation for entropy and mutual information but with the understanding that information is measured in nats.

\subsection{Multi-Armed Bandits}
\label{se:mab}

A multi-armed bandit, or {\it bandit} for short, is an environment $\environment = (\actions, \observations, \rho)$ for which each observation $O_{t+1}$ depends on the history $H_t$ only through the action $A_t$.  As such, the observation probability function can be written as $\rho(o|a) \equiv \rho(o|h,a)$.  We will refer to a reward function $r$ as a {\it bandit reward function} if reward similarly depends on history only through the current action and resulting observation, so that rewards can be written as $r(a,o) \equiv r(h,a,o)$.  To simplify exposition, we will assume that bandit reward functions range within $[0,1]$.  Since observation probabilities and rewards depend on $H_t$ only through $A_t$, there exists an optimal policy $\pi_*$ that selects actions independent from history, assigning a probability $\pi_*(a)$ to each action.

{\it Bandit} is an antiquated term for a slot machine, which ``robs'' the player of his money.  Each action can be viewed as pulling an arm of the machine.  The resulting symbol combination and payout serve as observation and reward.  A two-armed version is depicted in Figure \ref{fig:MAB}.

\begin{figure}[htb]
\begin{center}
\includegraphics[scale=0.25]{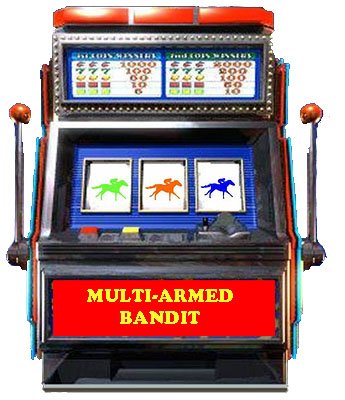}
\caption{A two-armed bandit.}
\label{fig:MAB}
\end{center}
\end{figure}

\vspace{0.1in}
\noindent{\bf Information Ratio and Regret}

\noindent The information ratio and our regret bound simplify when specialized to multi-armed bandits. Optimal action values are given by $Q_*(h, a) = \sum_{o \in \observations} \rho(o|a) r(a, o)$,
which do not depend on $h$.  Letting $\overline{R}_a = Q_*(h, a)$ and $\overline{R}_* = \max_{a \in \actions} \overline{R}_a$, per-period regret can be written as
$$\E[V_*(H_t) - Q_*(H_t, A_t)] = \E[\overline{R}_* - R_{t+1}].$$
For the purpose of examples in this section, we consider immediate information gain and thus the $1$-information ratio $\Gamma_t = \Gamma_{1,t}$, which takes the form
$$\Gamma_t = \frac{\E[\overline{R}_* - R_{t+1}]^2}{\I(\target; \environment|P_t) - \I(\target; \environment|P_{t+1})}.$$

More generally, we can consider an information ratio $\Gamma_{\overline{\pi}, t} = \Gamma_{1,\overline{\pi},t}$ with a baseline $\overline{\pi}$ that assigns a probability $\overline{\pi}(a)$ to each action, given by
$$\Gamma_{\overline{\pi}, t} = \frac{\E[\overline{R}_{\overline{\pi}} - R_{t+1}]_+^2}{\I(\target; \environment|P_t) - \I(\target; \environment|P_{t+1})},$$
where $\overline{R}_{\overline{\pi}} = \sum_{a \in \actions} \overline{\pi}(a) \overline{R}_a$.
Specializing Corollary \ref{co:regret-bound} to this context, with a target policy $\pi_\target$ that selects actions independently from history, we have
\[\Regret(T|\pi_{\rm agent}) \leq \sqrt{\I(\target; \environment) \sum_{t=0}^{T-1} \Gamma_{\pi_\target,t}} + \Regret(T|\pi_\target).\]

\vspace{0.1in}
\noindent{\bf Worst Case}

\noindent Consider an agent that takes the proxy $\proxy = \rho$ to be the observation probability function, the epistemic state $P_t$ to be the posterior distribution $\Pr(\proxy \in \cdot | H_t)$, the learning target $\target \in \argmax_{a \in \actions} \sum_{o \in \observations} \rho(o|a) r(a,o)$ to be an action that maximizes expected reward, and the target policy $\pi_\target$ to be a policy that executes action $\target$.  Suppose the agent applies Thompson sampling, which entails sampling an approximation $\rho_t$ independently from $\Pr(\proxy \in \cdot | P_t)$ and executing $A_t \in \argmax_{a \in \actions} \sum_{o \in \observations} \rho_t(o|a) r(a,o)$.

As established in \citet{russo2016information}, the associated information ratio satisfies $\Gamma_{t} \leq \frac{1}{2}\actions$.  Note that, as shorthand, when this interpretation is clear from context, we use set notation such as $\actions$ to denote cardinality $|\actions|$.  It follows from Theorem \ref{th:regret-bound} that
$$\mathrm{Regret}(T|\pi_{\rm agent}) \leq \sqrt{\frac{1}{2} \actions T \I(\target; \environment)} \leq \sqrt{\frac{1}{2} \actions T \ln \actions}.$$
Note that $\Regret(T|\pi_\target)=0$ here because the target policy represents an optimal policy.
The same bound was established in \citet{russo2016information} via a more specialized analysis.  We refer to this as a worst-case bound because it applies for a Thompson sampling agent so long as the environment $\environment$ is known to be a multi-armed bandit.

\vspace{0.1in}
\noindent{\bf Satisficing}

\noindent When there are many possible actions, it can be advantageous to target a satisficing one rather than search for too long to find an optimal action.  Let us illustrate this issue in the context of many-armed bandits.  While our previous worst-case regret bound grows with the number of actions, we will establish an alternative that relaxes this dependence and is therefore more attractive in the many-action regime.

Without loss of generality, let $\actions = \{1,\ldots,\actions\}$; recall that we use sets to denote their cardinality when that is clear from context.  As a learning target, consider
$$\target = \min\left\{a \in \actions: \sum_{o \in \observations} \rho(o|a) r(a,o) \geq \overline{R}_*-\epsilon\right\},$$
which is the first $\epsilon$-optimal action and can be interpreted as a satisficing one.  If the target policy $\pi_\target$ executes action $\target$ , we have $\Regret(T|\pi_\target) \leq \epsilon T$.

Such a learning target does not necessarily reduce regret.  For example, if all actions yield zero reward except one that yields reward one, the agent cannot do better than trying every action.  To restrict attention to cases where our target is helpful, let us assume that observation probabilities $\rho(\cdot|a)$ are independent and identically distributed across actions.  Then, regret can be bounded in a manner that depends on the probability $p_{\epsilon,\actions} = \Pr(\sum_{o \in \observations} \rho(o|a) r(a,o) \geq \overline{R}_* - \epsilon)$ of $\epsilon$-optimality.

Suppose that the agent applies a variant of Thompson sampling that samples an approximation $\hat{\rho}_t$ independently from the posterior $\Pr(\proxy \in \cdot | P_t)$ and executes 
$$A_t = \min\left\{a \in \actions: \sum_{o \in \observations} \hat{\rho}_t(o|a) r(a,o) \geq \hat{R}_{*,t}-\epsilon\right\},$$
where $\hat{R}_{*,t} = \max_{a\in\actions} \sum_{o\in\observations} \hat{\rho}_t(o|a) r(a,o)$.
In other words, the agent samples $A_t$ from the posterior distribution $\Pr(\target \in \cdot | P_t)$ of the satisficing action.  The analysis of \citet{russo2020satisficing} establishes bounds on entropy
$$\H(\target) \leq 1 + \ln \frac{1}{p_{\epsilon, \actions}},$$
and the information ratio
$$\Gamma_{\pi_\target, t} \leq 2 \left(2 + \frac{1 + \ln T}{p_{\epsilon,\actions}}\right).$$
Combining this with Theorem \ref{th:regret-bound} leads to a regret bound
$$\mathrm{Regret}(T|\pi_{\rm agent}) \leq \sqrt{2 \left(2 + \frac{1 + \ln T}{p_{\epsilon,\actions}}\right) \left(1 + \ln \frac{1}{p_{\epsilon,\actions}}\right) T } + \epsilon T.$$
Note that $p_{\epsilon,\actions}$ monotonically decreases and converges to some $p_{\epsilon, \infty} > 0$ as the number of actions goes to infinity. Therefore, the regret is upper bounded by the same expression before but with $p_{\epsilon,\actions}$ replaced by $p_{\epsilon,\infty}$. For example, if the mean reward of an action $\overline{R}_a = \sum_{o\in\observations} \rho(o|a)r(a,o)$ is positively supported on $[0, 1]$, then the regret bound only depends on $p_{\epsilon, \infty} = \Pr(\overline{R}_a \geq 1 - \epsilon)$ and not on the number of actions.

\vspace{0.1in}
\noindent{\bf Linear Bandit}

\noindent In a linear bandit, observations represent rewards with expectations that depend linearly on features of the action.  In particular, the action set $\actions \subset \{a \in \Re^d: \|a\|_2 = 1\}$ is comprised of $d$-dimensional unit vectors and the expected reward $\E[R_{t+1}|\environment, A_t=a] = \sum_{o \in \observations} \rho(o|a) r(a,o) = \theta^\top a$ depends linearly on a random vector $\theta$.  As established in \citet{russo2016information}, with the proxy and learning target taken to be the parameter vector $\proxy = \theta$ and an action $\target \in \argmax_{a \in \actions} \theta^\top a$ that maximizes expected reward, the associated information ratio is bounded according to $\Gamma_t \leq \frac{1}{2} d$ for Thompson sampling. Theorem \ref{th:regret-bound} then yields a regret bound
$$\mathrm{Regret}(T|\pi_{\rm agent}) \leq \sqrt{\frac{1}{2} d T \I(\target; \environment)} \leq \sqrt{\frac{1}{2} d T \ln \actions},$$
for Thompson sampling. 

An alternative bound, established in \citet{NEURIPS2018_f3e52c30}, relaxes the dependence on the number of actions and is preferable when there are many.  That bound can be produced by taking the proxy to be a lossy compression $\proxy = \tilde{\theta}$ of $\theta$ encoded by $\H(\proxy) = \H(\tilde{\theta}) \ll \H(\theta) = \H(\environment)$ nats.  As explained in \citet{NEURIPS2018_f3e52c30}, there exists a compression $\tilde{\theta}$ with $\H(\tilde{\theta}) \leq d \ln (1+1/\epsilon)$ nats such that a learning target $\target \in \argmax_{a \in \actions} \tilde{\theta}^\top a$ attains $\mathrm{Regret}(T|\pi_\target) \leq \epsilon T$.  With this proxy, Theorem \ref{th:regret-bound} implies
\begin{align*}
\mathrm{Regret}(T|\pi_{\rm agent}) &\leq \sqrt{\frac{1}{2} d T \H(\target)}  + \epsilon T  \\
&\leq \sqrt{\frac{1}{2} d T \H(\tilde{\theta})}  + \epsilon T \\
&\leq d \sqrt{\frac{1}{2} \ln \left(1 + \frac{1}{\epsilon}\right) T}  + \epsilon T,
\end{align*}
for Thompson sampling.  Since actions executed by Thompson sampling do not depend on the proxy $\proxy$, the bound holds for any choice of $\epsilon$.  Minimizing over $\epsilon$, we obtain
$$\mathrm{Regret}(T|\pi_{\rm agent}) \leq d \sqrt{T \ln \left(3 + \frac{3 \sqrt{2T}}{d}\right)},$$
as established in \citet{NEURIPS2018_f3e52c30}.  As is desirable for large action sets, this bound does not depend on the number of actions.

\vspace{0.1in}
\noindent{\bf Information-Directed Sampling}

\noindent Agents considered in the preceding examples employ Thompson sampling.  While this is an elegant approach to action selection, it is possible to design agents that suffer less regret, sometimes with dramatic differences.  One alternative, inspired by our regret bound, is {\it information-directed sampling} (IDS), a concept first developed in \citet{russo2014learning}.  We offer a more extensive discussion of IDS in Section \ref{se:seeking-information}, where we present a more general form intended for environments in which actions induce delayed consequences.  Here, we consider a special case that applies to multi-armed bandits.

To select an action $A_t$, the version of IDS we consider solves
\begin{equation}
\label{eq:MAB-IDS}
\min_{\nu \in \Delta_\actions} \frac{\E\left[\overline{R}_{\pi_\target} - \overline{R}_{\tilde{A}_t} | P_t\right]_+^2}{\E[\I(\target; \environment|P_t \leftarrow P_t) - \I(\target; \environment | P_{t+1} \leftarrow \tilde{P}_{t+1}) | P_t]},
\end{equation}
where $\Delta_\actions$ is the set of action probability vectors, $\tilde{A}_t$ is sampled from $\nu$, and $\tilde{P}_{t+1}$ is the next epistemic state realized as a consequence.  The objective can be thought of as a {\it conditional information ratio}, with the numerator determined by the conditional expectation of shortfall and the denominator a measure of conditional information gain.  To understand the latter, it is helpful to consider the special case in which the target $\target$ is determined by the environment, the epistemic state is the entire history $P_t = H_t$, and the observation $O_{t+1}$ reveals no information beyond the reward $R_{t+1}$.  In this case, the denominator becomes
$$\E[\I(\target; \environment|P_t \leftarrow P_t) - \I(\target; \environment| P_{t+1} \leftarrow \tilde{P}_{t+1}) | P_t] = \I(\target; \tilde{A}_t, \tilde{R}_{t+1} | P_t \leftarrow P_t),$$
where $\tilde{R}_{t+1}$ is the reward realized as a consequence of action $\tilde{A}_t$.  This is the number of bits about $\target$ revealed by $(\tilde{A}_t, \tilde{R}_t)$.

As originally observed by \citet{russo2014learning} and explained in Appendix \ref{ap:support-cardinality}, the objective of (\ref{eq:MAB-IDS}) is convex and the minimum can be attained by randomizing between no more than two actions.  In other words, it suffices to consider two-sparse vectors $\nu$.  This can help to keep the optimization problem computationally manageable.

In bandit environments that call for thoughtful information-seeking behavior, IDS often outperforms Thompson sampling.  As we will demonstrate in Section \ref{se:computation}, the performance difference can be dramatic. 
The conditional information ratio that IDS optimizes is by definition upper bounded by that of Thompson sampling given the epistemic state. For the bandit environments discussed earlier, the bounds on Thompson sampling's information ratios are in fact proven for the conditional version \citep{russo2014learning,russo2020satisficing,NEURIPS2018_f3e52c30}. Thus, the regret bounds discussed earlier, which applied to Thompson sampling agents, also apply to IDS agents.

\vspace{0.1in}
\noindent{\bf Upper-Confidence Bounds}

\noindent Upper-confidence bounds (UCBs) offer an alternative approach to agent design \citep{lai1985asymptotically,lai1987adaptive,auer2002finite,bubeck2012regret}.  As explained in \citet{russo2014posterior}, UCB algorithms are closely related to Thompson sampling and can often be analyzed using similar mathematical techniques.  As discussed in \citet{lu2019information-confidence,xlu2020dissertation}, one can also study UCB algorithms via the information ratio.  In particular, that work bounds information ratios of suitably designed UCB algorithms that address beta-Bernoulli bandits, Gaussian-linear bandits, and tabular Markov decision processes.  The regret bounds presented in that line of work can also be established using Theorem \ref{th:regret-bound}.

\subsection{Thompson Sampling for Episodic ``Ring'' Markov Decision Processes}
\label{se:episodic-mdp}

A Markov decision process (MDP) is an environment $\environment = (\actions, \observations, \rho)$ for which each observation $O_t$ serves as a sufficient statistic of the preceding history $H_t$.  In other words, observation probabilities depend on history only through the most recent observation.  It is natural to take the situational state to be observation $S_t = O_t$, except for the deterministic initial state $S_0$.  Further, let $\states = \observations$ and $\rho(s'|s,a) = \rho(o|h,a)$ if $o=s'$ and $h = (\ldots, s)$.

The use of Thompson sampling for episodic MDPs, as introduced by 
\citet{DBLP:conf_icml_Strens00}, is often referred to as {\it posterior sampling for reinforcement learning} (PSRL).  The algorithm has been analyzed extensively in the literature \citep{NIPS2013_6a5889bb,pmlr-v70-osband17a,NIPS2017_51ef186e,NIPS2017_3621f145,lu2019information-confidence,xlu2020dissertation}, leading to several of regret bounds.  Here we examine whether our information theoretic approach, and Theorem \ref{th:regret-bound} in particular, implies a regret bound for PSRL.

\begin{figure}[htb]
\begin{center}
\includegraphics[scale=0.25]{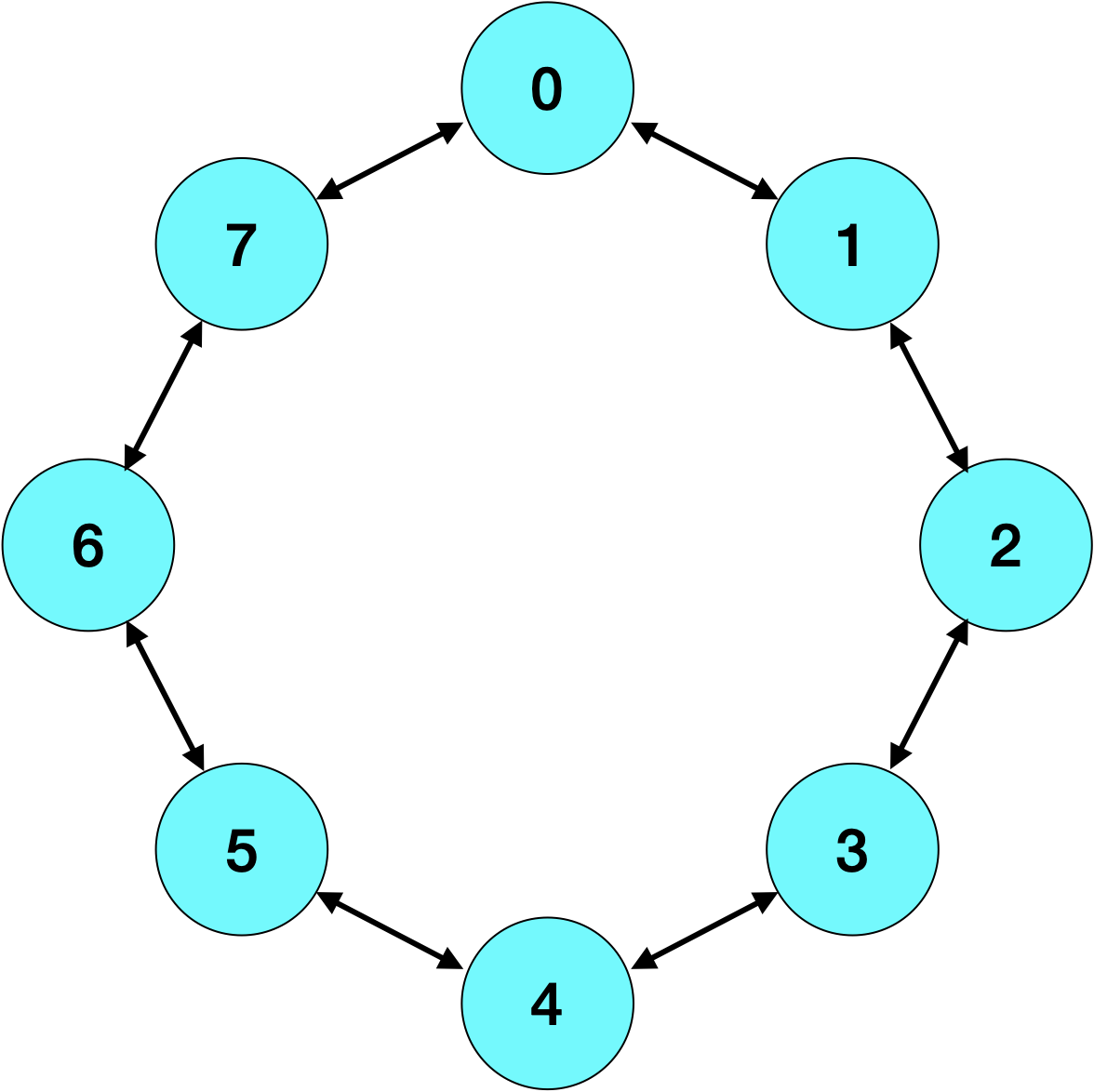}
\caption{In our example of an episodic MDP, state encodes location and phase.  Within an episode, location increases or decreases by one, modulo the number $M$ of locations, over each timestep.  This figure illustrates locations and possible intraepisodic transitions for the case of $M=8$.  At the end of each episode, the location transitions to $0$.}
\label{fig:EpisodicRing}
\end{center}
\end{figure}

In this section, we consider an MDP with which an agent interacts over episodes of fixed duration $\tau$.  In particular, the state sequence renews at the end of each episode, when it returns to a distinguished state $S_0$.

Conditioned on an episodic MDP of the sort we have described, an optimal policy can be derived by planning over $\tau$ timesteps.  To keep its structure simple, we take the state space to be $\states = \{0,\ldots,M-1\} \times \{0,\ldots,\tau-1\}$ for some positive integer $M$, so that each state is a pair $(m,k)$. Let $\states_k = \{(m, k): m \in \{0, \dots, M-1\}\}$. The optimal action value function $Q_{\tau,\rho}$ for the planning problem uniquely solves Bellman's equation
$$Q_{\tau,\rho}(s,a) = \left\{\begin{array}{ll}
\sum_{s' \in \states} \rho(s'|s,a) (r(s,a,s') + \max_{a' \in \actions} Q_{\tau,\rho}(s',a')) \quad & \text{if } s \notin \states_{\tau-1} \\
r(s, a, S_0) \quad & \text{otherwise.}
\end{array}\right.$$
Given $Q_{\tau,\rho}$, any policy that selects greedy actions $A_t \in \argmax_{a \in \actions} Q_{\tau,\rho}(S_t, a)$ is optimal.  Further, letting $V_{\tau,\rho}(s) = \max_{a \in \actions} Q_{\rho,\tau}(s,a)$, we have $V_*(H_t) - Q_*(H_t,\cdot) = V_{\tau,\rho}(S_t) - Q_{\tau,\rho}(S_t,\cdot)$.

We will make an additional assumption that simplifies our example without forgoing essential insight.  Given a state $s = (m,k) \in \states$, we write $s+1$ and $s-1$ as shorthand for $(m + 1 \mod{M}, k + 1 \mod{\tau})$ and $(m - 1 \mod{M}, k + 1 \mod{\tau})$.  We will assume that from any state $s \notin \states_{\tau-1}$ it is only possible to transition to $s+1$ and $s-1$.  With this assumption, the environment can be identified by assigning a probability $\rho(s+1|s,a)$ to each state-action pair $(s,a)$.  To keep things simple, let us assume that each $\rho(s+1|s,a)$ is independent and uniformly distributed over the unit interval.


Let us take the epistemic state $P_t$ to be the posterior distribution $\Pr(\cdot | H_t)$.  
Note that, since the uniform distribution is a beta distribution, conditioned on the history $H_t$, each $\rho(s+1|s,a)$ remains independent and beta distributed.

As an agent policy $\pi_{\rm agent}$, consider a version of Thompson sampling \citep{DBLP:conf_icml_Strens00,NIPS2013_6a5889bb}, referred to as $\pi_{\rm TS}$, which at the beginning of each $\ell^{\mathrm{th}}$ episode, executes the following steps: first, it samples $\hat{\rho}_\ell$ independently from $P_{\ell \tau}$; second, it computes the associated action value function $\hat{Q}_\ell = Q_{\tau,\hat{\rho}_\ell}$; finally, it generates a policy $\hat{\pi}_\ell$ that is greedy with respect to $\hat{Q}_\ell$, for which
$$\mathrm{support}(\hat{\pi}_\ell(\cdot|s)) \subseteq \argmax_{a \in \actions} \hat{Q}_\ell(s,a).$$

As is detailed in Appendix~\ref{se:mdp-analysis}, for any integer $m \geq 2$ and its reciprocal $\delta=1/m$, we define a learning target $\target$ to be a quantized approximation of $\rho$ for which $\target(s+1|s,a) = \delta \lceil \rho(s+1|s,a)/\delta \rceil$ and $\target(s-1|s,a) = 1 - \target(s+1|s,a)$. With this quantized learning target $\target$, Theorem~\ref{th:mdp-regret-bound} in Appendix~\ref{se:mdp-analysis} shows that under an ``optimism conjecture'' (Conjecture \ref{conj:ts-optimism}), which we support through simulations but leave as an open problem, we have
\begin{align*}
    \Regret(T | \pi_{\rm TS}) \leq &  \, \mathcal{O} \left(
    \tau^2 \sqrt{\ln \left(\frac{1}{\delta} \right) \ln \left(\frac{\states \actions}{\delta} \right)} \left[\sqrt{\states \actions T } + T \sqrt{\delta} \right]
    \right).
\end{align*}
At a high level, this regret bound is derived as follows. We first derive an upper bound $\bar{\Gamma}$ on the $(\tau, \epsilon)$-information ratio $\Gamma_{\tau, \epsilon, t}$ under $\pi_{\rm TS}$, with
\[
\epsilon= \left[3 \delta + \sqrt{ 6 \max \left \{ 3, \ln \left(\frac{2}{\delta} \right)\right \} \delta \ln \frac{2 \states \actions}{\delta}} \right] \tau^2.
\]
Then, we apply Theorem~\ref{th:regret-bound} to derive a regret bound based on $\bar{\Gamma}$, $\I(\target; \environment)$, and $\epsilon$. Finally, we bound the mutual information $\I(\target; \environment)$ by
$
\I(\target; \environment) \leq \H(\target) \leq \states \actions \ln \frac{1}{\delta}$. 


\subsection{Information-Directed Sampling with Delayed Consequences}
\label{se:chain}

Let us now consider a version of IDS designed for scalability and delayed consequences.  While we provide a more extensive discussion of IDS in Chapter \ref{se:seeking-information}, we describe here a simple version and demonstrate that it efficiently explores in some environments that require deep exploration.  Understanding the extent to which this capability extends to other environments remains an intriguing research direction.

\subsubsection{Algorithm}

We take our learning target to be a target policy $\target=\pi_\target$.  The version of IDS we consider is similar to (\ref{eq:MAB-IDS}), which was suitable for multi-armed bandits.  That objective balances expected shortfall against information revealed by the immediate reward $R_{t+1}$.  The version we consider here operates similarly, except for pretending that the action value $Q_{\pi_\target}(H_t, A_t)$ would be observed instead of $R_{t+1}$, and that the information is only relevant to the target through $\pi_\target(\cdot|S_t)$.  In particular, the objective becomes
\begin{equation}
\label{eq:basic-value-IDS}
\min_{\nu \in \Delta_\actions} \frac{\E\left[V_{\pi_\target}(H_t) - Q_{\pi_\target}(H_t, \tilde{A}_t) \big| X_t\right]_+^2}{\I(\pi_\target(\cdot|S_t); \tilde{A}_t, Q_{\pi_\target}(H_t, \tilde{A}_t) | X_t \leftarrow X_t)},
\end{equation}
where $\tilde{A}_t$ is sampled from $\nu$.  This is a special case of {\it value-IDS}, which we will present at greater length in Section \ref{se:value-IDS}.

\subsubsection{Analysis}

Analysis of value-IDS remains an active area of research.  Computational results in Chapter \ref{se:computation} demonstrate promise for value-IDS as a scalable and data-efficient approach to action selection.  As a sanity check, we apply Theorem \ref{th:regret-bound} and establish in Appendix \ref{ap:chain} a regret bound for the version of IDS specified by (\ref{eq:basic-value-IDS}) applied to a simple class of environments.

Consider an episodic environment $\environment = (\actions, \observations, \rho)$ with actions $\actions = \{0,1\}$ and observations $\observations = \{(0,0), (0,1), 0, 1, \ldots,2\tau-2\}$.  The environment is parameterized by $r_0,\ldots,r_{\tau-2} \in [0,1)$ and $r_{\tau-1} \in \{0,1\}$.  An illustration of the state dynamics and rewards can be found in Figure \ref{fig:chain}. Conditioned on $\environment$, observations are deterministic, with
$$O_{t+1} = \left\{\begin{array}{ll}
O_t + \tau \qquad & \text{if } O_t \in \{0, \ldots, \tau-2\} \text{ and } A_t = 0 \\
O_t + 1 \qquad & \text{if } O_t \in \{0, \ldots, \tau-2\} \text{ and } A_t = 1 \\
\tau \qquad & \text{if } O_t \in \{(0,0), (0,1)\} \text{ and } A_t = 0 \\
1 \qquad & \text{if } O_t \in \{(0,0), (0,1)\} \text{ and } A_t = 1 \\
(0,r_{\tau-1}) \qquad & \text{if } O_t = \tau-1 \\
O_t + 1 \qquad & \text{if } O_t \in \{\tau,\ldots, 2\tau-3\} \\
0 \qquad & \text{if } O_t = 2\tau-2.
\end{array}
\right.
$$
This environment is deterministic, in the sense that $O_{t+1}$ is determined by $O_t$ and $A_t$.  The environment is also episodic, since $O_t \in \{(0,0), (0,1), 0\}$ if $t$ is a multiple of $\tau$.  It is natural to take the situational state to be $S_t = 0$ if $O_t = (0, r_{\tau-1})$ and, otherwise, $S_t = O_t$.  Rewards are determined by state and action, according to
$$R_{t+1} = \left\{\begin{array}{ll}
r_{S_t} \qquad & \text{if } S_t \in \{0,\ldots, \tau-2\} \text{ and } A_t = 0 \\
r_{S_t} \qquad & \text{if } S_t = \tau-1 \\
0 \qquad & \text{otherwise.}
\end{array}
\right.
$$

\begin{figure}[htb]
\begin{center}
\includegraphics[scale=0.25]{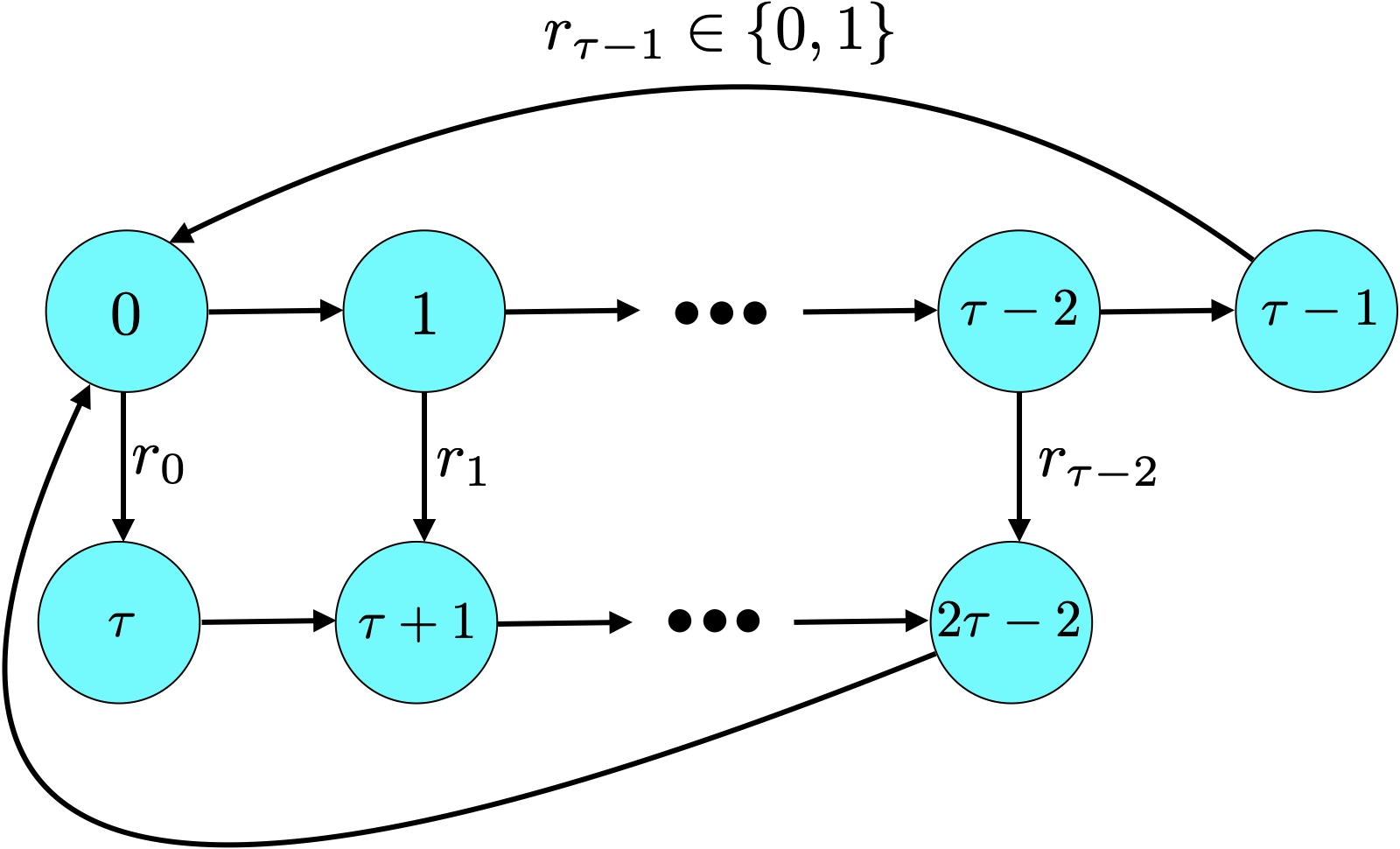}
\caption{A simple class of deterministic episodic environments, parameterized by episode duration $\tau$.  Edges represent possible state transitions, labeled with rewards.  The absence of a label indicates zero reward.}
\label{fig:chain}
\end{center}
\end{figure}

We consider a prior distribution $\Pr(\environment \in \cdot)$ with only two instances in its support.  Epistemic uncertainty arises only from an unknown value of $r_{\tau-1}$, for which $\Pr(r_{\tau-1} = 0) = \Pr(r_{\tau-1} = 1) = 1/2$.  Transition dynamics and the remaining reward values $r_0,\ldots,r_{\tau-2}$ are known.  Hence, to identify $\environment$, the agent need only observe the reward $r_{\tau-1}$ received upon leaving state $S_t = \tau-1$. For simplicity, assume that the exiting rewards are in a decreasing order, $1 > r_0 > r_1 > \dots > r_{\tau-2} > 0$, although the results in this section generalize to allow for any exiting rewards in $[0, 1)$.

This environment features delayed consequences. In particular, the agent needs to consecutively execute action $1$ in order to observe $r_{\tau-1}$. Note that the 1-information ratio is insufficient to guide efficient exploration in this example, because there is no information to be gained in the first $\tau - 1$ steps, resulting in infinite 1-information ratios. In general, 1-information ratios are unsuitable for representing the trade-off between sacrificing immediate reward and acquiring delayed information. For this environment, the regret bound in Theorem \ref{th:regret-bound} would be vacuous for all algorithms if using 1-information ratios, as some of the ratios are infinite. In contrast, a $\tau$-information ratio that considers information gain over multiple timesteps is more suitable for reasoning about sacrificing immediate reward for delayed information. Coincidentally, $\tau$ is also the horizon of this episodic environment, but any $\tau$ that is greater than the episode duration would be able to account for delayed information in this example.

The challenge in this environment is that the agent needs to consistently execute action $1$ in an episode in order to observe $r_{\tau - 1}$.  Simple dithering exploration schemes, such as $\epsilon$-greedy, can require in expectation an exponential number of episodes to reveal the value of $r_{\tau-1}$. This leads to regret exponential in $\tau$ (except for special cases in which the gap $1-r_s$ also decreases exponentially with $\tau$ for any $s$).

As established in Appendix \ref{ap:chain}, value-IDS interestingly avoids this exponential dependence for any $r_0, \dots, r_{\tau-2} \in [0, 1)$, despite the fact that it may randomize among actions during each timestep.  Because there are only two possible outcomes for $\pi_*$ with equal probability, we have $\H(\target) = \H(\pi_*) = 1$ bit, that is, the agent need only learn a single bit of information.  Further, Appendix \ref{ap:chain} establishes that the information ratio satisfies $\Gamma_{\tau, t} \leq \tau/8$, which together with Theorem \ref{th:regret-bound} yields
$$\Regret(T|\pi_{\rm agent}) \leq \frac{1}{2} \sqrt{\frac{\tau}{2} T}.$$
The regret bound avoids an exponential dependence on $\tau$, which implies that value-IDS efficiently explores in this environment, relative to dithering schemes.  In Chapter \ref{se:computation}, we further illustrate how variants of value-IDS scale to more complex environments.

\subsubsection{Limitations and Open Issues}

Let us close this section with a couple comments that may temper what readers infer from this regret bound.  First, while it is a special case of Theorem \ref{th:regret-bound}, the analysis presented in Appendix \ref{ap:chain} requires complicated calculations from which it is difficult to draw insight.  This subject would benefit from a more general and transparent analysis.  Second, \citet{qin2023technote} recently presented a similar class of environments that do not satisfy graceful regret bounds.  To offer a representative example, suppose that the exiting rewards are $r_s = 1 - \alpha^{\tau-1-s}$ for $s=0,\ldots,\tau-2$ and some $\alpha \in (0, 1)$, and $\Pr(r_{\tau-1} = 1) = \delta = 1- \Pr(r_{\tau-1} = 0)$ for some $\delta \in (0,1)$.  When $\delta=1/2$, we recover an example that satisfies the above regret bound, which scales with $\sqrt{\tau}$.  However, results of \citet{qin2023technote} establish that for sufficiently small $\delta$, $\sup_T \Regret(T|\pi_{\mathrm{agent}}) / \sqrt{T}$ grows exponentially in $\tau$.  While this result raises concerns about performance, simulations paint a qualitatively different picture.  In particular, Figure \ref{fig:chain_efficiency} plots the number of episodes required to attain expected average regret within tolerance $\epsilon$ as a function of $\epsilon$.  In contrast to what is suggested by \citet{qin2023technote}, the numbers point out that value-IDS learns efficiently across different combinations of $\alpha$, $\delta$, and $\tau$. Specifically, given tolerance $\epsilon$, Figure \ref{fig:chain_efficiency} shows that the number of episodes required to attain average regret under $\epsilon$ scales gracefully with $\tau$.  Understanding whether this sort of graceful behavior extends to other environments and why the mathematical framing of efficiency in terms of such regret bounds does not reflect that pose intriguing directions for future research.

\begin{figure}[htb]
\begin{center}
\includegraphics[scale=0.38]{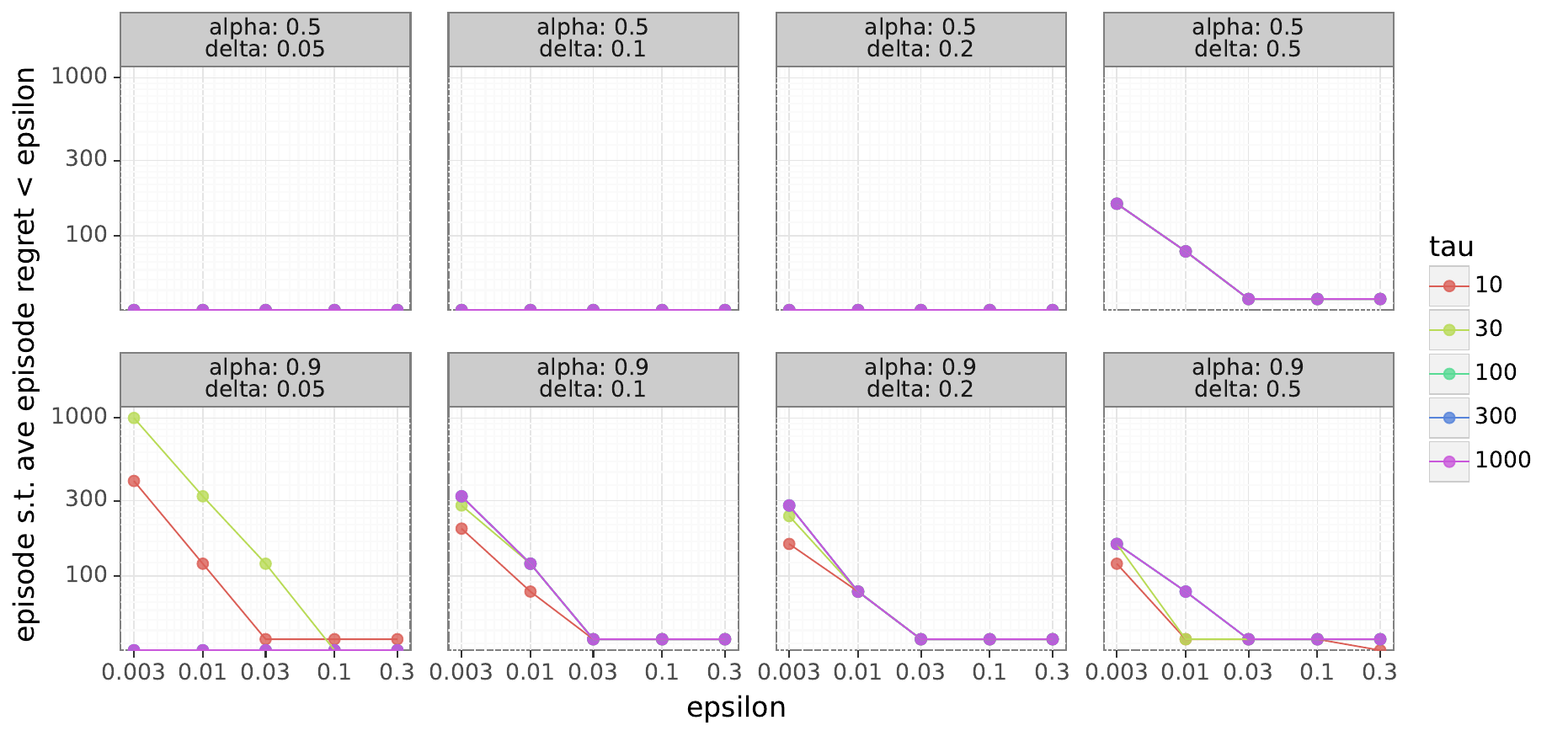}
\caption{The number of episodes required for value-IDS applied to a chain environment posed by \citet{qin2023technote} to attain expected average regret within tolerance $\epsilon$.  These results suggest that value-IDS learns efficiently across all horizons $\tau$, terminal reward prior probabilities $\delta$, and reward parameters $\alpha$.}
\label{fig:chain_efficiency}
\end{center}
\end{figure}

\chapter{Retaining Information}
\label{se:retaining-information}

In this chapter, we aim to provide insight into how the design of environment proxies and epistemic state dynamics influence information retention and regret.  We begin in Section \ref{se:point-estimate} by discussing the importance of retaining an epistemic state more than a point estimate of the proxy.  Section \ref{se:epistemic-update} follows up with considerations pertinent to design of epistemic state dynamics.  It may seem natural to take the environment proxy to be a target policy, but as we will discuss in Section \ref{se:target-vs-proxy}, using a richer proxy typically improves performance.  We then discuss a few possibilities, including value functions in Section \ref{se:vf-proxy}, general value functions in Section \ref{se:gvf}, and generative models in Section \ref{se:generative-models}.

\section{Epistemic State}

Recall that the agent represents its knowledge about the environment via an epistemic state $P_t$, which evolves according to
$$P_{t+1} = f_{\rm epis}(X_t, A_t, O_{t+1}, U_{t+1}),$$
for some update function $f_{\rm epis}$ and source $U_{t+1}$ of algorithmic randomness.  Epistemic state dynamics are generally designed to prioritize retention of information about an environment proxy $\proxy$.  This is quantified by the mutual information $\I(\proxy; P_t)$, which measures the number of proxy-relevant bits retained by $P_t$.

\subsection{Point Estimates Versus More Informative Epistemic States}
\label{se:point-estimate}

Common agent designs retain a point estimate of an environment proxy such as a policy, a value function, a general value function, or a generative model of environment dynamics.  For example, value function learning agents typically maintain and incrementally update an action value function, which can be thought of as a {\it best guess} of $Q_*$.  Such agents often retain no additional information that characterizes uncertainty about the point estimate.  However, representing such epistemic uncertainty is critical both to updating the epistemic state in a way that retains essential new information and gathering informative data.  As explained in \citet{osband2019deep}, for example, such point estimates are not sufficient to enable deep exploration.

Limitations of point estimates can be demonstrated via the simple coin tossing environment of Example \ref{ex:coin-tossing}.  Suppose that an agent's epistemic state is a vector of point estimates $P_t \in [0,1]^M$ of the heads probabilities for each of the $M$ coins.  Without further information retained to estimate epistemic uncertainty, the agent cannot efficiently update its estimate $P_{t, A_t}$ upon observing a toss of coin $A_t$.  In particular, the amount by which the point estimate changes ought to depend on the agent's confidence about its current value, which is not encoded in the point estimate.  Beyond that, epistemic uncertainty is also essential to seeking useful data.  The need to represent information beyond point estimates has been emphasized by the Bayesian reinforcement learning literature, as surveyed in \citet{vlassis2012bayesian,ghavamzadeh2012bayesian}.

Despite the aforementioned limitations, an agent can be designed to operate with point estimates as epistemic state.  For example, an agent could select, with probability $1-\epsilon$, the coin with the largest point estimate, and otherwise, sample uniformly.  The point estimate can then be adjusted in response to the observed outcome via an incremental stochastic gradient step that reduces cross-entropy prediction error.  However, because this exploration scheme and update process do not account for epistemic uncertainty, the information ratio is typically much larger than what is achievable.  Analogous shortcomings arise with value function learning agents that operate in more complex environments by incrementally updating point estimates and engaging in $\epsilon$-greedy exploration.


\subsection{Epistemic State Dynamics}
\label{se:epistemic-update}

Even when the epistemic state retains information that can be used to estimate epistemic uncertainty, state dynamics prescribed by the update function $f_{\rm epis}$ influence the regret.  Our regret bound suggests an interpretation of this dependence through the information ratio, which balances between expected shortfall and information gain.  Recall that the latter is quantified by $\I(\target; \environment|P_t) - \I(\target; \environment|P_{t+\tau})$, which is the incremental number of environment-relevant bits retained about $\target$.  A well-designed agent ought to trade off between maximizing information gain and minimizing expected shortfall.

Consider, for example, a linear bandit environment $\environment = (\actions, \observations, \rho)$ for which $\actions \subset \Re^d$, $\observations \subset \Re$, and
$$R_{t+1} = O_{t+1} = \theta^\top A_t + W_{t+1},$$
for an unknown $d$-dimensional vector $\theta$ distributed as $N(\mu_0, \Sigma_0)$ and i.i.d. noise $(W_{t+1}:t=0,1,2,\ldots)$ with $\E[W_t]=0$ and $\Var[W_t]=\sigma^2$.  Let us take $\proxy = \theta$ to be the environment proxy and an optimal action $A_* \in \argmax_{a \in \actions} \theta^\top a$ to be the learning target $\target = A_*$.  If $d$ is not too large, a natural choice of epistemic state is given by $P_t = (\mu_t, \Sigma_t)$, updated via the Kalman filter according to
$$\Sigma_{t+1} = (\Sigma_t^{-1} + A_t A_t^\top / \sigma^2)^{-1}$$
$$\mu_{t+1} = \Sigma_{t+1} (\Sigma_t^{-1} \mu_t + A_t R_{t+1} / \sigma^2).$$
With Gaussian noise, this update retains all relevant information, in the sense that $\Pr(\proxy \in \cdot | P_t) = \Pr(\proxy \in \cdot | H_t)$.  This is not necessarily the case, though it may still hold in some approximate sense, with non-Gaussian noise.

The epistemic state we have described entails maintaining $d^2$ parameters and incurring per-timestep computation scaling with $d^2$.  When $d$ is very large, this becomes infeasible.  One can instead consider alternative epistemic states and dynamics that approximate this process with $O(d)$ parameters.  For example, incremental gain adaptation schemes by \citet{sutton1992gain} can serve this need.  Such epistemic state dynamics decrease information retention relative to the Kalman filter.  However, this reduction may be required due to computational considerations, and the information ratio may remain sufficiently small to yield a reasonably data-efficient agent.

\section{Environment Proxies}

The choice of environment proxy can play a crucial role in guiding epistemic state dynamics to retain some and discard other information.  We will discuss in this section several proxies that are commonly used in reinforcement learning, but before starting, to anchor our discussion of abstract concepts, let us revisit the role played by a proxy in the context of the DQN agent discussed in Section \ref{se:dqn}.

The DQN agent relies on a proxy $\proxy = Q_{\tilde{\theta}}$ that is a neural network representation of an action value function, with weights $\tilde{\theta}$.  One can take $\proxy$ to be the neural network produced by minimizing a loss function that assesses its desirability given the environment $\environment$.  In this case, assuming there is a unique minimizer, $\proxy$ is determined by $\environment$ and can be thought of as a lossy compression.  More generally, one could take $\proxy$ to be what would be produced by some optimization algorithm, for example, stochastic gradient descent, that minimizes loss given data generated from interacting with the environment for, say, a trillion timesteps.  In this case, $\proxy$ may be determined not solely by the environment but also by algorithmic and aleatoric randomness.

\subsection{Policies}
\label{se:target-vs-proxy}

The environment proxy $\proxy$ prioritizes information to retain while the learning target $\target$ prioritizes information to seek.  Suppose that the designer uses a class of policies $\pi_\theta$, parameterized by $\theta$.  The learning target could then be a policy $\target = \pi_{\tilde{\theta}}$ that delivers desirable performance.  One could also consider taking the proxy to be the policy $\proxy=\target=\pi_{\tilde{\theta}}$.  While it may be a suitable choice in some contexts, as we will explain, this can lead to an information ratio far worse than the best achievable because retaining additional information facilitates efficient subsequent learning.

To keep things simple, suppose that an optimal policy $\pi_* = \pi_{\tilde{\theta}}$ is within the parameterized class and taken to be the learning target.
If all and only information about $\pi_*$ is retained by the epistemic state then the agent maintains exactly enough to construct the posterior distribution $\Pr(\pi_* \in \cdot|H_t) = \Pr(\pi_* \in \cdot | P_t)$.  This knowledge is insufficient to efficiently update the epistemic state.  To understand why, let us consider the coin tossing environment in Example \ref{ex:coin-tossing}.

For this environment, it is natural to take the set of situational states to be a singleton, in which case $\pi_*$ and $r$ need not depend on the situational state.  Without loss of generality, consider an optimal policy $\pi_*$ that assigns probability one to the first optimal action $A_* = \min\{a \in \actions: p_a \geq \max_{a' \in \actions} p_{a'}\}$.  Then, the posterior distribution $\Pr(\pi_* \in \cdot|P_t)$ is equivalent to $\Pr(A_* \in \cdot|P_t)$.  With some abuse of notation, we will write $P_t(\cdot) = \Pr(A_* = \cdot|P_t)$.

If the epistemic state $P_t$ retains only enough information to recover this posterior distribution, then it cannot be incrementally updated in a manner that makes effective use of new information. To see this, suppose that $P_t(1) = P_t(2) = 1/2$. If $A_t = 1$ and $O_{t+1} = 1$ then $P_{t+1}(1)$ should be assigned a larger value than $P_t(1)$, but neither epistemic state nor observation can guide magnitude of the difference.  If the posterior distribution $P_t(\cdot)$ is highly concentrated, $P_{t+1}(1)$ should not differ substantially from $P_t(1)$.  On the other hand, if the distribution is diffuse, $P_{t+1}(1)$ can be much larger than $P_t(1)$.  Indeed, $P_t$ is insufficient to retain much new information.

The issue we have highlighted applies to all so-called {\it incremental policy learning} approaches to reinforcement learning, which retain in epistemic state {\it only} information about policies and discard past observations.  Perhaps this is a reason why such approaches are so inefficient in their data usage and are applied almost exclusively to simulated environments.

\subsection{Value Functions}
\label{se:vf-proxy}

Value functions are commonly used as environment proxies.  The most common design, as employed by DQN, involves an approximation $\proxy = Q_{\tilde{\theta}}$, within a parameterized class, to the optimal action value function $Q_*$.  The learning target could be, for example, the proxy itself or a greedy policy with respect to $Q_{\tilde{\theta}}$.  The latter target avoids actively seeking information to refine action value estimates that will not improve decisions.

As discussed earlier, incremental Q-learning algorithms maintain a point estimate of an action value function, though in the interest of data-efficiency, it is important to retain a richer epistemic state.  DQN agents retain a somewhat richer epistemic state through the addition of a replay buffer.  With the right update scheme, this may dramatically improve data efficiency in some simple environments, as demonstrated in \citet{dwaracherla2020langevin} using a variation of Langevin stochastic gradient descent \citep{WelTeh2011a}. However, scaling this approach to address complex environments is likely to require prohibitive increases in memory and computation, as the associated data needs to be stored and repeatedly processed.

An alternative is to use an ensemble of point estimates, as in the simplified ensemble-DQN example from Section \ref{se:dqn} and more sophisticated approaches such as those discussed in \citet{osband2016deep,osband2018randomized}.  In this case, the variation among point estimates reflects epistemic uncertainty.  This approach has proven fruitful in addressing environments that require deep exploration.  One drawback of ensemble-based approaches is that computational demands grow with the number of elements, and because of this, agent designs typically use no more than ten point estimates, which limits the fidelity of uncertainty estimates.  Design of neural network architectures that enable the benefits of large ensembles with manageable computational resources is an important and active area of research \citep{Dwaracherla2020Hypermodels,oswald2021neural}.  As we will discuss further in Section \ref{se:computation}, {\it epistemic neural networks} offer a general framing of such architectures.

The aforementioned approaches can be thought of as approximating the posterior distribution of a value function.  A growing literature develops approaches beyond ensembles to approximating such posterior distributions \citep{o2018uncertainty,o2018variational,Dwaracherla2020Hypermodels} or, alternatively, posterior distributions over temporal-differences \citep{flennerhag2020temporal}.   Approximations to value function posterior distributions have been studied in the more distant past by \citet{engel2003bayes,engel2005reinforcement}.  That work proposed Gaussian representations of posterior distributions that can be incrementally updated via temporal-difference learning.  The computational methods we apply in Section \ref{se:computation} share much of this spirit.

Another form of proxy augmentation entails tracking {\it counts} of state-action pairs sampled in the history used to train a point estimate \citep{NIPS2017_3a20f62a,NEURIPS2018_d3b1fb02}.  Uncertainty in an action value can then be inferred from the associated counts.  While this approach enables deep exploration, one shortcoming is that it does not represent interdependencies across state-action pairs.  A richer epistemic state can capture these interdependencies and enhance data efficiency.

While the use of action value functions as proxies enables more sophisticated and data-efficient behavior relative to policies, it imposes fundamental limitations.  In particular, action value functions predict future rewards, and, as such, retain only information relevant to making such predictions.  Hence, associated agents learn only from rewards.  Observations typically carry much more information that can help an agent learn how to operate effectively. Especially in a complex environment, forgoing this information can hurt data efficiency.  In terms of our regret bound, learning from rich observations can dramatically reduce the information ratio, as it increases information retention and also potentially enables more sophisticated action selection schemes. An example of what additional information to retain can be found in the next section, and a concrete computational example may be found in Section \ref{se:gvf_comp}.

\subsection{General Value Functions}
\label{se:gvf}

General value functions (GVFs) serve as a proxy that can guide an agent to retain useful information from observations beyond rewards \citep{sutton2011horde,pmlr-v37-schaul15}.  To understand their role, first recall that an action value function 
$$Q_\pi(h,a) = \overline{r}_{ah} + \left(P_a \sum_{t=|h|+1}^{T-1} P_\pi^{t-|h|-1} \overline{r}_\pi\right)(h)$$
represents expected cumulative reward from taking action $a$ and then executing policy $\pi$.  This notion can be generalized to discounted value
$$Q_{\pi,\gamma}(h,a) = \overline{r}_{ah} + \gamma \left(P_a \sum_{t=|h|+1}^{T-1} \gamma^{t-|h|-1} P_\pi^{t-|h|-1} \overline{r}_\pi\right)(h),$$
where $\gamma$ is a discount factor in $[0,1]$.  If $\gamma = 1$ then $Q_{\pi,\gamma} = Q_\pi$.  For $\gamma < 1$, this value function emphasizes near term over subsequent rewards.  We can further generalize by replacing the reward function $r$ with a cumulant function $c:\states\times\actions\times\observations\mapsto\Re$, arriving at
$$Q_{\pi,\gamma, c}(h,a) = \overline{c}_{ah} + \gamma \left(P_a \sum_{t=|h|+1}^{T-1} \gamma^{t-|h|-1} P_\pi^{t-|h|-1} \overline{c}_\pi\right)(h),$$
where $\overline{c}$ is to $c$ as $\overline{r}$ is to $r$.
If $c=r$ then $Q_{\pi,\gamma, c} = Q_{\pi,\gamma}$.  Cumulants serve to measure statistics of future observations that are not necessarily reflected by rewards.  They can be thought of as {\it features of the future}.

Consider using as a proxy $\proxy$ a parameterized representation $G_{\tilde{\theta}}: \states\times\actions \mapsto \Re^N$ intended to approximate a vector of $N$ GVFs
$$G_{\tilde{\theta}}(s,a) \approx \left[\begin{array}{c}
Q_{\pi_1,\gamma_1, c_1}(h,a) \\
\vdots \\
Q_{\pi_N,\gamma_N, c_N}(h,a) \\
\end{array}\right],$$
where $\pi$, $\gamma$, and $c$ are vectors of policies, discount factors, and cumulants.  Note that each policy $\pi_n$ could be a random policy like $\pi_*$ or $\pi_\target$, which is initially unknown, or a fixed policy.  Such a proxy can represent much more than an approximation to $Q_*$.  For example, $Q_* = Q_{\pi_*, 1,r}$ could be a single component, while others could reflect statistics that if learned would be helpful to decision-making, either to minimize expected shortfall or maximize information gain.  By using such a proxy, the agent expands the scope of information retained from observations.  As suggested by results of \citet{NIPS2017_1264a061}, which provides a case study that extends a DQN agent using GVFs specialized to the arcade game Ms. PacMan, such a proxy can reduce data requirements by orders of magnitude.  An important area of research in reinforcement learning concerns the design of agents that discover useful general value functions rather than leverage prespecified ones \citep{NEURIPS2019_10ff0b5e}.  In such situations, the proxy can be thought of as encoding, among other things, the vectors $\pi$, $\gamma$, and $c$, which the agent could learn, so that the epistemic state retains information relevant to identifying these objects.

Benefits of GVFs can be studied through the lens of our regret bound.  Well-designed GVFs can dramatically reduce the information ratio.  This may reflect the gains in data efficiency reported in \citet{NIPS2001_1e4d3617}. In Section \ref{se:gvf_comp}, we present simple computational examples that demonstrate some of these benefits.

\subsection{Generative Models}
\label{se:generative-models}

The types of proxies we have discussed -- policies, value functions, and general value functions -- do not enable simulation of agent-environment interactions.  It may sometimes be beneficial to instead take the proxy to be a generative model $\proxy = \tilde{\environment}_{\tilde{\theta}} = (\actions, \observations, \tilde{\rho}_{\tilde{\theta}})$ within a parameterized class.  To keep computational demands manageable, this would typically be much simpler than the true environment $\environment = (\actions, \observations, \rho)$, while still enabling simulation of interactions.  The MuZero agent, for example, is designed around such a proxy \citep{schrittwieser2020mastering}.

Given the class of environments $\tilde{\environment}_{\theta} = (\actions, \observations, \tilde{\rho}_{\theta})$ parameterized by $\theta$, the choice of $\tilde{\theta}$ and thus the proxy $\proxy$ could be specified as a minimizer of a loss function that compares $\tilde{\environment}_{\tilde{\theta}}$ to $\environment$.  The MuZero agent can be viewed as using a proxy based on a loss function that minimizes differences between GVFs associated with $\tilde{\environment}_{\theta}$ versus $\environment$.  Such a loss function may be seen as following the value equivalence principle \citep{GrimmNeuripsValueEquivalence}. With this view, the generative model $\tilde{\environment}_{\tilde{\theta}}$ can be thought of simply as a way of {\rm representing} a proxy comprised of GVFs.

Much remains to be understood about the differing levels of data-efficiency afforded by alternative proxies.  It is possible that generative models will play a critical role in the design of data-efficient agents. The information ratio may offer a useful mechanism for studying the trade-offs.  Further, with a generative model, computationally intensive planning algorithms, such as Monte Carlo tree search \citep{chang2005adaptive,coulom2006efficient,kocsis2006bandit}, can be applied to produce actions.  While this remains to be more fully understood, sentiment arising from the experimental process of designing MuZero suggests that there are benefits to data-efficiency from incorporation of such planning tools \citep{schrittwieser2020mastering}.  It is conceivable that taking such an action generation process to be the learning target and target policy yields a desirable information ratio.


\chapter{Seeking Information}
\label{se:seeking-information}

To learn quickly, an agent must efficiently gather the right data.  This chapter addresses elements of agent design that guide how an agent seeks information.  We first revisit the concept of a learning target.  We then consider issues arising in the design of agents that balance between exploration and exploitation.  Finally, we present information-directed sampling (IDS) as an approach to striking this balance.  We view these concepts and their influence on data-efficiency through the lens of the information ratio and our regret bound.  In fact, IDS is motivated by an intention to minimize the information ratio.

\section{Learning Targets}

While the environment proxy $\proxy$ prioritizes information to retain, the learning target $\target$ prioritizes which information to seek.  One might consider taking the proxy itself to be the learning target.  However, this can sacrifice a great deal of data efficiency.  As an example, consider the many-armed bandit problem described in Section \ref{se:mab}, for which it can be important to seek a satisficing action 
$$\target = \min\left\{a \in \actions: \sum_{o \in \observations} \rho(o|a) r(a,o) \geq \overline{R}_*-\epsilon\right\},$$
rather than an optimal one.  The issue was that the information $\I(\environment; A_*)$ about the environment required to identify the optimal action $A_*$ grows with the cardinality of the action set $\actions$, which could be extremely large.

The satisficing action $\target$ is designed to moderate entropy.  The first plot of Figure \ref{fig:satisficing} presents entropy as a function of the satisficing action parameter $\epsilon$ for a particular distribution over environments.  Each is a multi-armed bandit problem with 5000 arms, with independent mean rewards $\overline{R}_a$ each uniformly distributed over $[0,1]$.  As can be seen from the plot, with $\epsilon=0$, the target and optimal action have identical entropy $\H(\target) = \H(A_*)$, but the entropy of the former decays rapidly as $\epsilon$ increases.  For comparison, we also plot entropy of the optimal action $\H(A_*)$, which does not depend on $\epsilon$.

\begin{figure}[ht]
    \centering
    \includegraphics[scale=0.49]{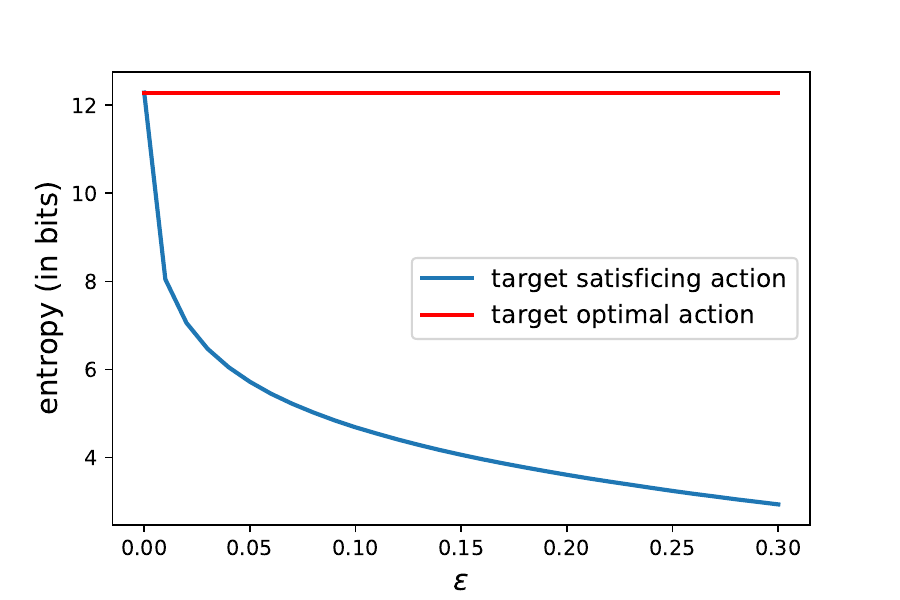}
    \includegraphics[scale=0.49]{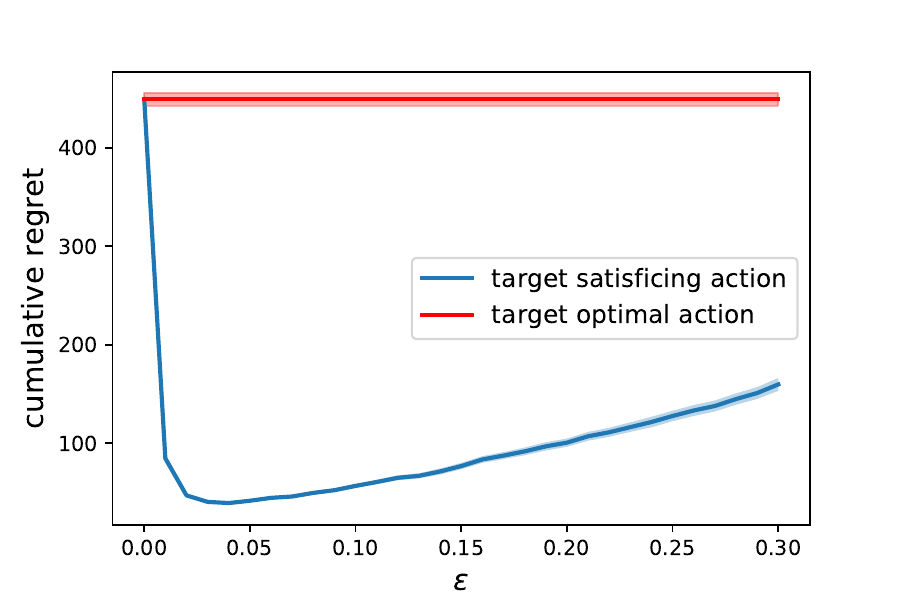}
    \caption{Entropy of the learning target and cumulative regret over 1000 timesteps, as functions of $\epsilon$.}
    \label{fig:satisficing}
\end{figure}

The second plot is of the regret generated by a satisficing Thompson sampling agent after $T=1000$ timesteps, averaged over $200$ independent simulations.  For comparison, we also plot the regret incurred by a Thompson sampling agent with target taken to be the optimal action $A_*$.  The satisficing Thompson sampling agent can be viewed as targeting the satisficing action $\target$.  With $\epsilon = 0$, the learning target is identical to the environment proxy $\proxy = \target$, and thus, regret is identical to that of the Thompson sampling agent.  As $\epsilon$ increases, regret first declines rapidly, but then pivots and grows.  Regret is minimized when $\epsilon$ strikes an optimal balance between $\H(\target)$ and $\Regret(T|\pi_\target)$.

What we have observed in our simple example extends much more broadly.  Rather than an optimal policy, it is often advantageous to target one that is effective but requires fewer bits of information.  The learning target is a tool for prioritizing those bits when the agent explores.  Our example involved a hand-crafted learning target.  Recent work proposes computational methods that automate selection of a learning target to strike a balance between information and regret \citep{arumugam2021deciding}.



\section{Exploration, Exploitation, and the Information Ratio}

A data-efficient agent must select actions that suitably balance between exploration and exploitation.  An optimal balance can be achieved by designing the agent to minimize regret
$$\min_{\pi_{\rm agent}} \Regret(T|\pi_{\rm agent}),$$
subject to bounds on memory and per-timestep computation required to execute $\pi_{\rm agent}$.  In simple special cases, like some multi-armed bandits environments with independent arms, one can tractably solve variations of this problem via Gittin's indices \citep{gittins1974dynamic,gittins1979dynamic}.  More broadly, the problem is typically intractable, and efforts to directly approximate solutions through intensive numerical computation have not led to approaches that scale well to complex environments.

As an alternative to minimizing regret, to simplify the problem, we consider how design decisions influence the regret upper bound established by Corollary \ref{co:regret-bound},
\[\Regret(T|\pi_{\rm agent}) \leq \sqrt{\I(\target; \environment) \sum_{t=0}^{T-1} \Gamma_{\tau,\pi_\target,t}} + \Regret(T|\pi_\target).\]
Choice of the learning target plays an important role here.  A well-designed learning target trades off effectively between regret $\Regret(T|\pi_\target)$ of the target policy, the number of bits $\I(\target; \environment)$ the agent must learn about the environment, and the $(\tau,\pi_\target)$-information ratio $\Gamma_{\tau,\pi_\target,t}$, which reflects how costly it is to learn those bits.  The latter is also heavily impacted by how the agent selects actions that balance between exploration and exploitation.

Recall that the $(\tau,\pi_\target)$-information ratio is given by
$$\Gamma_{\tau, \pi_\target, t} = \frac{\E\left[V_{\pi_\target}(H_t) - Q_{\pi_\target}(H_t,A_t) \right]_+^2}{(\I(\target; \environment|P_t) - \I(\target; \environment|P_{t+\tau}))/\tau}.$$
The numerator is the squared expected shortfall incurred by the agent, while the denominator represents informational benefits over the next $\tau$ timesteps.  When the information ratio is large, the agent pays a steep price per bit.  When it is small, the agent economically learns about the target.  Hence, the information ratio quantifies the agent's efficacy in balancing exploration and exploitation.  Information-directed sampling aims to strike a balance by minimizing variations of the information ratio.

\section{Information-Directed Sampling}

The agent policy $\pi_{\rm agent}$ impacts the above regret bound through the sum $\sum_{t=0}^{T-1} \Gamma_{\tau,\pi_\target,t}$ of information ratios.  Optimization of such a sum appears complex because it requires consideration of interdependencies across time and agent states.  However, the information ratio motivates elegant objectives that, when optimized at each time to produce an immediate action, can strike an effective balance between exploration and exploitation.  Examples of such action selection schemes, which we refer to collectively as {\it information-directed sampling (IDS)}, appeared earlier in Sections \ref{se:mab} and \ref{se:chain}.  In this section, we motivate the design of IDS algorithms, develop versions that address complex information structures, and discuss variations that are amenable to efficient computation.

One might consider optimizing at each time the information ratio $\Gamma_{\tau,\pi_\target,t}$, with $\tau$ chosen to reflect the time horizon over which the agent strategizes about information acquisition.  However, two obstacles arise.  The first is that this involves simultaneously optimizing for each agent state, since $\Gamma_{\tau,\pi_\target,t}$ represents an expectation over possibilities.  This issue can be addressed by instead considering a {\it conditional information ratio} conditioned on the agent state at time $t$
\begin{equation}
\label{eq:conditional-information-ratio}
\frac{\E\left[V_{\pi_\target}(H_t) - Q_{\pi_\target}(H_t, A_t) | X_t\right]_+^2}{\E[\I(\target; \environment|P_t \leftarrow P_t) - \I(\target; \environment|P_{t+\tau} \leftarrow P_{t+\tau}) | X_t] / \tau}.
\end{equation}
The numerator and denominator represent squared expected shortfall and $\tau$-step information gain, each conditioned on the agent state $X_t$.  In the case of bandit environments, for which $\tau=1$ oftentimes represents a suitable choice, this can serve as a practical objective, as has been demonstrated in prior work \citep{russo2014learning,russo2018learning}.  However, for $\tau > 1$, another obstacle arises due to the denominator's dependence on a sequence of future actions.  This introduces coupling between actions at different times, which is what we hope to avoid due to scalability concerns.  Value-IDS, as we will introduce in this section, overcomes this obstacle.

\subsection{Approximating Delayed Information Gain}

One of the obstacles to treating (\ref{eq:conditional-information-ratio}) as an objective for selecting each action stems from dependence of the conditional information gain
$$\E[\I(\target; \environment|P_t \leftarrow P_t) - \I(\target; \environment|P_{t+\tau} \leftarrow P_{t+\tau}) | X_t]$$
on subsequent actions.  To address this, we consider using an approximation that reasons about the future only through the manner in which particular predictions about the future would inform the agent.

Let us begin with the case of a single prediction: the action value function $Q_{\pi_\target}(H_t, A_t)$.  This represents a prediction of future return starting at history $H_t$ if the agent executes $A_t$ and, subsequently, $\pi_\target$.  Note that this prediction depends on the environment $\environment$, not just the epistemic state $P_t$.  As an approximation of conditional information gain, we consider a notion of {\it pseudo-information gain}, given by
$$\I(\target; \environment | X_t \leftarrow X_t) - \I(\target; \environment | A_t, Y_{t+1}, X_t \leftarrow X_t),$$
where $Y_{t+1} = Q_{\pi_\target}(H_t, A_t) + W_{t+1}$  and $W_{t+1}$ is independent noise.  The noise could, for example, be distributed $N(0,\sigma_{\rm noise}^2)$,  for some variance parameter $\sigma_{\rm noise}^2$.  We call $Y_{t+1}$ a {\it pseudo-observation}, as the pseudo-information gain measures information about $\target$ provided by this fictitious observation.  Intuitively, the pseudo-information gain should incentivize the agent to take actions that may lead to information about $Q_{\pi_\target}(H_t, A_t)$ that will be useful for identifying $\target$.

If the environment $\environment$ determines the learning target $\target$, the expression for pseudo-information gain simplifies to
$$\I(\target; A_t, Y_{t+1} | X_t \leftarrow X_t) = \I(\target; \environment | X_t \leftarrow X_t) - \I(\target; \environment | A_t, Y_{t+1}, X_t \leftarrow X_t).$$
While value-IDS can address the more general case where $\target$ may depend on aleatoric or algorithmic uncertainty, to reduce complexity of our exposition in the following sections, we will focus on this simpler case, where $\target$ is determined by $\environment$ and thus only epistemic uncertainty.

Action values $Q_{\pi_\target}(H_t, A_t)$ offer only a narrow view of the environment, and it can be important to broaden this scope.  General value functions offer a tool for accomplishing this.  In particular, given a vector $Q_\dagger$ of general value functions, a more general notion of pseudo-information gain is given by taking the pseudo-observation to be $Y_{t+1} = Q_\dagger(H_t, A_t) + W_{t+1}$, where $W_{t+1}$ is now vector-valued random noise.  Note that $Q_\dagger = Q_{\pi_\target}$ constitutes a special case.  An agent guided by pseudo-information gain from a more informative pseudo-observation can produce sophisticated behaviors that improve data efficiency.  An example is provided in Section \ref{se:computation}.

\subsection{Value-IDS}
\label{se:value-IDS}

We consider generating each action $A_t$ by first computing a vector $\nu$ of action probabilities and then sampling the action according to these probabilities.  Value-IDS selects $\nu$ to optimize a single-timestep objective motivated by the information ratio.  The basic version solves
$$\min_{\nu \in \Delta_\actions} \frac{\E\left[V_{\pi_\target}(H_t) - Q_{\pi_\target}(H_t, \tilde{A}_t) \big| X_t \right]_+^2}{\I(\target; \tilde{A}_t, \tilde{Y}_{t+1} | X_t \leftarrow X_t)},$$
where $\tilde{A}_t$ is sampled from $\nu$ and $\tilde{Y}_{t+1} = Q_\dagger(H_t, \tilde{A}_t) + W_{t+1}$ is the resulting pseudo-observation.  The denominator measures pseudo-information gain, and the objective trades that off against shortfall in order to balance exploration versus exploitation.

Value-IDS optimizes over probability vectors $\Delta_\actions$.  As explained in Appendix \ref{ap:support-cardinality}, the objective is convex, and there is always an optimal solution that assigns positive probabilities to no more than two actions.  Hence, the optimization problem can be solved by iterating over action pairs and, for each pair, carrying out one-dimensional convex optimization.  One can also find an approximate solution within a constant factor of the optimal ratio using $O(|\actions|)$ computation \citep{kirschner2021thesis}.

It is often natural to take the learning target to be a policy $\pi_\target = \target$.  In this case, IDS prioritizes seeking of information about this policy.  Accounting for only information about what the policy would do at the current state yields a lower bound on the associated pseudo-information gain by the data processing inequality:
$$\I(\pi_\target(\cdot|S_t); \tilde{A}_t, \tilde{Y}_{t+1} | X_t \leftarrow X_t) \leq \I(\target; \tilde{A}_t, \tilde{Y}_{t+1} | X_t \leftarrow X_t).$$
Using this lower bound leads to a variant of value-IDS
\begin{equation} \label{eq:value-ids-target-policy}
\min_{\nu \in \Delta_\actions} \frac{\E\left[V_{\pi_\target}(H_t) - Q_{\pi_\target}(H_t, \tilde{A}_t) \big| X_t \right]_+^2}{\I(\pi_\target(\cdot|S_t); \tilde{A}_t, \tilde{Y}_{t+1} |X_t \leftarrow X_t)}.
\end{equation}
It can be useful to consider this variant as a means to simplifying computational requirements.  The version of value-IDS (\ref{eq:basic-value-IDS}) studied in Section \ref{se:chain} is a special case of this for which the noise variance is taken to be zero.

Another possible learning target is a vector of functions $\target = Q_\target$ that approximates $Q_\dagger$.  This requires a procedure that, for each state $s$ and action $a$, can produce the vector $Q_\target(s,a)$ if the proxy $\proxy$ is known.  To digest this concept, it is helpful to discuss one simple use case.  Suppose $Q_*(H_t, A_t)$ depends on $H_t$ only through $S_t$.  Then, we can choose the environment proxy and {\it target value function} $\proxy = \target = Q_\target$ to be an action value function for which $Q_\target(S_t,A_t) = Q_*(H_t,A_t)$.  For this choice of learning target, it is natural to take the target policy $\pi_\target$ to be greedy with respect to $Q_\target$.  This target policy generates optimal actions.

When the learning target is an action value function or vector of target GVFs, value-IDS encourages seeking of information that can improve associated predictions.  This can be particularly valuable in problems where learning how to make these predictions will be useful to the agent in different ways over time.  Such contexts arise, for example, in multi-task and hierarchical reinforcement learning.  For the sake of computational efficiency, it is also sometimes worth considering a variant that accounts only for information about predictions at the current state-action pair:
\begin{equation} \label{eq:value-ids-target-gvfs}
\min_{\nu \in \Delta_\actions} \frac{\E\left[V_{\pi_\target}(H_t) - Q_{\pi_\target}(H_t, \tilde{A}_t) \big| X_t \right]_+^2}{\I(Q_\target(S_t, \tilde{A}_t); \tilde{A}_t, \tilde{Y}_{t+1} | X_t \leftarrow X_t)}.
\end{equation}

\subsection{Variance-IDS}
\label{se:variance-IDS}

Although decisions are decoupled across time, versions of value-IDS based on mutual information such as \eqref{eq:value-ids-target-policy} and \eqref{eq:value-ids-target-gvfs} can require intensive computation.  In order to reduce the time required by an agent to select an action, it can be helpful to use an approximation of mutual information based on variance.  In this section, we present associated versions of IDS based on variance approximation, which we refer to as {\it variance-IDS}. While variance-IDS can greatly simplify computations relative to value-IDS, in its pure form, it requires evaluating conditional expectations with respect to the agent state, which is not tractable except in very simple contexts. In Section \ref{se:belief-IDS}, we will consider a generalization of variance-IDS that instead evaluates expectations with respect to belief distributions that represent approximations of posteriors.

To illustrate the idea of variance approximations, let us first consider a version of value-IDS \eqref{eq:value-ids-target-policy} with pseudo-information gain given by
\[ \I(\pi_\target(\cdot|S_t); \tilde{A}_t, \tilde{Y}_{t+1} | X_t \leftarrow X_t), \]
where $\tilde{A}_t$ is sampled from some distribution $\nu \in \Delta_\actions$ that the IDS agent will minimize over, and $\tilde{Y}_{t+1}$ is a pseudo-observation of $Q_\dagger(H_t, \tilde{A}_t)$. The learning target $\target$ here is a policy $\target = \pi_\target$, and the pseudo-information gain about the target is lower bounded by the information gain about what the agent would do at the current state $\pi_\target(\cdot|S_t)$. 

Now, for simplicity, let us assume that each component of $Q_\dagger$ has a span bounded by $M_1$, and the noise $W_{t+1}$ used to generate pseudo-observation $\tilde{Y}_{t+1}$ has a bounded span $M_2$, where the span of a random variable $X$ is defined as $\mathrm{span}(X) = \mathrm{ess}\sup(X) - \mathrm{ess}\inf(X)$.
Letting
\[ v(a | X_t) = \mathrm{tr} \Big( \mathrm{Cov}\Big[ \E\left[Q_\dagger(H_t, a) | X_t, \pi_\target(\cdot|S_t)])\right] \big| X_t \Big] \Big) \]
for all actions $a \in \actions$ and using results established in Appendix \ref{ap:information-variance}, we can lower bound the pseudo-information gain (in nats) by 
\[ \I(\pi_\target(\cdot|S_t); \tilde{A}_t, \tilde{Y}_{t+1} | X_t \leftarrow X_t) \geq \frac{2}{n(M_1 + M_2)^2} \E[v(\tilde{A}_t | X_t) | X_t], \]
where $n$ is the dimension of the GVF vector $Q_\dagger$. This result can also be extended to subgaussian distributions in addition to bounded random variables.
Up to some scaling factor, the mutual information is lower bounded by the average conditional variance of $Q_\dagger(H_t, \tilde{A}_t)$, averaged over components and conditioned on the agent state $X_t = (\emptyset, S_t, P_t)$ and the target policy $\pi_\target(\cdot|S_t)$ at the current situational state $S_t$. To  intuitively see why this makes sense, note that the conditional variance is large when varying $\pi_\target(\cdot|S_t)$ gives widely different values of $Q_\dagger(H_t, \tilde{A}_t)$, implying that trying to learn about $Q_\dagger(H_t, \tilde{A}_t)$ would likely reveal a lot of information about the target $\pi_\target(\cdot|S_t)$. Applying this lower bound on the mutual information, we obtain a version of variance-IDS, which at each timestep solves for a distribution $\nu$ over actions that minimizes
\begin{equation} \label{eq:variance-approx-target-policy}
\min_{\nu \in \Delta_\actions} \frac{\E\left[V_{\pi_\target}(H_t) - Q_{\pi_\target}(H_t, \tilde{A}_t) \big| X_t \right]_+^2}{\E[v(\tilde{A}_t | X_t) | X_t]}.
\end{equation}

A similar variance approximation can be applied to the version of value-IDS specified in \eqref{eq:value-ids-target-gvfs}. This version of value-IDS prioritizes seeking information about $Q_\target$, an approximation to the GVFs $Q_\dagger$. Using results from Appendix \ref{ap:information-variance}, we can lower bound the mutual information in \eqref{eq:value-ids-target-gvfs} by 
\begin{align*}
& \I(Q_\target(S_t, \tilde{A}_t); \tilde{A}_t, \tilde{Y}_{t+1} | X_t \leftarrow X_t) \\
& \qquad \geq \frac{2}{n(M_1+M_2)^2} \E \left[ \mathrm{tr}\left(\mathrm{Cov}\left[Q_\dagger(H_t, \tilde{A}_t) \big| X_t, \tilde{A}_t \right]\right) | X_t \right].
\end{align*}
Thus, a version of variance-IDS that approximates \eqref{eq:value-ids-target-gvfs} is given by
\begin{equation} \label{eq:variance-approx-target-gvfs}
\min_{\nu \in \Delta_\actions} \frac{\E\left[V_{\pi_\target}(H_t) - Q_{\pi_\target}(H_t, \tilde{A}_t) \big| X_t\right]_+^2}{\E \left[ \mathrm{tr}\left(\mathrm{Cov}\left[Q_\dagger(H_t, \tilde{A}_t) \big| X_t, \tilde{A}_t \right]\right) | X_t \right]}.
\end{equation}

Similar to value-IDS, the objectives of variance-IDS are also convex, and as established in Appendix \ref{ap:support-cardinality}, there is always an optimal solution that assigns positive probabilities to no more than two actions.

\subsection{Belief Distributions and Computation}
\label{se:belief-IDS}

In the previous section, we have simplified the computation of value-IDS with variance-based approximations. However, objectives such as \eqref{eq:variance-approx-target-policy} and \eqref{eq:variance-approx-target-gvfs} are still typically not computationally tractable, as the expected one-step regret and information gain depend on the distribution of $\environment$ conditioned on the agent state, which is expensive to compute. In this section, we introduce a further approximation based on belief distributions that are not necessarily the exact conditional distributions.

For the purpose of this section, let us focus on the following setting. Suppose we take the proxy $\proxy = \tilde{Q}_\dagger$ to be an approximation to $Q_\dagger$, which we assume includes the optimal action value function, and the target $\target$ to be a greedy policy with respect to the approximate optimal action value function. It is natural for the proxy $\tilde{Q}_\dagger$ to depend on the situational state $S_t$ rather than history $H_t$.

In variance-IDS \eqref{eq:variance-approx-target-policy} and \eqref{eq:variance-approx-target-gvfs}, instead of assessing the conditional expectations exactly, we could have a mechanism that, based on the epistemic state, produces beliefs of proxy $\tilde{Q}_\dagger$ which represent approximations to the conditional distribution of $\tilde{Q}_\dagger$. For example, in Section \ref{se:enn}, we will introduce one possible incremental method for encoding reasonable belief distributions via epistemic neural networks and temporal-difference style updates. The belief distribution is trained to be consistent, in some approximate sense, with the prior distribution and past observations. We will use belief distributions to approximate conditional regret and information gain.

With some abuse of notation, let us use $P_t(\cdot)$ to denote such belief probabilities over the proxy given epistemic state $P_t$. If the belief distribution does perfect inference, $P_t(\cdot) = \Pr(\proxy \in \cdot | X_t)$, although in general it will likely be an approximation. Conditioned on the agent state, let $\hat{Q}_{\dagger, t}$ be an independent sample drawn from $P_t(\cdot)$, $\hat{Q}_{*,t}$ be the optimal action value function which is a component of $\hat{Q}_{\dagger, t}$, and $\hat{\pi}_t$ be the greedy policy with respect to $\hat{Q}_{*,t}$. A variant of variance-IDS that approximates \eqref{eq:variance-approx-target-policy} is given by solving for an action distribution $\nu$ that minimizes
\begin{equation} \label{eq:variance-ids-target-policy}
\min_{\nu \in \Delta_\actions} \frac{\E\left[\max_{a \in \actions}\hat{Q}_{*,t}(S_t, a) - \hat{Q}_{*,t}(S_t, \tilde{A}_t) \,\big|\, X_t \right]^2}{\E \left[ \mathrm{tr}\left(\mathrm{Cov}\left[\E\left[\hat{Q}_{\dagger,t}(S_t, \tilde{A}_t) | X_t, \tilde{A}_t, \hat{\pi}_t(\cdot|S_t) \right] \big| X_t, \tilde{A}_t \right]\right) | X_t \right]}.
\end{equation}
Comparing \eqref{eq:variance-ids-target-policy} to \eqref{eq:variance-approx-target-policy}, note that first, we approximate the joint conditional distribution of $Q_\dagger$ and target policy $\pi_\target$ with belief over the proxy $\tilde{Q}_\dagger$ and target policy. Further, we approximate the conditional distribution of the action value function of the target policy, $Q_{\pi_\target}$, with belief over $\tilde{Q}_*$.
With a mechanism that can generate samples of $\hat{Q}_{\dagger, t}$ from the belief based on epistemic state $P_t$, we can approximately optimize this objective through Monte Carlo approximation. We will describe the sample-based algorithm in Section \ref{se:computation}. 

Similarly, a variant of variance-IDS that approximates \eqref{eq:variance-approx-target-gvfs} is given by
\begin{equation} \label{eq:variance-ids-target-gvfs}
\min_{\nu \in \Delta_\actions} \frac{\E\left[\max_{a \in \actions}\hat{Q}_{*,t}(S_t, a) - \hat{Q}_{*,t}(S_t, \tilde{A}_t) \,\big|\, X_t \right]^2}{\mathrm{tr}\left(\mathrm{Cov}\left[\hat{Q}_{\dagger,t}(S_t, \tilde{A}_t) \big| X_t\right]\right)}.
\end{equation}
We will similarly discuss a sample-based algorithm in Section \ref{se:computation}. 

Finally, as with our earlier versions of variance-IDS, the objectives for approximate variance-IDS are convex and there is always an optimal solution that has support two.

\chapter{Computational Examples}
\label{se:computation}

In earlier chapters, we have outlined abstract concepts that can inform the design of information seeking agents.
In this chapter, we show that these concepts can be put to work in practical and scalable reinforcement learning agents.
The insights afforded by our analysis are supported by the computational results in two key ways:
\begin{enumerate}
	\item Proxy design can have a large impact on data efficiency (via information retention).
	\item Agents that actively seek out useful information can be more data efficient.
\end{enumerate}

To show this, we present a series of didactic experiments designed to test key predictions of our theory.
We attempt to provide a simple and clear demonstration of not only the performance on any one problem, but also the \textit{scaling} as we vary agent and environment complexity.
In each case, elements of our environment specification are stylized and simplified, so that we can provide clear insight into elements of reinforcement learning agent design.
However, it is important to stress that we only do this to more clearly demonstrate insights beyond the particular experiments in this paper.
We focus on simplified domains precisely because we are interested in problem domains \textit{beyond} the scope of less efficient learning algorithms.\footnote{This effort mirrors the `behaviour suite' for RL \citep{Osband2020Behaviour}, which aims to test general RL capabilities, but with a targeted focus on the issues raised in this paper. }

We build on prior work of \citet{Dwaracherla2020Hypermodels}, who demonstrate that IDS can effectively leverage hypermodels for more efficient exploration.
Related work has previously demonstrated that an IDS-inspired variant of DQN can improve scores in Atari games \citep{nikolov2019information}.
However, their adaptation to deep Q-networks, especially the epistemic state design and action selection scheme, does not fully leverage the information-seeking behavior that IDS enables, which will be discussed in greater details in the rest of the section. First, we introduce a practical, information seeking agent that brings together the concepts we have discussed in this tutorial.

\section{A Practical Information Seeking Agent}

This section presents a concrete sample-based implementation of an information-seeking agent that extends common building blocks of deep reinforcement learning.
These include: neural network function approximation, incremental learning via SGD and other components
from DQN.
Our example agent is:
\begin{itemize}
    \item \textbf{Minimal}: the key concepts are demonstrated in the simplest possible manner.
    \item \textbf{General}: the agent can operate in any environment.
    \item \textbf{Performant}: the agent effectively seeks, retains, and utilizes information.
\end{itemize}

We specify the agent through the agent state, and the action selection policy $\pi(\cdot | X_t).$
Agent state dynamics are determined by the update rules ($f_{\rm algo}$, $f_{situ}$, $f_{epis}$) introduced in Equations \eqref{eq:algorithmic-state-dynamics}, \eqref{eq:aleatoric-state-dynamics}, and \eqref{eq:epistemic-state-dynamics}.
In the interests of space and clarity, we defer full algorithmic details to Appendix \ref{app:computation}, although we make sure to highlight the key choices in our main text.
Existing tools and notation in RL agent design are mostly focused on algorithmic state and situational state updates, and so we can make use of standard tools in this field.
In order to facilitate an elegant treatment of epistemic state updates we briefly introduce the notion of an \textit{epistemic neural network}.

\section{Epistemic Neural Networks}
\label{se:enn}

An efficient agent needs to represent its state of knowledge, its epistemic state.
The coherent use of Bayes' rule and probability theory are the gold standard for updating beliefs, but exact computation quickly becomes infeasible for even simple problems.
Modern machine learning has developed an effective toolkit for learning in high-dimensional spaces including temporal-difference learning, neural net function approximation and SGD learning.
Together with these tools, is a relatively standard notation that for inputs $x_i \in \Xc$ and outputs $y_i \in \Yc$ we train a function approximator $f_\theta(x)$ to fit the observed data $D = \{(x_i, y_i) \text{ for } i=1,...,N\}$ based on a loss function $\ell(\theta, D) \in \R.$
However, there is no consistent convention as to which elements of the parameters $\theta$ correspond to epistemic components of uncertainty.

To help make this distinction clear, we introduce the concept of an \textit{epistemic index} $z \in \Ic \subseteq \Re^{n_z}$ distributed according to some reference distribution $p_z(\cdot)$, and form the augmented \textit{epistemic} function approximator $f_\theta(x, z)$.
Where the function class $f_\theta(\cdot , z)$ is a neural network we call this an \textit{epistemic neural network} (ENN).
The index $z$ is controlled by the agent's \textit{algorithmic} state, but allows us to unambiguously identify a corresponding function value.

On some level, ENNs are purely a notational convenience and all existing approaches to dealing with uncertainty in deep learning can be rephrased in this way.
For example, an ensemble of point estimates $\{\tilde{f}_{\theta_1}, .., \tilde{f}_{\theta_K}\}$ can be viewed as an ENN with $\theta = (\theta_1,..,\theta_K)$, $z \in \{1,..,K\}$, and $f_\theta(x,z) := \tilde{f}_{\theta_z}(x)$.
Similarly, for a base network $\tilde{f}_{\tilde{\theta}}(x)$ where the weights $\tilde{\theta}$ are given by some hypernetwork $g_\theta(z)$, $z \sim N(0, I)$, we can write $f_\theta(x, z) := \tilde{f}_{g_\theta(z)}(x)$.

However, this simplicity hides a deeper insight: that the process of epistemic update \textit{itself} can be tackled through the tools of machine learning typically reserved for point estimates, through the addition of this epistemic index.
Further, since these machine learning tools were explicitly designed to scale to large and complex problems, they might provide tractable approximations to large scale Bayesian inference even where the exact computations are intractable.

\vspace{0.1in}
\noindent{\bf Epistemic Q-learning}

To illustrate this point, we consider a simple modification of DQN designed to maintain estimates of uncertainty about action values.
Let $f_\theta(s, z)$ be an ENN representation of action values in $\R^\Ac$ and use pythonic index $f_\theta(s, z)[a]$ for the value of action $a$.
We can then define a Q-learning loss,
\begin{equation}
    \label{eq:q_learning_enn}
    \ell^{Q, \gamma}\left(\theta; f, \theta^-, z, (s, a, r, s')\right) := \left(f_\theta(s, z)[a] - (r + \gamma \max_{a' \in \actions} f_{\theta^-}(s', z)[a'] ) \right)^2,
\end{equation}
where $\gamma$ is a discounting hyperparameter used by the algorithm. Note that \eqref{eq:q_learning_enn} is the regular Q-learning update for a single epistemic index $z$.

To specify the epistemic update we first identify the epistemic state $P_t$ through parameters $\theta_t$ and replay buffer $B_t$.
For any loss function 
\mbox{$\ell \left(\theta; f, \theta^-, z, (s, a, r, s')\right)$}, 
we provide a sample-based approach to the epistemic update based on sampling $n_{\rm batch}$ transitions from the replay buffer and $n_{\rm index}$ epistemic indices from the reference distribution and taking a gradient step with step size $\alpha$:
\begin{equation}
    \label{eq:average_enn_learn}
    \theta_{t+1} \leftarrow \theta_t - 
    \alpha \nabla_\theta \left( \frac{1}{n_{\rm index} \times n_{\rm batch}}
    \sum_{i=1}^{n_{\rm index}} 
    \sum_{j=1}^{n_{\rm batch}}
    \ell \left(\theta; f, \theta_t, z_i, (s, a, r, s')_j\right)
    \right) \Bigg|_{\theta = \theta_t}.
\end{equation}

The learning rule in \eqref{eq:average_enn_learn} has been studied in the specific context of an ensemble ENN for $\ell = \ell^{Q,\gamma}$ in the algorithm `bootstrapped DQN' \citep{osband2016deep}.
Importantly, and unlike exact Bayesian inference, this approach is naturally compatible with modern deep learning techniques.
Our next example shows that this simple learning rule can produce reasonable epistemic updates in certain cases. 
Developing reliable and scalable approximate inference algorithms in general is an active area of research and an important component of data-efficient agent design.

\vspace{0.1in}
\noindent{\bf Uncertainty estimates in an ensemble ENN}

Figure \ref{fig:value_concentration} shows the epistemic evolution of an agent trained by \eqref{eq:average_enn_learn} using an ENN comprised of an ensemble of 10 $\times$ [50, 50]-MLPs with randomized prior functions \citep{osband2018randomized} in two environments.
Action selection is taken through information-directed sampling as are described in \ref{se:sample_ids_alg}, but for now it is sufficient to note a few high level observations, which we will then support in more detail.
In each plot, the dashed lines represent the true action values, while the agent's mean beliefs of the action values are given by the solid lines with 2 standard deviations in the shaded region.

Figure \ref{fig:bandit_confidence} shows the agent's action value estimates in a simple three-armed bandit, where each action returns an observation equal to its mean reward plus $N(0, \sigma^2\mkern1.5mu{=}\mkern1.5mu0.01)$ noise.
This problem is trivially amenable to conjugate Bayesian posterior updates, and has no delayed consequences, so the use of ENNs is not really warranted.
In Figure \ref{fig:deep_sea_value}, the environment is described by a difficult 210-state MDP (DeepSea size=20 in \citet{osband2016deep}) and we examine the value estimates in the initial state for each episode.
In this context, exact Bayesian inference would necessitate an explicit model and repeated application of dynamic programming, whereas our ENN implementation provides a constant (and far lower) cost per timestep.
Reassuringly, in both cases we observe the desired phenomena: initially the agent's posterior is diffuse, but as the agent gathers more data, the estimates concentrate around the true value.
Further, once the agent is reasonably confident about the optimal action it does not significantly invest in reducing uncertainty of other actions, consistent with efficient information-seeking behaviour.
The next subsection describes how the agent accomplishes this via targeted action selection.

\begin{figure}[!ht]
\centering
\begin{subfigure}{.5\textwidth}
  \centering
  \includegraphics[width=.99\linewidth]{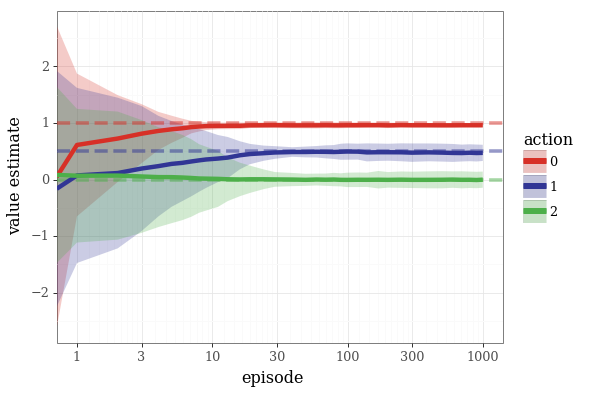}
  \caption{ENN posterior in a simple bandit problem.}
  \label{fig:bandit_confidence}
\end{subfigure}%
\begin{subfigure}{.5\textwidth}
  \centering
  \includegraphics[width=.99\linewidth]{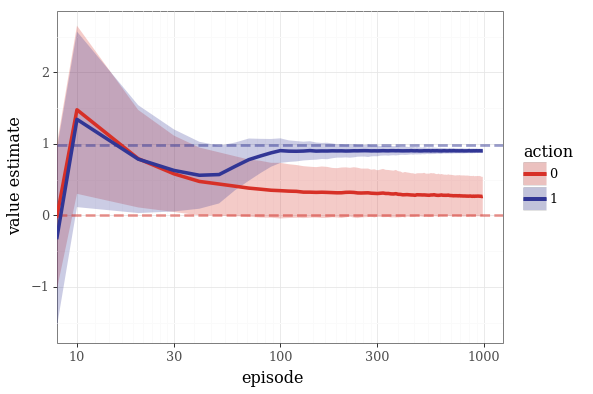}
  \caption{ENN posterior in DeepSea20 initial state.}
  \label{fig:deep_sea_value}
\end{subfigure}
\caption{ENNs are capable of maintaining epistemic state in an effective manner.}
\label{fig:value_concentration}
\end{figure}

\subsection{Sample-Based Action Selection}
\label{se:sample_ids_alg}

We specialize our discussion to feed-forward variants of DQN with aleatoric state $S_t = O_t$ and epistemic state $P_t = (\theta_t, B_t)$.
Here $\theta_t$ represents parameters of an ENN $f$ and $B_t$ an experience replay buffer.
The epistemic state is updated according to \eqref{eq:average_enn_learn} for $\theta_t$ and the first-in-first-out (FIFO) rule for $B_t$.
To complete our agent definition we need to define the action selection policy from agent state $X_t = (Z_t, S_t, P_t)$.
With this notation we can concisely review three approaches to action selection:
\begin{itemize}
    \item \textbf{$\epsilon$-greedy}: algorithmic state $Z_t = \emptyset$; select $A_t \in \argmax_{a} f_{\theta_t}(S_t, Z_t)[a]$ with probability $1-\epsilon$, and a uniform random action with probability $\epsilon$ \citep{mnih-atari-2013}.
    \item \textbf{Thompson sampling (TS)}: algorithmic state $Z_t = Z_{t_k}$, resampled uniformly at random at the start of each episode $k$; select $A_t \in \argmax_a f_{\theta_t}(S_t, Z_t)[a]$ \citep{osband2016deep}.
    \item \textbf{Information directed sampling (IDS)}: algorithmic state $Z_t = \emptyset$; compute action distribution $\nu_t$ (with support 2) that minimizes a sample-based estimate of the information ratio given by \eqref{eq:variance-ids-target-policy} with $n_{\rm IDS}$ samples; sample action $A_t$ from $\nu_t$.
\end{itemize}

We use $\epsilon$-greedy as the classic dithering approach used in DQN, as well as many `deep RL' baselines.
Thompson sampling represents an approach to \textit{deep exploration}, which can be exponentially more efficient than dithering in some environments.
The remainder of this section will focus on demonstrating that:
(1) IDS provides another approach to action selection equally capable at deep exploration as TS, and 
(2) IDS can be exponentially more efficient than existing approaches (including TS) where there are benefits to information-seeking behaviour.
\hspace{-2mm}\footnote{Appendix \ref{app:bsuite-report} provides a full evaluation of all three agents on the general RL benchmark: `the behaviour suite for reinforcement learning' \citep{Osband2020Behaviour}.}

\section{Information Seeking via Variance-IDS}
\label{se:efficient-var-ids}

This section focuses on a like for like comparison of two DQN variants with ENN knowledge representation, but contrasting TS with variance-IDS action selection.
For convenience, we will refer to variance-IDS as simply IDS for the rest of this section.
In both cases we use an ensemble of size 20 with 50-50-MLPs and randomized prior functions, simple replay buffer of 10k transitions, Q-learning according to \eqref{eq:average_enn_learn} with $n_{\rm batch}=128, n_{\rm index}=20$, $\ell=\ell^{Q,\gamma}$ and ADAM learning update rate 0.001 \citep{osband2018randomized,kingma2014adam}; for IDS we set $n_{\rm IDS}=40$.
Our first experiments consider a special family of environments designed to test for \textit{deep exploration}.

\subsection{Deep Exploration}

One of the key challenges in RL is that of \textit{deep exploration}, which accounts not only for the information gained from an action, but also how that action may position the agent to more effectively gather information over subsequent time periods \citep{osband2019deep}.
The DeepSea experiment presents a sequence of environments parameterized by a size $N$ and a random seed.  Agents that explore via dithering (e.g., $\epsilon$-greedy or Boltzmann exploration) take in expectation more than $2^N$ episodes to learn the optimal policy.
To investigate the empirical scaling of our computational approximations we use the standardized DeepSea experiment from \bsuite\ \citep{Osband2020Behaviour}.

\begin{figure}[htb]
\begin{center}
\includegraphics[width=0.9\linewidth]{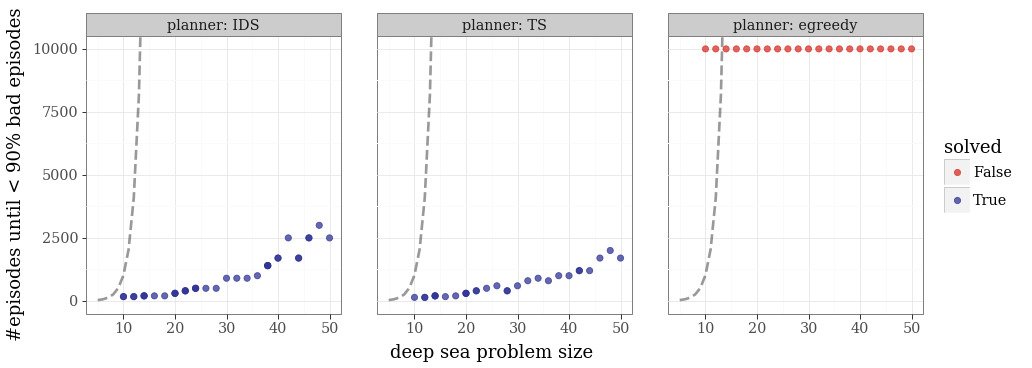}
\caption{Empirical scaling in DeepSea, both IDS and TS demonstrate deep exploration.}
\label{fig:deep_sea_scaling}
\end{center}
\end{figure}

Figure \ref{fig:deep_sea_scaling} shows the number of episodes required by each agent in order to have fewer than 90\% suboptimal episodes as the problem size $N$ grows.
The dashed line shows the scaling $2^N$ and each experiment is run for only 10k episodes.
Empirically, this demonstrates that value-IDS can drive deep exploration in an environment that has been challenging to other agents.

The above simulation results with the DeepSea environment demonstrate a case of successful deep exploration, competitive with TS.  However, the main motivation for IDS is not to match TS in deep exploration, but instead to more efficiently seek information where the opportunity arises, as our next example will show.

\subsection{Targeting Informative Actions}
\label{se:sparse_bandit}

Thompson sampling, like other optimistic approaches, restricts its attention to actions which have nonzero probability in some optimal policy.
In many cases this heuristic is effective, even in some sense optimal \citep{russo2018tutorial}.
However, in some environments this approach leaves a lot of room for improvement,
for example, where there are potentially-informative actions which have no chance of being optimal.
For this experiment we consider a simple environment with this characteristic: a sparse linear model.
We show that IDS can take advantage of the informative actions and perform exponentially better than optimistic approaches.

Consider an environment parameterized by $\phi_* \in \R^N$ with action set $\actions = \left\{a_1, .., a_A\right\} \subset \R^N$ and observation set $\observations = \R$.
Let the observations be deterministic given an action with $O_{t+1} = A_t^\top \phi_*$.
Additionally $\phi_*$ is known to be one-sparse, i.e., $\|\phi_*\|_0 = 1$.
Assume $N=2^d$ and let the action space to be composed of the one-hot vectors $\Ac^{\rm opt} = \{e_1,..,e_N\}$ together with binary search `probes' $\Ac^{\rm probe} = \left\{\frac{x}{2} \mid x \in \Bc_N \right\}$ where $\Bc_N$ is the set of indicator sub-vectors that can be obtained by bisecting the $N$ components \citep{russo2016information,Dwaracherla2020Hypermodels}.
Any TS agent will only select actions from $\Ac^{\rm opt}$, since only these actions have some probability of being optimal.
However, an information-seeking agent may prefer to choose from $\Ac^{\rm probe}$, since it allows the agent locate the unknown optimal arm in $O(\log(N))$ rather than $O(N)$ steps.

\begin{figure}[!ht]
\centering
\begin{subfigure}{.5\textwidth}
  \centering
  \includegraphics[width=.99\linewidth]{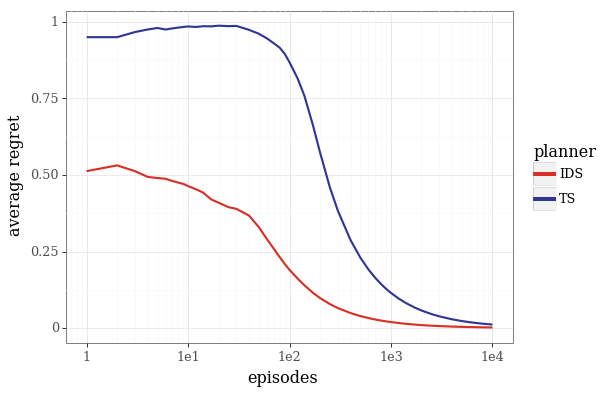}
  \caption{Regret in the bandit problem with 128 arms}
  \label{fig:sparse_bandit_regret}
\end{subfigure}%
\begin{subfigure}{.5\textwidth}
  \centering
  \includegraphics[width=.99\linewidth]{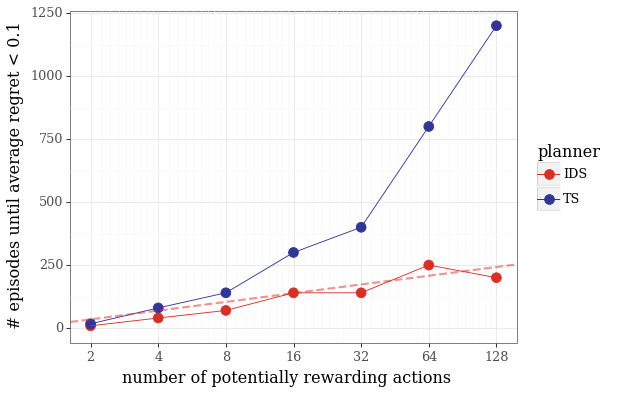}
  \caption{IDS scales better to large number of arms}
  \label{fig:sparse_bandit}
\end{subfigure}
\caption{Contrasting TS and IDS on sparse bandit}
\label{fig:sparse_bandit_all}
\end{figure}

Figure \ref{fig:sparse_bandit_all} shows the results of applying IDS and TS to a series of these environments as the number of potentially optimal actions $N$ grows.
We encode the prior knowledge of sparsity in the ENN architecture by learning an ensemble of 10 logits estimates over $N$ possible components.
Both agents learn by SGD and a cross-entropy loss against observed data.\footnote{Consult Appendix \ref{app:computation} for implementational details.}
Figure \ref{fig:sparse_bandit_regret} shows the regret through time with dimension $N=128$, averaged over 100 seeds.
IDS greatly outperforms TS, since it is able to prioritize informing binary search before settling to the optimal arm.
Figure \ref{fig:sparse_bandit} shows the empirical scaling as we grow the problem dimension $N$, and we see that this difference becomes more pronounced as the problem size grows.
The dashed line shows a perfect logarithmic scaling in $N$, which we see IDS matches well empirically.
These results show how the information seeking behavior enabled by IDS can greatly improve data efficiency, even when the agents match in the environment proxy (action value function) and use the same update rules.
Our next results show that there are settings where agents that maintain a richer epistemic state involving general value functions can greatly outperform those that do not.

\section{Variance-IDS with General Value Functions}
\label{se:gvf_comp}

This paper highlights the role of the \textit{environment proxy} and the \textit{learning target} in the design of an efficient agent.
The results in Section \ref{se:efficient-var-ids} show that IDS can drive efficient exploration where both the proxy and learning target are given by the action value function.
In this section, we show that general value functions (GVFs) can facilitate trade-offs in agent design, by controlling what information is retained (environment proxy) and what information is sought out (learning target).

\subsection{Retaining Information: The Environment Proxy}

A limited agent has to make trade-offs about what information to retain.
Naively, an agent interested only in maximizing cumulative reward, might disregard all observations not directly pertinent to the value function.
However, as we will show, agents that judiciously expand their attention to general value functions can benefit from large gains in efficiency.

Consider a simple environment with an action set $\Ac \subset \R^{d \times K}$ of size $N$, and $\observations = \{0, 1\}^K$.
The environment is parameterized by $\phi_* \in \R^d$, with observation probability function given componentwise as $\rho_{\phi_*}(o_i=1|a) = \left(1 + \exp(-\left<a_i, \phi_* \right>)\right)^{-1}$, i.e., each element of the observation vector $O_t = (O_{t,0},..,O_{t,K-1})$ is an independent binary observations.
This environment can be thought of as a logistic bandit with $O_{t,0}$ as the reward, and $O_{t,i}$ as auxiliary observations for $i > 0$.
In this setting, it may be beneficial to retain information from observation $O_{t,i}$ for $i > 0$, since they provide additional signal related to the unknown parameter $\phi_*$.

We consider a variant of DQN that extends to general value functions, which, in this case, are simply trying to predict observations $O_{t,0}, \dots, O_{t, K-1}$.
The agent is based around an ENN with replay buffer $B_t$, and we define the loss over the first $k$ GVFs to be
\begin{equation}
    \label{eq:log_loss}
    \ell^k(\theta ; f, z, o, a) = \sum_{j=1}^k \frac{o_j + (1 - o_j) \exp(-\left<a_j, f_\theta(1, z) \right>)}{1 + \exp(-\left<a_j, f_\theta(1, z) \right>)}.
\end{equation}
As in the case for Q-learning, we train this ENN via minibatch sampling, together with averaging over epistemic indices $z$.
Our agent implements sample-based variance IDS with $n_{\rm index}=100$ samples and uses a hypermodel ENN trained with perturbed SGD as in \citet{Dwaracherla2020Hypermodels}.\footnote{The results are almost identical when learning with an ensemble ENN.}
We perform our experiment for $D=N=100$ and average each evaluation over 100 random seeds.

\begin{figure}[!ht]
\centering
\begin{subfigure}{.5\textwidth}
  \centering
  \includegraphics[width=.99\linewidth]{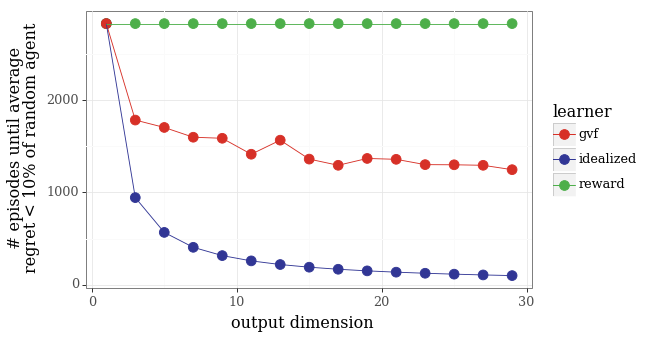}
  \caption{Learning scaling in logistic bandit}
  \label{fig:logistic_scale}
\end{subfigure}%
\begin{subfigure}{.5\textwidth}
  \centering
  \includegraphics[width=.99\linewidth]{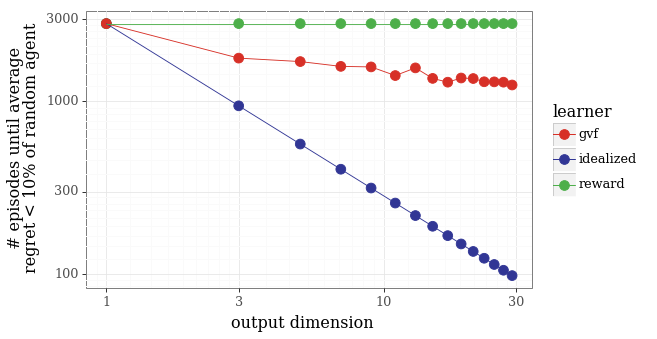}
  \caption{Scaling on the log scale}
  \label{fig:logistic_scale_log}
\end{subfigure}
\caption{Environment proxy can impact information retention.}
\label{fig:logistic_bandit}
\end{figure}

Figure \ref{fig:logistic_bandit} shows the learning time as we vary the number of observation elements used in updating the epistemic state $k=1,..,K$.
We see that the agent can learn significantly faster by learning from more auxiliary observations in addition to the reward compared to learning from reward alone.
However, ingesting $k$ times more observations does not lead to $k$-times faster learning, as shown by the gap between the GVF scaling and the idealized $\frac{1}{k}$ relationship.
This simple example, perhaps unsurprisingly, shows that the use of GVFs can improve data efficiency by improving information retention. The insight ought to carry over to more complex environments.
In the next section, we show that the gains from GVFs as learning \textit{targets} can be even more pronounced.

\subsection{Seeking Information: The Learning Target}

An efficient agent should actively seek information, and it must also be judicious about which information to seek.
The learning target allows an agent to prioritize the information that it seeks.
So far, all of our agents have prioritized information relating to the optimal action $A_*$.
In this section, we will highlight the great potential benefits that taking general value functions as the learning target can afford.

Consider a very simple bandit problem with $N$ arms exactly one of which is rewarding.
The environment is similar to the `sparse bandit' of Section \ref{se:sparse_bandit}, except rather than additional `binary probes' there is a distinguished $N$+1$^{th}$ action that provides no reward, but instead reveals $\phi_*$, the index of the rewarding action, as part of the observation.
Clearly, for large $N$, the optimal policy is to select action $N+1$, find out the rewarding action and then act optimally subsequently.
However, only an agent that actively seeks information would ever select the revealing action, which has no possible reward in itself.

To anchor our conversation we once again consider agents based around DQN with an ensemble ENN that extends to general value functions.
In each case, the environment proxy $\tilde{\Ec}$ is a generalized value function equal to the full observation=(reward, reveal).
Agents learn via SGD on observed data, using a log-loss on their posterior probability against observed data, and their epistemic updates are identical.
The agents only differ in terms of their action selection and, in turn, through their learning target:
\begin{itemize}
    \item \texttt{TS}: Thompson sampling for the optimal action.
    \item \texttt{IDS\_Q}: Sample-based variance-IDS \eqref{eq:variance-ids-target-gvfs} with learning target $\tilde{Q}_*$ (action value only).
    \item \texttt{IDS\_GVF}: Sample-based variance-IDS  \eqref{eq:variance-ids-target-gvfs} with learning target $\tilde{\Ec}$ (full GVF).
\end{itemize}

That Thompson sampling should fail here is perhaps unsurprising: it only selects actions that have some probability of being optimal, and action $N+1$ can never be optimal.
variance-IDS (Equation \ref{eq:variance-ids-target-gvfs}) fails when the learning target is $Q_*$ because the approximation in information gain (based on the variance of $\tilde{Q}_*$ at action $N+1$, which is 0) fails to account for the information gain accrued from the revealing observation not \textit{directly} tied to value.
However, a variant of IDS which includes the revealing observation within the GVF can, and does, prioritize the action $N+1$.
Figure \ref{fig:inform_bandit_log} shows the learning scaling as we increase the number of actions $N$ together with a solid line of $y=x$ to show scaling.

\begin{figure}[!ht]
\centering
\begin{subfigure}{.5\textwidth}
  \centering
  \includegraphics[width=.99\linewidth]{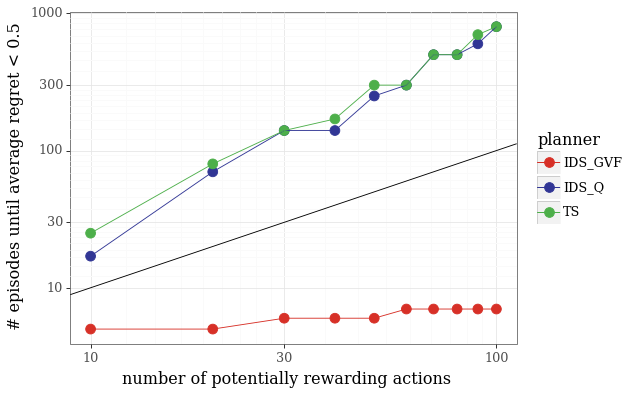}
  \caption{Learning time in the informative bandit}
  \label{fig:inform_bandit_log}
\end{subfigure}%
\begin{subfigure}{.5\textwidth}
  \centering
  \includegraphics[width=.99\linewidth]{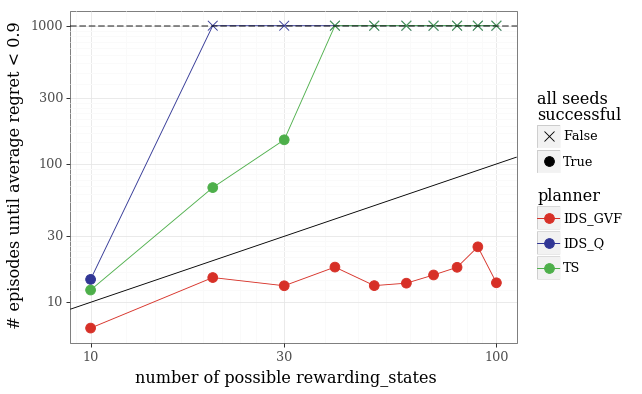}
  \caption{Learning time in the informative chain}
  \label{fig:chain_log}
\end{subfigure}
\caption{Learning targets with GVFs can actively seek out informative states and actions.}
\label{fig:information_action}
\end{figure}

Figure \ref{fig:chain_log} presents a similar result in an RL setting.
The environment shares a similar spirit as the informative bandit problem, but with delayed consequences. The environment is similar to the chain described in Figure \ref{fig:chain}, except that exactly one of $r_0, \dots, r_{\tau-1}$ is $1$ while all the rest are $0$, and that if the agent reaches state $\tau-1$, there will be an additional observation revealing the index of the nonzero reward.
Compared to the bandit problem, this environment presents the additional challenge that an agent seeking the revealing information must \textit{commit} to its plan of action over multiple timesteps.

Nevertheless, we see exactly the same phenomenon as the bandit setting.
We take the cumulants of GVFs to be equal to the observation vector.
Only IDS with GVFs as the learning target is able to effectively prioritize the revealing state, and uses this to learn the optimal policy in episodes close to constant in the number of states.
These examples provide an evocative illustration for the potentially huge benefits from a judiciously information-seeking agent.
However, as we will show in the next section, they also provide an interesting edge case for some of the downsides in IDS' one-step approximation.

\subsection{One-Step Approximations and the Benefits of Pessimism}
\label{se:pessimism}

One surprising blind spot for the IDS agent is that if in any state the agent is certain that a particular action is optimal then it will \textit{always} select an optimal action at that state.
This is because the information ratio is zero, and therefore minimal, for the optimal action.
At first sight this might not seem like a severe problem, but it can have some surprisingly negative consequences.
To highlight this failure mode, we consider the informative chain environment described in the previous section, but modify it so that the final informative state has no chance of being rewarding (see Figure \ref{fig:chain_scale}).
In this setting, an IDS-GVF agent will make it all the way to state $\tau-2$, driven by potential information gain.
However, at state $\tau - 2$ it would never select the `forward' action and enter the revealing state, because it is certain that `down' is an optimal action at state $\tau - 2$.
This would repeat in subsequent episodes so that the agent would require a number of episodes that grows linearly with the number of states to find the rewarding policy, similar to less sophisticated forms of exploration like Thompson sampling.

\begin{figure}[htb]
\begin{center}
\includegraphics[scale=0.25]{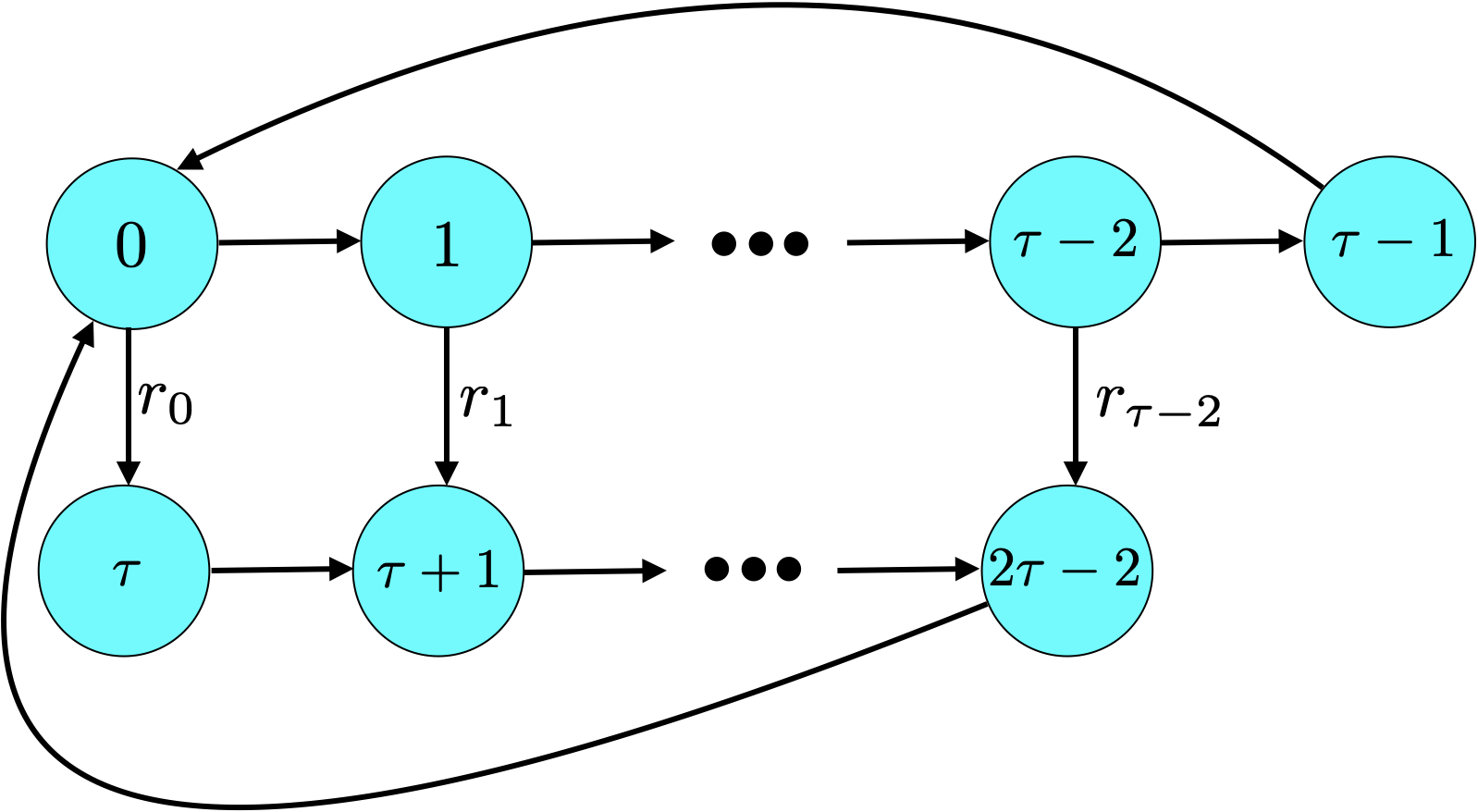}
\caption{A simple class of deterministic episodic environments. 
Edges represent possible state transitions, labeled with rewards $r_0,..,r_{\tau-2}$, exactly one of which gives rewards one and all other transitions provide a reward of zero.
}
\label{fig:chain_scale}
\end{center}
\end{figure}

The issue here is that once an agent is sure of any optimal action in a state, it does not value information gain from other actions.
This can be problematic because it can be worthwhile to take an action \textit{known} to be suboptimal in order to gain more information.
To illustrate this possibility, we modify variance-IDS by incorporating a pessimistic adjustment $\epsilon^{\rm pess}$ to the expected shortfall:
\begin{equation}
\label{eq:ir_pess}
\min_{\nu \in \Delta_\actions} \frac{\E\left[\max_{a \in \actions}\hat{Q}_{*,t}(S_t, a) - \hat{Q}_{*,t}(S_t, \tilde{A}_t) \,\big|\, X_t \right]^2 + \epsilon^{\rm pess}} {\E \left[ \mathrm{tr}\left(\mathrm{Cov}\left[\hat{Q}_{\dagger,t}(S_t, \tilde{A}_t) \big| X_t, \tilde{A}_t \right]\right) | X_t \right]}.
\end{equation}
This additional term ensures that highly-informative actions can still be considered even when the optimal action is known.
We refer to this adjustment as `pessimistic' since it effectively increases the estimated shortfall of each action.
This term also contrasts with the typical heuristic for exploration via optimism in the face of uncertainty, but maintains the judicious balance between regret and information gain.

Figure \ref{fig:intrinsic_motivation} shows the results for the same IDS agent as Figure \ref{fig:chain_log} where the revealing state has no chance of being rewarding.
Without the additional $\epsilon^{\rm pess}$ the IDS agent never selects the informative action.
However, with even the addition of a small $\epsilon^{\rm pess} > 0$ the learning reverts to constant time in the number of states.
Figure \ref{fig:inform_ridge} shows that these results are remarkably robust over many orders of magnitude of $\epsilon^{\rm pess}$.
Better understanding the effect of the pessimism term and, more generally, developing principled guidelines for addressing this downside of IDS is an area for future work.

\begin{figure}[!ht]
\centering
\begin{subfigure}{.5\textwidth}
  \centering
  \includegraphics[width=.99\linewidth]{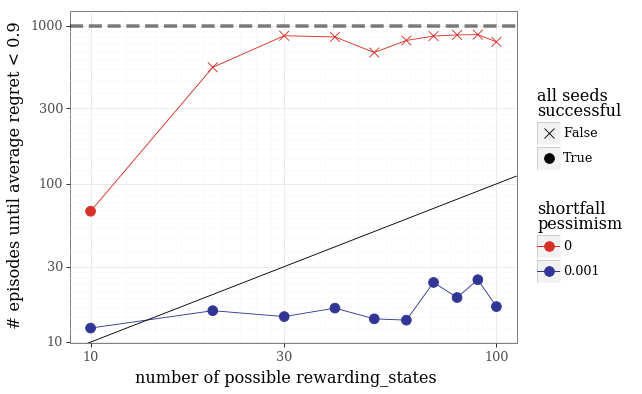}
  \caption{With $\epsilon^{\rm pess} > 0$ IDS scales sublinearly.}
  \label{fig:chain_ridge_log}
\end{subfigure}%
\begin{subfigure}{.5\textwidth}
  \centering
  \includegraphics[width=.99\linewidth]{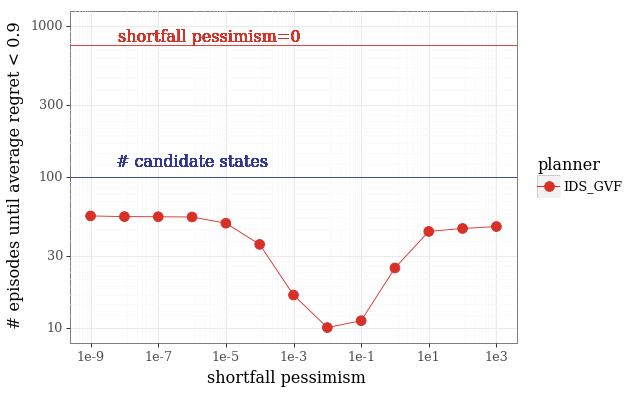}
  \caption{Results are very robust to the scale of $\epsilon^{\rm pess}$.}
  \label{fig:inform_ridge}
\end{subfigure}
\caption{Additional regret pessimism can remedy some of the shortcomings of IDS.}
\label{fig:intrinsic_motivation}
\end{figure}

\chapter{Closing Remarks}

The concepts and algorithms we have introduced are motivated by an objective to minimize regret.  They serve to guide agent design.  Resulting agents are unlikely to attain minimal regret, though these concepts may lead to lower regret than otherwise.  There are many directions in which this line of work can be extended and improved, some of which we now discuss.

We began with a framing of the exploration-exploitation problem as one of figuring out what to do when the agent has not identified its target action, which could, for example, be the optimal action for its current situational state.  As discussed in Section \ref{se:pessimism}, this poses a limitation when there are complex interdependencies across situational states that can make it helpful to execute alternative actions even when the optimal one is known.  That section also introduced a possible remedy, involving injection of pessimism in the information ratio.  Both this specific approach and the broader issue deserve further investigation.

Our regret bound and, more broadly, our discussion of information assumed a fixed learning target, with epistemic state dynamics monotonically reducing uncertainty about it.  This allowed us to think of an agent as converging over time on the performance of a baseline policy, as indicated by sublinear accumulation of regret.  However, in a sufficiently complex environment, it may be advantageous for the agent to adapt its learning target based on agent state.  For example, if an agent with bounded memory lives in one country but moves to another, it ought to forget information specific to its previous home in order to make room required for retaining new information.  It would be interesting to extend our information-theoretic concepts and results to accommodate such settings.

We have taken the learning target to be fixed and treated the target policy as a baseline.  An alternative could be to prescribe a class of learning targets, with varying target policy regret.  The designer might then balance between the number of bits required, the cost of acquiring those bits, and regret of the resulting target policy.  This balance could also be adapted over time to reduce regret further.  While the work of \citet{arumugam2021deciding} presents an initial investigation pertaining to very simple bandit environments, leveraging concepts from rate-distortion theory, much remains to be understood about this subject.

More broadly, one could consider simultaneous optimization of proxies and learning targets.  In particular, for any reward function and distribution over environments, the designer could execute an algorithm that automatically selects a learning target and proxy, possibly from sets she specifies.  This topic could be thought of as {\it automated architecture design}.

Concepts and algorithms introduced in this paper require quantification of information, and for this we have used entropy to assess the number of bits required to identify a learning target.  A growing body of related work in the area of bandit learning considers alternative information measures \citep{frankel2019quantifying,kirschner2020asymptotically,kirschner2020information,lattimore2020mirror,ryzhov2012kg,powell2012optimal,devraj2021information}.  Indeed, the {\it bit} was introduced for studying communication systems, and it is conceivable that reinforcement learning calls for a different notion of information.  It remains to be seen whether such alternatives will lead to improved reinforcement learning agents.

\begin{acknowledgements}
Our thinking about the relation between information and sequential decision was shaped by an earlier collaboration with Dan Russo, which focused on bandit environments.  Tor Lattimore offered many helpful comments on an earlier draft.  Johannes Kirschner provided valuable feedback during the review process that helped improve the paper significantly.  Chao Qin's careful study of value-IDS revealed technical limitations, as captured by his insightful result that we cite and discuss at the end of Chapter \ref{se:cost-benefit}.  We also would like to thank Chao Qin for the helpful discussion on the ``optimism conjecture'' in the Thompson sampling analysis. The paper also benefited from discussions with and feedback from Dilip Arumugam, Seyed Mohammad Asghari, Andy Barto, Dimitri Bertsekas, Adithya Devraj, Shi Dong, Abbas El Gamal, Yanjun Han, Mike Harrison, Geoffrey Irving, Anmol Kagrecha, Ayfer Ozgur, Warren Powell, Doina Precup, Omar Rivasplata, David Silver, Satinder Singh, Rich Sutton, David Tse, John Tsitsiklis, and Tsachy Weissman.
\end{acknowledgements}

\appendix

\chapter{Analysis of Thompson Sampling with an Episodic MDP}
\label{se:mdp-analysis}

In this appendix, we provide an analysis of Thompson sampling for the episodic MDP described in Section \ref{se:episodic-mdp}.
We start by establishing useful concentration inequalities in Section \ref{se:mdp-concentration}. We then propose an optimism conjecture in Section \ref{sec:optimism} and support it through empirical simulations. Finally, in Section \ref{se:mdp-regret}, we establish a regret bound for Thompson sampling in the environments described in Section \ref{se:episodic-mdp}, assuming that the optimism conjecture holds.
Note that all the entropy and mutual information terms in this section are measured in nats.

\section{Information and Concentration}
\label{se:mdp-concentration}


\begin{lemma}
\label{le:beta-information-bound}
If $p$ and $\hat{p}$ are independent and identically beta-distributed random variables with parameters $\alpha \geq 1$ and $\beta \geq 1$ then, for all $c > 0$,
$$\Pr(\sqrt{c \I (p; b)} - |p - \hat{p}| \leq 0) \leq 2 e^{- c/6},$$
where $b \sim \mathrm{Bernoulli}(p)$ conditioned on $p$. 
\end{lemma}

\begin{proof}
A real-valued random variable $X$ is said to be $\sigma^2$-sub-Gaussian if $\E[\exp(\lambda(X-\E X))] \leq \exp(\lambda^2 \sigma^2 / 2)$ for all $\lambda$.
We first prove that $p - \hat{p}$ is $\frac{1}{2(\alpha+\beta)}$-sub-Gaussian. Since $p$ and $\hat{p}$ are i.i.d. from $\rm{Beta}(\alpha, \beta)$, thus $\E [p -
\hat{p}]=0$. Moreover, from Theorem~4 of \citet{elder2016bayesian}, for $p \sim \rm{Beta}(\alpha, \beta)$, $p - \E[p]$ is $\frac{1}{4(\alpha+\beta)+2}$-sub-Gaussian. Consequently, $\E[p] -\hat{p} =
\E[\hat{p}] - \hat{p} = - \left( \hat{p} - \E[\hat{p}] \right)$ is also $\frac{1}{4(\alpha+\beta)+2}$-sub-Gaussian. Since, $p - \E[p]$ and  $\E[p] -\hat{p}$ are also independent, we have
\[
p - \hat{p} = (p - \E[p]) + (\E[p] - \hat{p})
\]
is $\frac{1}{2(\alpha+\beta)+1}$-sub-Gaussian. Since $\frac{1}{2(\alpha+\beta)+1} < \frac{1}{2(\alpha+\beta)}$, $p - \hat{p}$ is also $\frac{1}{2(\alpha+\beta)}$-sub-Gaussian.

Consequently, from the sub-Gaussian tail bound, we have
\begin{align*}
    & \Pr(\sqrt{c \I (p; b)} - |p - \hat{p}| \leq 0) \\
    & = \, \Pr \left(|p - \hat{p}| \geq \sqrt{c \I (p; b)}\right) \\
    & \leq \, 2 \exp \left( - \frac{c \I (p; b) }{2 \cdot \frac{1}{2(\alpha+\beta)}} \right) =
    2 \exp \left( - c \I (p; b) (\alpha + \beta) \right).
\end{align*}
From Lemma~10 of \citet{xlu2020dissertation}, for $p \sim {\rm Beta}(\alpha, \beta)$ with $\alpha \geq 1 \text{ and} \beta \geq 1$, 
$$\I (p; b) \geq \frac{1}{6 (\alpha + \beta)}.$$ Thus, we have
\[
\Pr(\sqrt{c \I(p; b)} - |p - \hat{p}| \leq 0) \leq 2 \exp \left( - c \I(p; b) (\alpha + \beta) \right) \leq 2 e^{-c/6}.
\]
\end{proof}

\begin{lemma}
\label{le:uniform-beta-information-bound}
For any positive integer $N$, if $p_1,\ldots,p_N$ are independent beta-distributed random variables with parameters greater than or equal to $1$ and, for each $n$, $\hat{p}_n$ is independent and distributed identically with $p_n$ then, for all $\delta \in (0, 1)$,
$$\E\left[\min_{n \in \{1,\ldots,N\}} \left(\sqrt{6 \I(p_n; b_n) \ln \frac{2 N}{\delta}} - |p_n - \hat{p}_n| + \delta \right) \right] \geq 0,$$ 
where $b_n \sim \mathrm{Bernoulli}(p_n)$ conditioned on $p_n$ for $n = 1, \dots, N$.
\end{lemma}

\begin{proof}
For any $c > 6\ln(2)$, we have $2 e^{-c/6} <1$, thus
\begin{align*}
& \Pr\left(\min_{n \in \{1, \ldots, N\}} \left( \sqrt{c\I(p_n; b_n)} - |p_n -\hat{p}_n|\right) \leq 0 \right) \\
 & = 1 - \Pr\left(\min_{n \in \{1, \ldots, N\}} \left( \sqrt{c\I(p_n; b_n)} - |p_n -\hat{p}_n|\right) \geq 0 \right) \\
 & = 1 - \prod_{n=1}^{N}\Pr\left( \sqrt{c\I(p_n; b_n)} - |p_n -\hat{p}_n| \geq 0 \right) \\
 & = 1 - \prod_{n=1}^{N}\left( 1-\Pr\left( \sqrt{c\I(p_n; b_n)} - |p_n -\hat{p}_n| \leq 0 \right)\right) \\
 & \leq 1 - \left(1 - 2 e^{-c/6} \right)^N ~~~ \text{ from Lemma \ref{le:beta-information-bound} and $2 e^{-c/6} <1$} \\
 & \leq 1 - 1 + 2Ne^{-c/6} ~~~ \text{ from Bernoulli's inequality} \\
 & =  2Ne^{-c/6}.
\end{align*}

On substituting $c=6 \ln\frac{2N}{\delta} > 6 \ln(2)$, 
\begin{align*}
\Pr\left(\min_{n \in \{1, \ldots, N\}} \left( \sqrt{6 \I (p_n; b_n) \ln \frac{2 N}{\delta}} - |p_n -\hat{p}_n|\right) \leq 0 \right) \leq \delta \label{eq:uniform-beta-prob-bound}
\end{align*}
To simplify the notation, we define
\[
h = \min_{n \in \{1,\ldots,N\}} \left(\sqrt{6 \I(p_n; b_n) \ln \frac{2 N}{\delta}} - |p_n - \hat{p}_n| \right).
\]
Notice that since $p_n, \hat{p}_n \in [0, 1]$, thus $h \geq -1$ always holds.
Also, from the above results, $\Pr(h \leq 0) \leq \delta$.
Therefore,
\begin{align*}
    \E[h] =& \, \E[h | h \geq 0] \Pr(h \geq 0) + \E[h | h < 0] \Pr(h < 0) \\
    \stackrel{(a)}{\geq} & \, \E[h | h < 0] \Pr(h < 0)
    \stackrel{(b)}{\geq}  -\delta,
\end{align*}
where (a) holds since $\E[h | h \geq 0] \Pr(h \geq 0) \geq 0$, and (b) holds since $\E[h | h < 0] \geq -1$ and $\Pr(h < 0) \leq \Pr(h \leq 0) \leq \delta$. Consequently, $\E [h +\delta] \geq 0$, that is
$$\E\left[\min_{n \in \{1,\ldots,N\}} \left(\sqrt{6 \I(p_n; b_n) \ln \frac{2 N}{\delta}} - |p_n - \hat{p}_n| + \delta \right) \right] \geq 0.$$ 



\end{proof}

\begin{lemma}
\label{le:mutual-info-proxy-bound}
Assume that $p$ is a beta-distributed random variable with $\alpha \geq 1$ and $\beta \geq 1$, and $b$ is drawn from $\mathrm{Bernoulli}(p)$ conditioned on $p$. For $\delta = 1/2, 1/3, 1/4, \ldots$, define $q = \left \lceil p/ \delta \right \rceil$. Then we have
\[
\I(p; b) \geq \I(q; b) \geq \I(p; b) - \max \left \{ 3, \ln \left(\frac{2}{\delta} \right) \right \} \delta.
\]
%
\end{lemma}

\begin{proof}
Notice that conditioning on $p$, $q$ is deterministic, and $q$ and $b$ are independent. Consequently, 
\[
\I(p; b) = \I(p, q; b) \geq \I(q; b).
\]

Let $f(\cdot)$ denote the probability density function of $p$. To simplify exposition, we use $\tilde{q}_i$ to denote $\Pr (b=1 | q= i \delta)$ for $i = 1, \dots, 1/\delta$. Note that
\[
\tilde{q}_i = \Pr (b=1 | q= i \delta) = \E \left[ p | q = i \delta \right] = \frac{\int_{(i-1)\delta}^{i \delta} p f(p) dp}{\int_{(i-1) \delta}^{i \delta} f(p) dp} \quad \forall i = 1, 2, \ldots, 1/\delta .
\]
With some algebraic manipulation, we can show that
\begin{align*}
\H(b | q) &= \sum_{i=1}^{1/\delta} \Pr(q = i\delta) \left( \tilde{q}_i \ln \frac{1}{\tilde{q}_i} + (1-\tilde{q}_i) \ln \frac{1}{1-\tilde{q}_i} \right) \\
&= \sum_{i=1}^{1/\delta} \int_{(i-1)\delta}^{i\delta} p f(p) dp \ln\frac{1}{\tilde{q}_i} + \int_{(i-1)\delta}^{i\delta} (1-p) f(p) dp \ln\frac{1}{1-\tilde{q}_i} \\
&= \sum_{i=1}^{1/\delta} \int_{(i-1)\delta}^{i\delta} \left( p \ln \frac{1}{\tilde{q}_i} + (1-p) \ln \frac{1}{1-\tilde{q}_i} \right) f(p) dp.
\end{align*}
Thus,
\begin{align*}
\I (p; b) -  \I (q; b)
&= \H(b|q) - \H(b|p) \\
&= \sum_{i=1}^{1/\delta} \int_{(i-1)\delta}^{i\delta} \left( p \ln \frac{p}{\tilde{q}_i} + (1-p) \ln \frac{1-p}{1-\tilde{q}_i} \right) f(p) dp \\
&= \sum_{i=1}^{1/\delta} \int_{(i-1)\delta}^{i \delta} \KL \left(p \| \tilde{q}_i \right) f(p) dp,
\end{align*}
where, with some abuse of notation, we use $\KL (p \| \tilde{q}_i )$ to denote a shorthand for $\KL \left(\mathrm{Bern}(p) \| \mathrm{Bern}(\tilde{q}_i) \right)$.


Without loss of generality, we assume that $\alpha \leq \beta$ (the other case is symmetric). 
Obviously, $\forall i =1 , 2, \ldots, 1/\delta$, we have $(i-1) \delta \leq \tilde{q}_i \leq i \delta$. Moreover, since $1 \leq \alpha \leq \beta$, we have $\tilde{q}_{1/\delta} \leq 1 - \delta/2$. This is because for $\mathrm{Beta}(\alpha, \beta)$ with $1 \leq \alpha \leq \beta$, $\mathrm{Beta}(\alpha, \beta)$ is either a uniform distribution, or a uni-modal distribution with mode less than or equal to $0.5$. Hence, $f(p)$ is strictly decreasing on interval $[1-\delta, 1]$, and hence $\tilde{q}_{1/\delta} \leq 1- \delta/2$. Consequently, for $i = 2, \ldots, 1/\delta$, we have



\begin{align*}
    \KL (p \| \tilde{q}_i ) \stackrel{(a)}{\leq} \frac{p^2}{\tilde{q}_i} + \frac{(1-p)^2}{1 - \tilde{q}_i} - 1 = 
    \frac{(p - \tilde{q}_i)^2}{\tilde{q}_i (1 - \tilde{q}_i)} \stackrel{(b)}{\leq} \frac{\delta^2}{\frac{3}{8} \delta} = \frac{8}{3} \delta < 3 \delta,
\end{align*}
where (a) follows from Theorem~1 of \citet{dragomir2000some}, and (b) follows from $|p - \tilde{q}_i| \leq \delta $, and $\delta \leq \tilde{q}_i \leq 1 - \delta/2$ for $i \geq 2$. Specifically, for $\delta \leq \tilde{q}_i \leq 1 - \delta/2$, we have
\[
\tilde{q}_i (1 - \tilde{q}_i) \geq \frac{\delta}{2} \left(1 - \frac{\delta}{2} \right) \geq \frac{\delta}{2} \left(1 - \frac{1}{4} \right) = \frac{3}{8} \delta,
\]
where the second inequality follows from $\delta \leq \frac{1}{2}$. 

We now consider the case when $i=1$ and bound $\KL(p \| \tilde{q}_1)$ for $p \in (0, \delta]$. Notice that for $p \in (0, \delta]$, we have
\begin{align*}
    \KL(p \| q_1) \stackrel{(a)}{\leq} & \, \max \left \{\KL(0 \| q_1), \,
    \KL(\delta \| q_1) \right \} \\
    =& \, \max \left \{
    \ln \frac{1}{1 - \tilde{q}_1}, \, \delta \ln \frac{\delta}{\tilde{q}_1} + (1-\delta) \ln \frac{1 -\delta}{1 - \tilde{q}_1}
    \right \} \\
    \stackrel{(b)}{\leq}& \,
    \max \left \{
    2 \delta, \delta \ln \frac{\delta}{\tilde{q}_1}
    \right \},
\end{align*}
where (a) follows from $p \in (0, \delta]$, and (b) follows from $ \ln \frac{1 -\delta}{1 - \tilde{q}_1} \leq 0$ and 
\[
\ln \frac{1}{1 - \tilde{q}_1} \leq \ln \frac{1}{1 - \delta} \leq \ln \left(1 + \frac{\delta}{1-\delta} \right) \leq \frac{\delta}{1-\delta} \leq 2 \delta,
\]
where the last inequality follows from $\delta \leq 1/2$. We now derive a lower bound on $\tilde{q}_1$. Let $F(\cdot ; \alpha, \beta)$ denote the CDF of $\mathrm{Beta}(\alpha, \beta)$, then we have
\begin{align*}
    \tilde{q}_1 = & \,
    \frac{1}{F(\delta; \alpha, \beta)}   \int_{0}^{\delta} \frac{x^{\alpha} (1-x)^{\beta-1}}{B(\alpha, \beta)} dx \\
    =& \, \frac{B(\alpha+1, \beta)}{B(\alpha, \beta) F(\delta; \alpha, \beta)}   \int_{0}^{\delta} \frac{x^{\alpha} (1-x)^{\beta-1}}{B(\alpha+1, \beta)} dx \\
    =& \, \frac{B(\alpha+1, \beta) F(\delta; \alpha+1, \beta)}{B(\alpha, \beta) F(\delta; \alpha, \beta)} \\
    \stackrel{(a)}{=}& \, \frac{\alpha}{\alpha+\beta} \left[1 -\frac{\delta^{\alpha}(1-\delta)^{\beta}}{\alpha \int_{0}^{\delta} x^{\alpha-1} (1-x)^{\beta-1} dx} \right] \\
    \stackrel{(b)}{\geq} & \, \frac{\alpha}{\alpha+\beta} \left[1 -\frac{\delta^{\alpha}(1-\delta)^{\beta}}{\alpha \int_{0}^{\delta} x^{\alpha-1} (1-\delta)^{\beta-1} dx} \right] =
    \frac{\alpha}{\alpha+\beta} \delta,
\end{align*}
where $B(\cdot, \cdot)$ is the beta function. Note that (a) follows from $B(\alpha+1, \beta) = B(\alpha, \beta) \frac{\alpha}{\alpha+\beta}$, and 
\[
F(\delta; \alpha+1, \beta) = F(\delta; \alpha, \beta) -\frac{\delta^\alpha (1-\delta)^\beta}{\alpha B(\alpha, \beta)},
\]
and (b) follows from $1-x \geq 1-\delta > 0$ and $\beta \geq 1$, and hence $(1 - x)^{\beta-1} \geq (1-\delta)^{\beta-1}$.

Combining the above results, we have
$\KL(p\| q_1) \leq \max \left \{2 , \ln \frac{\alpha+\beta}{\alpha} \right \} \delta$. Since $\KL(p \| q_i) < 3 \delta$ for $i \geq 2$, we then have
\[
\KL(p\| q_i) \leq \max \left \{3 , \ln \frac{\alpha+\beta}{\alpha} \right \} \delta \quad \forall i = 1, 2, \ldots, 1/\delta \text{ and } \forall p \in ( (i-1)\delta, i\delta] .
\]
This implies that
\[
\I(p;b) - \I (q, b) \leq 
\max \left \{3 , \ln \frac{\alpha+\beta}{\alpha} \right \} \delta .
\]

On the other hand, from Lemma~10 and Lemma~11 in \citet{xlu2020dissertation}, we have
\begin{align*}
    \I(p; b) =& \, \frac{\alpha}{\alpha+\beta} \left( \psi(\alpha+1) - \ln \alpha \right)  \\
    & + \frac{\beta}{\alpha+\beta} \left( \psi(\beta+1) - \ln \beta \right) - \left( \psi(\alpha+\beta+1)- \ln(\alpha+\beta)\right),
\end{align*}
where $\psi$ is the digamma function. From the digamma inequalities $\ln(x + 0.5) \leq \psi(x+1) \leq \ln(x) + \frac{1}{2x}$ for $x>0$ (Lemma~11 of \citet{xlu2020dissertation}), we have $\psi(\alpha+1) - \ln \alpha \leq \frac{1}{2 \alpha}$, $\psi(\beta+1) - \ln \beta \leq \frac{1}{2 \beta}$, and 
$\psi(\alpha+\beta+1) - \ln (\alpha + \beta)>0$. Consequently, we have $\I (p; b) < \frac{1}{\alpha+\beta}$. Thus we have
\[
\I (p; b) - \I (q; b) \leq \I (p; b) < \frac{1}{\alpha+\beta}.
\]
Combining the above results, we have
\[
\I (p;b) - \I (q, b) \leq 
\min \left\{ \max \left \{3 , \ln \frac{\alpha+\beta}{\alpha} \right \} \delta, \,
\frac{1}{\alpha+\beta}
\right \} .
\]
Finally, note that 
\begin{align*}
& \min \left\{ \max \left \{3 , \ln \frac{\alpha+\beta}{\alpha} \right \} \delta, \,
\frac{1}{\alpha+\beta} \right \} \\
\stackrel{(a)}{\leq} & \, \max \left \{
3 \delta, \, \min \left \{ \delta \ln (\alpha + \beta), \, \frac{1}{\alpha + \beta} \right \}
\right \} \\
\stackrel{(b)}{\leq} & \, \max \left \{
3 \delta, \, \max_{x \geq 2} \min \left \{ \delta \ln x, \, \frac{1}{x} \right \}
\right \} \\
\stackrel{(c)}{=} & \, \max \left \{
3 \delta, \, \min_{x \geq 2} \max \left \{ \delta \ln x, \, \frac{1}{x} \right \}
\right \} \\
\stackrel{(b)}{\leq} & \, \max \left\{3 \delta, \delta \ln \left(\frac{2}{\delta} \right)   \right \} = \max \left \{ 3, \ln \left(\frac{2}{\delta} \right) \right \} \delta,
\end{align*}
where (a) follows from $\alpha \geq 1$, (b) follows from $\alpha + \beta \geq 2$, (c) follows from $\max_{x \geq 2} \min \left \{ \delta \ln x, \, \frac{1}{x} \right \} = \min_{x \geq 2} \max \left \{ \delta \ln x, \, \frac{1}{x} \right \}$ for $x \geq 2$ (note that $\delta \leq 1/2$), and (d) follows by choosing $x = 2/\delta$ in $\max \left \{ \delta \ln x, \, \frac{1}{x} \right \}$.
%
\end{proof}

\begin{lemma}
\label{le:uniform-beta-proxy-information-bound}
For any positive integer $N$, let $p_1,\ldots,p_N$ be independent, beta-distributed random variables such that $p_n \sim \mathrm{Beta}(\alpha_n, \beta_n)$ with $\alpha_n > 1$ and $\beta_n > 1$ for each $n$. Moreover, for each $n$, $\hat{p}_n$ is independent and distributed identically with $p_n$, and $b_n \sim \mathrm{Bernoullil}(p_n)$ conditioned on $p_n$. Then, for all $\delta = 1/2, 1/3, 1/4, \ldots$,
\begin{align*}
\E\Bigg[\min_{n \in \{1,\ldots,N\}} \Bigg(& \sqrt{6 \I (q_n; b_n) \ln \frac{2 N}{\delta}} - |q_n - \hat{q}_n| + 2 \delta \\
& + \sqrt{ 6 \max \left \{ 3, \ln \left(\frac{2}{\delta} \right)\right \} \delta \ln \frac{2N}{\delta}} \Bigg) \Bigg] \geq 0,
\end{align*}
where $q_n = \delta \lceil p_n / \delta \rceil$ and $\hat{q}_n = \delta \lceil \hat{p}_n / \delta \rceil$ are quantized approximations.
\end{lemma}
\begin{proof}
Notice that $q_n - \delta < p_n \leq q_n$ and $\hat{q}_n - \delta < \hat{p}_n \leq \hat{q}_n$. We now prove that $|q_n - \hat{q}_n| \leq |p_n - \hat{p}_n| + \delta$. Without loss of generality, assume that $q_n \geq \hat{q}_n$ (the other case is symmetric), then we have
\[
|q_n - \hat{q}_n| = q_n - \hat{q}_n \stackrel{(a)}{\leq} q_n - \hat{p}_n
\stackrel{(b)}{<} p_n + \delta - \hat{p}_n \leq |p_n - \hat{p}_n| + \delta ,
\]
where (a) follows from $\hat{p}_n \leq \hat{q}_n$, and (b) follows from $q_n < p_n + \delta$. Thus, we have
$-|q_n - \hat{q}_n| + \delta \geq -|p_n - \hat{p}_n| $.

On the other hand, we have
\begin{align*}
& \sqrt{6 \I (q_n; b_n) \ln \frac{2 N}{\delta}} + \sqrt{ 6 \max \left \{ 3, \ln \left(\frac{2}{\delta} \right)\right \} \delta \ln \frac{2N}{\delta}} \\
\geq & \, 
\sqrt{6 \left( \I(q_n; b_n) + \max \left \{ 3, \ln \left(\frac{2}{\delta} \right)\right \} \delta \right) \ln \frac{2 N}{\delta}} \\
\geq & \, \sqrt{6  \I (p_n; b_n)  \ln \frac{2 N}{\delta}},
\end{align*}
where the last inequality follows from Lemma~\ref{le:mutual-info-proxy-bound}. Combining the above results, we have
\begin{align*}
& \sqrt{6 \I (q_n; b_n) \ln \frac{2 N}{\delta}} - |q_n - \hat{q}_n| + 2 \delta + \sqrt{ 6 \max \left \{ 3, \ln \left(\frac{2}{\delta} \right)\right \} \delta \ln \frac{2N}{\delta}} \\
\geq & \sqrt{6 \I(p_n; b_n) \ln \frac{2 N}{\delta}} - |p_n - \hat{p}_n| +  \delta .
\end{align*}
Then, the result of this lemma directly follows from Lemma~\ref{le:uniform-beta-information-bound}.
\end{proof}

\section{Optimism}
\label{sec:optimism}

The following conjecture concerns the optimistic behavior of Thompson sampling for the ``ring'' MDPs considered in Section~\ref{se:episodic-mdp}.  
Our analysis in this appendix assumes that the conjecture holds.
\begin{conjecture}
\label{conj:ts-optimism}
Under the ``ring'' MDPs considered in 
Section~\ref{se:episodic-mdp},
for any episode $\ell$ and any time $t=\ell\tau, \ldots,(\ell+1)\tau-1$ within the episode,
$$\E[V_{\tau, \rho}(S_t) | P_{\ell\tau}] \leq \E[\hat{V}_{\ell}(S_t) | P_{\ell\tau}].$$
\end{conjecture}

\noindent
Note that Conjecture~\ref{conj:ts-optimism} always holds with equality for $t=\ell \tau$ and $t=(\ell +1) \tau-1$, $\forall \ell=0,1,\ldots$. To understand why, note that for $t=\ell \tau$, $S_{\ell \tau} = S_0$ is deterministic, and $V_{\tau, \rho}$ and $\hat{V}_{\ell}$ are i.i.d. conditioned on $P_{\ell \tau}$. Thus, 
$$\E[V_{\tau, \rho}(S_{\ell \tau}) | P_{\ell\tau}] =
\E[V_{\tau, \rho}(S_{0}) | P_{\ell\tau}] =
\E[\hat{V}_{\ell}(S_0) | P_{\ell\tau}]
=
\E[\hat{V}_{\ell}(S_{\ell \tau}) | P_{\ell\tau}].$$
Note that the reward model $r$ is deterministic. Consequently, $V_{\tau, \rho}(s) = \hat{V}_{\ell}(s) = \max_{a \in \actions} r(s, a, S_0)$ for all $s \in \states_{\tau-1}$. Thus, by definition, for $t=(\ell +1) \tau-1$, we have $S_t \in \states_{\tau-1}$ and
$$\E[V_{\tau, \rho}(S_t) | P_{\ell\tau}] = \E[\hat{V}_{\ell}(S_t) | P_{\ell\tau}] = \E \left[ 
\max_{a \in \actions} r(S_t, a, S_0)
\right ].$$

We leave the proof of Conjecture~\ref{conj:ts-optimism} for 
$\ell \tau < t < (\ell+1)\tau-1$ in this ``ring'' example for future work. In the remainder of this section, we provide some numerical results suggesting that  Conjecture~\ref{conj:ts-optimism} holds.

It is worth pointing out that Conjecture~\ref{conj:ts-optimism} might not hold in more general problems. In particular, if the prior distribution admits \emph{generalization} of transition probabilities across state-action pairs -- for example, if $\rho$ is \emph{correlated} across state-action pairs -- this conjecture may fail to hold.


\subsection{Numerical Verification}

We now provide numerical verification of Conjecture~\ref{conj:ts-optimism}. Note that Conjecture~\ref{conj:ts-optimism} states that for any episode $\ell$ and any time $t = \ell \tau, \ldots, (\ell + 1) \tau -1$, 
\begin{equation}
\label{eq:optimism-conj-reframe}
\E[\hat{V}_{\ell}(S_t) - V_{\tau, \rho}(S_t) | P_{\ell\tau}] \geq 0.
\end{equation}
We numerically verify this conjecture as follows: we sweep over $M=5, 10, 20$ and $\tau=3, 8, 10, 20, 30$. For each $(M, \tau)$ pair, we rerun the Thompson sampling algorithm on the ``ring'' MDP with state space 
$\states = \{0,\ldots,M-1\} \times \{0,\ldots,\tau-1\}$ for $50$ times, and each time we run for $300$ episodes. Then, we numerically test Conjecture~\ref{conj:ts-optimism} every three episodes. Moreover, when we test this conjecture, we test it for every time $t = \ell \tau, \ldots, (\ell + 1) \tau -1$. Thus, in total we test Conjecture~\ref{conj:ts-optimism} for $3 \times 50 \times 100 \times (3+8+10+20+30) = 1,065,000$ times.

In particular, we test if the left-hand side of \ref{eq:optimism-conj-reframe} is non-negative, as well as if it is strictly positive in some cases. The former indicates if this conjecture holds, while the latter indicates if this conjecture holds with strict inequality in some cases. The following procedure illustrates how we compute a point estimate of the left-hand side of equation~\ref{eq:optimism-conj-reframe}, as well as its standard error, at a given episode $\ell$ based on the Monte-Carlo simulation. Specifically, for each round of Monte-Carlo simulation $i=1, 2, \ldots, L$:

\begin{enumerate}
    \item sample transition models $\rho$, $\hat{\rho}$ i.i.d. from $P_{\ell \tau}$
    \item compute $V_{\tau, \rho}$, the optimal state value function under $\rho$
    \item compute $\hat{V}_\ell$ and $\hat{\pi}$, which are respectively the optimal state value function and an optimal policy under $\hat{\rho}$
    \item for all $t= \ell \tau, \ldots, (\ell+1) \tau -1$, compute $\nu_t$, the state distribution of $S_t$, under policy $\hat{\pi}$ and transition model $\rho$
    \item finally, compute
    $d^i_{\ell, t} = \sum_s \nu_t(s) \left[ \hat{V}_\ell (s) -  V_{\tau, \rho}(s) \right]$ for all $t= \ell \tau, \ldots, (\ell+1) \tau -1$
\end{enumerate}
We carry out $L=10,000$ Monte-Carlo simulations, and compute the point estimate of the left-hand side of equation~\ref{eq:optimism-conj-reframe} and its standard error according to
\[
\bar{d}_{\ell, t} = \frac{1}{L} \sum_{i=1}^L d^i_{\ell, t} \quad \text{and} \quad \mathtt{stderr}_{\ell, t} = \frac{1}{L} \sqrt{\sum_{i=1}^L \left( d^i_{\ell, t} - \bar{d}_{\ell, t} \right)^2}.
\]
 For any $\kappa>0$, we define the upper confidence bound (UCB) and the lower confidence bound (LCB) parameterized by $\kappa$ as
\begin{align*}
\mathtt{UCB}_{\ell, t}(\kappa) =& \, \bar{d}_{\ell, t} + \kappa \cdot \mathtt{stderr}_{\ell, t} \nonumber \\
\mathtt{LCB}_{\ell, t}(\kappa) =& \,  \bar{d}_{\ell, t} - \kappa \cdot \mathtt{stderr}_{\ell, t} \nonumber
\end{align*}
We report the fractions of cases for which $\mathtt{UCB}_{\ell, t}(\kappa)<0$ or $\mathtt{LCB}_{\ell, t}(\kappa) > 0$ for a wide range of $\kappa$, and compare it with the Gaussian benchmark $1-\Phi(\kappa)$, where $\Phi (\cdot)$ is the cumulative distribution function of the standard normal distribution $N(0, 1)$. Intuitively, a negative UCB suggests that the conjecture does not hold, while a positive LCB suggests that the conjecture holds with strict inequality. The results are summarized in Table~\ref{tab:optimism_conj}.

\begin{table}[t]
    \centering
    \begin{tabular}{|c|c|c|c|}
    \hline
    $\kappa$ & frac. $\mathtt{UCB}_{\ell, t}(\kappa)<0$   &   frac. $\mathtt{LCB}_{\ell, t}(\kappa)>0$    & $1-\Phi(\kappa)$   \\
    \hline
	$1$	& $6.31315\%$  & $45.29878\%$	& $15.86553\%$ \\
	$1.5$	&  $2.49822\%$  &	$34.76901\%$	& $6.68072\%$	 \\
	$2$ & $0.80826\%$	& $26.56451\%$	& $2.27501\%$	 \\
	$2.5$	& $0.20451\%$	& $20.3923\%$	&  $0.62097\%$	 \\
	$3$	&  $0.04085\%$	&  $15.8954\%$	& $0.13499\%$	 \\
	$3.5$ &  $0.00516\%$	&  $12.53643\%$ &	$0.02326\%$	 \\
	$4$ & $0.00075\%$	& $10.02516\%$ &	$0.00317\%$	\\
	$4.5$	&  $0\%$  &  $8.10714\%$	&  $0.00034\%$	\\
	$5$	 &  $0\%$  & $6.62685\%$	&  $0.00005\%$	   \\
	$5.5$ 	&  $0\%$  & $5.4477\%$	&  $0\%$	\\
	$6$	 & $0\%$  &  $4.50235\%$	&  $0\%$  \\
	\hline
    \end{tabular}
    \caption{Numerical verification of Conjecture~\ref{conj:ts-optimism}}
    \label{tab:optimism_conj}
\end{table}

The experiment results suggest that Conjecture~\ref{conj:ts-optimism} holds in the ``ring" MDP. In particular, for each chosen $\kappa$, the fraction of negative UCBs is much smaller than the benchmark $1-\Phi(\kappa)$. Moreover, they also suggest that there are cases where the conjecture holds with strict inequality. In particular, for each chosen $\kappa$, the fraction of positive LCBs is much larger than the benchmark.

\section{Regret}
\label{se:mdp-regret}

The following results pertain to application of Thompson sampling to the ``ring'' episodic MDP described in Section \ref{se:episodic-mdp}.
\begin{lemma}
\label{le:psrl-information-ratio}
Assume that Conjecture~\ref{conj:ts-optimism} holds.  Then, for all integers $m \geq 2$ and times $t$,
\[
\E[V_*(H_t) - Q_*(H_t, A_t) - g(\delta) \tau^2 ]_+^2 
 \leq  \,
 6 \tau^{3}
 \ln \frac{2 \states \actions}{\delta}
 \I (\target; H_{t:(\ell+1)\tau}  | P_t),
\]
where $\delta = 1/m$, $g(\delta) = 3 \delta + \sqrt{ 6 \max \left \{ 3, \ln \left(\frac{2}{\delta} \right)\right \} \delta \ln \frac{2 \states \actions}{\delta}}$, 
%
and $\target$ is a quantized approximation of $\rho$ for which $\target(s+1|s,a) = \delta \lceil \rho(s+1|s,a)/\delta \rceil$ and $\target(s-1|s,a) = 1 - \target(s+1|s,a)$.
\end{lemma}
\begin{proof}
Time $t$ resides in episode $\ell = \lfloor t/\tau \rfloor$.  Recall that $\hat{\rho}_\ell$ is the observation probability function sampled at the start of episode $\ell$.  Let $\hat{\target}_\ell$ be the corresponding quantization, for which $\hat{\target}_\ell(s+1|s,a) = \delta \lceil \hat{\rho}_\ell (s+1|s,a)/\delta \rceil$ and $\hat{\target}_\ell(s-1|s,a) = 1 - \hat{\target}_\ell(s+1|s,a)$. Also note that in this problem, since $P_t = \Pr(\cdot | H_t)$, $S_t$ is conditionally deterministic given $P_t$; thus, conditioning on $H_t$ is equivalent to conditioning on $(P_t, S_t)$ and conditioning on $P_t$.
Let
$$I_t = \I (\target(S_t+1|S_t,A_t); A_t, S_{t+1} | P_t=P_t).$$
Note that
$$\I(\target(S_t+1|S_t,A_t); A_t, S_{t+1} | P_t=P_t) = \I(\target; A_t, S_{t+1} | P_t=P_t).$$
By the chain rule of mutual information,
$\I (\target; H_{t:(\ell+1)\tau} | P_t=P_t) = \sum_{k=t}^{(\ell+1)\tau}
\E \left[ I_k | P_t \right].$
Recall that $\states = \states_0 \cup \cdots \cup \states_{\tau-1}$ and $|\states_0| = \cdots = |\states_{\tau-1}| = M$.

By Lemma~\ref{le:uniform-beta-proxy-information-bound}, 
\begin{align*}
& \E\left[|\hat{\rho}_\ell(S_t+1|S_t,A_t) - \rho(S_t+1|S_t,A_t)|\ \Big| P_t \right] \\
\leq& \E\left[|\hat{\target}_\ell(S_t+1|S_t,A_t) - \target(S_t+1|S_t,A_t)|\ \Big| P_t \right] + \delta \\
\leq& \E\left[\sqrt{6 I_t \ln \frac{2 \states \actions}{\delta}} \Big| P_t \right] + \underbrace{3 \delta + \sqrt{ 6 \max \left \{ 3, \ln \left(\frac{2}{\delta} \right)\right \} \delta \ln \frac{2 \states \actions}{\delta}}}_{g(\delta)},
\end{align*}
where we define $g(\delta) = 3 \delta + \sqrt{ 6 \max \left \{ 3, \ln \left(\frac{2}{\delta} \right)\right \} \delta \ln \frac{2 \states \actions}{\delta}}$ to simplify the exposition.
It follows that, at time $t=(\ell+1)\tau-1$,
\begin{align*}
& \, \E[\hat{V}_\ell(S_t) - Q_{\tau, \rho}(S_t, A_t) | P_t] \\
=& \, \E\left[\max_{a \in \actions} \sum_{s' \in \states} \hat{\rho}_\ell(s'|S_t,a) r(S_t,a,s') - \sum_{s' \in \states} \rho(s'|S_t,A_t) r(S_t,A_t,s') | P_t\right] \\
=& \, \E\left[\sum_{s' \in \states} (\hat{\rho}_\ell(s'|S_t,A_t) - \rho(s'|S_t,A_t)) r(S_t,A_t,s') \Big| P_t\right] \\
\leq& \frac{1}{2} \E\left[\sum_{s' \in \states} |\hat{\rho}_\ell(s'|S_t,A_t) - \rho(s'|S_t,A_t)|\ \Big| P_t\right] \\
=& \, \E\left[|\hat{\rho}_\ell(S_t+1|S_t,A_t) - \rho(S_t+1|S_t,A_t)|\ \Big| P_t\right] \\
\leq& \, \E\left[\sqrt{6 I_t \ln \frac{2 \states \actions}{\delta}} \Big| P_t\right] + g(\delta).
\end{align*}
Similarly, at times $t=\ell \tau,\ell\tau+1,\ldots,(\ell+1)\tau-2$, 
\begin{align*}
& \E[\hat{V}_\ell(S_t) - Q_{\tau, \rho}(S_t, A_t) \Big| P_t] \\
=& \, \E\Bigg[\max_{a \in \actions} \sum_{s' \in \states} \hat{\rho}_\ell(s'|S_t,a) (r(S_t,a,s') + \hat{V}_\ell(s')) \\
& \hspace{3cm} - \sum_{s' \in \states} \rho(s'|S_t,A_t) (r(S_t,A_t,s') + V_{\tau,\rho}(s')) \Big| P_t\Bigg] \\
=& \, \E\Bigg[\sum_{s' \in \states} \hat{\rho}_\ell(s'|S_t,A_t) (r(S_t,A_t,s') + \hat{V}_\ell(s')) \\
&  \hspace{3cm}- \sum_{s' \in \states} \rho(s'|S_t,A_t) (r(S_t,A_t,s') + V_{\tau,\rho}(s')) \Big| P_t\Bigg] \\
\end{align*}
\begin{align*}
=& \, \E\Bigg[\sum_{s' \in \states} \hat{\rho}_\ell(s'|S_t,A_t) (r(S_t,A_t,s') + \hat{V}_\ell(s'))) \\
& \hspace{3cm} - \sum_{s' \in \states} \rho(s'|S_t,A_t) (r(S_t,A_t,s') + \hat{V}_\ell(s')) \Big| P_t\Bigg] \\
& \hspace{3cm} + \E\left[\sum_{s' \in \states} \rho(s'|S_t,A_t) (\hat{V}_\ell(s') - V_{\tau,\rho}(s')) \Big| P_t\right] \\
=& \, \E\Bigg[\sum_{s' \in \states} (\hat{\rho}_\ell(s'|S_t,A_t)-\rho(s'|S_t,A_t)) (r(S_t,A_t,s') + \hat{V}_\ell(s')) \Big| P_t\Bigg] \\
& \hspace{3cm} + \E\left[\hat{V}_\ell(S_{t+1}) - V_{\tau,\rho}(S_{t+1}) \big| P_t\right] \\
\leq& \frac{\tau}{2} \E\Bigg[\sum_{s' \in \states} |\hat{\rho}_\ell(s'|S_t,A_t) - \rho(s'|S_t,A_t)|\ \big| P_t\Bigg] \\
& \hspace{3cm} + \E\left[\hat{V}_\ell(S_{t+1}) - V_{\tau,\rho}(S_{t+1}) \big| P_t\right] \\
=& \, \tau \E\Bigg[|\hat{\rho}_\ell(S_t+1|S_t,A_t) - \rho(S_t+1|S_t,A_t)|\ \Big| P_t\Bigg] \\
& \hspace{3cm} + \E\left[\hat{V}_\ell(S_{t+1}) - V_{\tau,\rho}(S_{t+1}) \big| P_t\right] \\
\leq& \, \tau \E\Bigg[\sqrt{6 I_t \ln \frac{2 \states \actions}{\delta}} \Big| P_t\Bigg] + g(\delta) \tau + \E\left[\hat{V}_\ell(S_{t+1}) - Q_{\tau,\rho}(S_{t+1}, A_{t+1}) \big| P_t\right].
\end{align*}
It follows that
\begin{align*}
& \E[\hat{V}_\ell(S_t) - Q_{\tau, \rho}(S_t, A_t) | P_t] \\
\leq& \sum_{k=t}^{(\ell+1)\tau-1} \left(\tau \E\left[\sqrt{6 I_k \ln \frac{2 \states \actions}{\delta}} \Big| P_t\right] + g(\delta) \tau \right) \\
\leq& \tau^{3/2} \sqrt{6 \sum_{k=t}^{(\ell+1)\tau-1} \E[I_k | P_t] \ln \frac{2 \states \actions}{\delta}} +  g(\delta) \tau^2 \\
\leq& \tau^{3/2} \sqrt{6 \I (\target; H_{t:(\ell+1)\tau} | P_t \leftarrow P_t)  \ln \frac{2 \states \actions}{\delta}} + g(\delta) \tau^2.
\end{align*}
Further, under Conjecture~\ref{conj:ts-optimism}, for all episode $\ell$ and time $t=\ell\tau, \ldots,(\ell+1)\tau-1$,
\begin{align*}
\E[V_*(H_t) - Q_*(H_t, A_t) | P_{\ell \tau}] 
=& \, \E[V_{\tau,\rho}(S_t) - Q_{\tau, \rho}(S_t, A_t) | P_{\ell \tau}] \\
\leq& \, \E[\hat{V}_\ell(S_t) - Q_{\tau, \rho}(S_t, A_t) | P_{\ell \tau}],
\end{align*}
where the last inequality follows from Conjecture~\ref{conj:ts-optimism}.
Therefore, we have
\begin{align*}
& \E[V_*(H_t) - Q_*(H_t, A_t) - g(\delta) \tau^2 | P_{\ell \tau}] \\
\leq & \, \tau^{3/2} \sqrt{6 \ln \frac{2 \states \actions}{\delta}}
\E \left[ \sqrt{ \I (\target; H_{t:(\ell+1)\tau}  | P_t \leftarrow P_t)  } \middle | P_{\ell \tau} \right ],
\end{align*}
which further implies that
\begin{align*}
& \E[V_*(H_t) - Q_*(H_t, A_t) - g(\delta) \tau^2 ] \\
\leq & \, \tau^{3/2} \sqrt{6 \ln \frac{2 \states \actions}{\delta}} \E \left[\sqrt{ \I (\target; H_{t:(\ell+1)\tau}  | P_t \leftarrow P_t)  }  \right ].
\end{align*}
Since the right-hand side of the above inequality is positive, we then have
\begin{align*}
& \E[V_*(H_t) - Q_*(H_t, A_t) - g(\delta) \tau^2 ]_+ \\
\leq & \, \tau^{3/2}
\sqrt{6 \ln \frac{2 \states \actions}{\delta}}
\E \left[\sqrt{ \I (\target; H_{t:(\ell+1)\tau}  | P_t \leftarrow P_t)  }  \right ].
\end{align*}
Therefore, we have
\begin{align*}
& \, \E[V_*(H_t) - Q_*(H_t, A_t) - g(\delta) \tau^2 ]_+^2 \\
\leq& \,  6 \tau^{3} \ln \frac{2 \states \actions}{\delta}
\left( \E \left[ \sqrt{ \I (\target; H_{t:(\ell+1)\tau}  | P_t \leftarrow P_t)  }  \right ] \right)^2 \\
\leq & \, 6 \tau^{3} \ln \frac{2 \states \actions}{\delta} \E \left[ \I (\target; H_{t:(\ell+1)\tau}  | P_t \leftarrow P_t) \right ] \\
= & \, 6 \tau^{3} \ln \frac{2 \states \actions}{\delta} \I (\target; H_{t:(\ell+1)\tau}  | P_t).
\end{align*}
%
\end{proof}

\noindent
Finally, we prove the following theorem based on Lemma~\ref{le:psrl-information-ratio} and Theorem~\ref{th:regret-bound}, under Conjecture~\ref{conj:ts-optimism}.

\begin{theorem}
\label{th:mdp-regret-bound}
Assume Conjecture~\ref{conj:ts-optimism} holds.  Then, for all integers $m \geq 2$ and times $t$,
\begin{small}
\begin{align*}
    & \, \Regret(T | \pi_{\rm TS}) \\
    \leq & \,
    \tau^2 \sqrt{ 6   \states \actions T \ln \left(\frac{1}{\delta} \right) \ln \left(\frac{2 \states \actions}{\delta}\right)
    } 
    + \left[3 \delta + \sqrt{ 6 \max \left \{ 3, \ln \left(\frac{2}{\delta} \right)\right \} \delta \ln \left(\frac{2 \states \actions}{\delta}\right)}  \right] \tau^2 T \\
    = & \, \mathcal{O} \left( 
    \tau^2 \sqrt{\log \left(\frac{1}{\delta} \right) \log \left(\frac{\states \actions}{\delta} \right)} \left[\sqrt{\states \actions T } + T \sqrt{\delta} \right]
    \right),
\end{align*}
\end{small}
where $\delta = 1/m$.
%
\end{theorem}

\begin{proof}
By Lemma~\ref{le:psrl-information-ratio}, we first prove that $\Gamma_{\tau, \epsilon, t} \leq 6 \tau^4 \ln \frac{2 \states \actions}{\delta}$ for 
\[
\epsilon=g(\delta) \tau^2 = \left[3 \delta + \sqrt{ 6 \max \left \{ 3, \ln \left(\frac{2}{\delta} \right)\right \} \delta \ln \frac{2 \states \actions}{\delta}} \right] \tau^2.
\]
From Lemma~\ref{le:psrl-information-ratio}, we have
\[
\E[V_*(H_t) - Q_*(H_t, A_t) -\epsilon ]_+^2 
 \leq  \,
 6 \tau^{3}
 \ln \frac{2 \states \actions}{\delta}
 \I (\target; H_{t:(\ell+1)\tau}  | P_t).
\]
Also, note that in this case, the environment $\environment$ determines the proxy $\proxy$ and hence the target $\target$, and consequently
\begin{align*}
& \I(\target; \environment | P_{t}) -
\I(\target; \environment | P_{t+\tau}) \\
= & \, \H (\target | P_t) - \H (\target | P_{t+\tau})
= \I (\target; H_{t:t+\tau} | P_t)
\geq \I (\target; H_{t:(\ell+1)\tau} | P_t),
\end{align*}
where $\H$ is the entropy function in nats.
Consequently, by definition of $\Gamma_{\tau, \epsilon, t}$, we have
\begin{align*}
    \Gamma_{\tau, \epsilon, t} = & \,
    \frac{\E[V_*(H_t) - Q_*(H_t, A_t) - \epsilon_t]_+^2}{(\I(\target; \environment|P_t) - \I (\target; \environment|P_{t+\tau}))/\tau} \\
    = & \, 
    \frac{\E[V_*(H_t) - Q_*(H_t, A_t) - \epsilon_t]_+^2}{\I (\target; H_{t:t+\tau} | P_t) /\tau} \\
    \leq & \,
    \frac{\E[V_*(H_t) - Q_*(H_t, A_t) - \epsilon_t]_+^2}{\I (\target; H_{t:(\ell+1)\tau} | P_t) /\tau} \leq \, 6 \tau^4 \ln \frac{2 \states \actions}{\delta}.
\end{align*}
Then, by Theorem~\ref{th:regret-bound}, we have
\begin{align*}
    \Regret(T | \pi_{\rm TS}) \leq \sqrt{\I(\target; \environment) 
    \sum_{t=0}^{T-1} \Gamma_{\tau, \epsilon, t}} + T \epsilon \leq 
    \sqrt{ 6 \tau^4 T \H (\target) \ln \frac{2 \states \actions}{\delta}
    } + T \epsilon.
\end{align*}
Note that $\H(\target) \leq \states \actions \ln \left(\frac{1}{\delta} \right)$, we have
\begin{small}
\begin{align*}
    & \, \Regret(T | \pi_{\rm TS}) \\
    \leq & \,
    \tau^2 \sqrt{ 6  T \states \actions \ln \left(\frac{1}{\delta} \right) \ln \left(\frac{2 \states \actions}{\delta}\right)
    } 
    + \left[3 \delta + \sqrt{ 6 \max \left \{ 3, \ln \left(\frac{2}{\delta} \right)\right \} \delta \ln \left(\frac{2 \states \actions}{\delta}\right)}  \right] \tau^2 T \\
    = & \, \mathcal{O} \left( 
    \tau^2 \sqrt{\ln \left(\frac{1}{\delta} \right) \ln \left(\frac{\states \actions}{\delta} \right)} \left[\sqrt{\states \actions T } + T \sqrt{\delta} \right]
    \right).
\end{align*}
\end{small}
This concludes the proof.
%
\end{proof}

\chapter{Analysis of IDS in an Episodic Environment}
\label{ap:chain}


Consider the environment and agent described in Section \ref{se:chain}. Let $\proxy = \environment$ and let the epistemic state indicate the value of $r_{\tau-1}$ or that it has not been observed. The value-IDS agent in Section \ref{se:chain} selects actions by optimizing
\[ \min_{\nu \in \Delta_\actions} \frac{\E\left[V_*(S_t) - Q_*(S_t, \tilde{A}_t) | X_t\right]^2}{\I(\pi_*(\cdot|S_t); \tilde{A}_t, Q_*(S_t, \tilde{A}_t) |X_t \leftarrow X_t)} \]
where $\tilde{A}_t$ is drawn from $\nu$.

We will bound the $\tau$-information ratio. Since $\environment$ determines $\target = \pi_*$, the $\tau$-information ratio simplifies to
\[ \Gamma_{\tau, t} = \frac{\E[V_*(S_t) - Q_*(S_t, A_t)]^2}{(\H(\pi_* | P_t) - \H(\pi_* | P_{t+\tau}))/\tau}. \]
To do this, we will find a uniform bound $\overline{\Gamma}$ on the conditional information ratio,
\[ \tilde{\Gamma}_{\tau, t} = \frac{\E[V_*(S_t) - Q_*(S_t, A_t) | X_t]^2}{\E[\H(\pi_* | P_t \leftarrow P_t) - \H(\pi_* | P_{t+\tau} \leftarrow P_{t+\tau}) | X_t] / \tau} \equiv \frac{\Delta_t^2}{I_t}. \]
If $\tilde{\Gamma}_{\tau, t} \leq \overline{\Gamma}$ for all $t$, then
\begin{align*}
& \E[V_*(S_t) - Q_*(S_t, A_t)]^2 \\
\leq & \, \E\left[\E[V_*(S_t) - Q_*(S_t, A_t) | X_t]^2\right] \\
= & \, \E\left[ \tilde{\Gamma}_{\tau, t} \E[\H(\pi_* | P_t \leftarrow P_t) - \H(\pi_* | P_{t+\tau} \leftarrow P_{t+\tau}) | X_t] / \tau \right] \\
\leq & \, \overline{\Gamma} \left(\H(\pi_* | P_t) - \H(\pi_* | P_{t+\tau})\right) / \tau,
\end{align*}
and so $\Gamma_{\tau, t} \leq \overline{\Gamma}$.

Let's consider an episode where the agent is still uncertain about the environment. Otherwise, the agent is done learning and the conditional information ratio will be zero. Let us overload $\tilde{\Gamma}_s$ to denote the conditional $\tau$-step information ratio for convenience for $s = 0, \dots, \tau-2$, with the intention that $\tilde{\Gamma}_s = \tilde{\Gamma}_{\tau, t}$ given $S_t = s$ and $P_t = \mathrm{null}$. Similarly, we use $\Delta_s$ and $I_s$ to denote the corresponding regret and information gain. Further, let $\nu_{a, s}$ denote the probability that value-IDS selects action $a$ given $S_t = s$ and $P_t = \mathrm{null}$.

Let $\overline{r}_{\tau-2} = 0$ and $\overline{r}_s = \max\{r_{s+1}, \dots, r_{\tau-2}\}$ for $0 \leq s \leq \tau - 3$ denote the maximum exit rewards starting from state $s+1$ to state $\tau-2$. Then, for $0 \leq s \leq \tau-2$, we have
\begin{align*}
Q_*(s, 0) &= r_s, \\
Q_*(s, 1) &= \begin{cases}
1 & \text{w.p. } \frac{1}{2}, \\
\overline{r}_s & \text{w.p. } \frac{1}{2}.
\end{cases}
\end{align*}

Let $\Delta_{a, s} = \E[V_*(s) - Q_*(s, a)]$. We have for states $s = 0, \dots, \tau-2$
\[ \Delta_{0, s} = \tfrac{1}{2}(1 - r_s) + \tfrac{1}{2}( \overline{r}_s - r_s)_+, 
\text{ and } \Delta_{1, s} = \tfrac{1}{2}(r_s - \overline{r}_s)_+. \]
The information gain in the denominator of value-IDS is zero for action 0, and $1$ bit for action 1. Thus, value-IDS selects probabilities $(\nu_{0, s}, \nu_{1, s})$ that minimizes 
\[ \min_{\nu_{0, s}', \nu_{1, s}'} \frac{(\nu_{0, s}' \Delta_{0, s} + \nu_{1, s}' \Delta_{1, s})^2}{\nu_{1, s}'}. \]
Note that
\begin{align*}
\tfrac{1}{\nu_{1, s}'} (\nu_{0, s}' \Delta_{0, s} + \nu_{1, s}' \Delta_{1, s})^2
&= \tfrac{1}{\nu_{1, s}'} ((1 - \nu_{1, s}') \Delta_{0, s} + \nu_{1, s}' \Delta_{1, s})^2 \\
&= (\Delta_{1, s} - \Delta_{0, s})^2 \nu_{1, s}' + \tfrac{\Delta_{0, s}^2}{\nu_{1, s}'} + 2 \Delta_{0, s} (\Delta_{1, s} - \Delta_{0, s}). 
\end{align*}
Thus, the minimizer
\[ \nu_{1,s} = \min \left( \frac{\Delta_{0,s}}{(\Delta_{1,s} - \Delta_{0,s})_+}, 1 \right). \]
Also note that $\nu_{1,s} < 1$ if and only if
\[ 2 \Delta_{0,s} < \Delta_{1,s} \quad \Leftrightarrow \quad r_s > \frac{2 + \overline{r}_s}{3}. \]

Now, we show inductively that $\tilde{\Gamma}_s \leq \frac{\tau}{8}$ for all $0 \leq s \leq \tau-2$.

For the base case $s = \tau-2$, note that the average $\tau$-step information gain is equal to $\nu_{1,s} / \tau$. Thus, 
\[ \tilde{\Gamma}_s = \frac{\Delta_s^2}{\nu_{1,s} / \tau} = \begin{cases}
\tau (1 - r_s)(2r_s - \overline{r}_s - 1) & \nu_{1,s} < 1 \\
\tfrac{\tau}{4}\left((r_s - \overline{r}_s)_+\right)^2 & \nu_{1,s} = 1.
\end{cases} \]
For $s = \tau - 2$, $\overline{r}_s = 0$. We have $(1-r_s)(2r_s-1) \leq \frac{1}{8}$, and $\frac{1}{4}r_s^2 \leq \frac{1}{9}$ if $\nu_{1,s}=1$, since $\nu_{1,s}=1$ implies that $r_s \leq \frac{2}{3}$. Thus, the base case holds.

Our induction hypothesis is that $\tilde{\Gamma}_{s'} \leq \frac{\tau}{8}$ for all $s+1 \leq s' \leq \tau-2$. Now let's consider state $1 \leq s < \tau-2$. Let $\overline{s} = \min\{s': s < s' \leq \tau-2, \nu_{1, s'} < 1\}$ and if the set is empty, let $\overline{s} = \mathrm{null}$. We have three cases to consider. 
\begin{enumerate}[(i)]
\item If $\overline{s}$ is null, then $\nu_{1, s'} = 1$ for all $s' = s+1, \dots, \tau-2$. Thus, the average information gain $I_s = \nu_{1,s} / \tau$. Then, similar to the base case,
\[ \tilde{\Gamma}_s = \frac{\Delta_s^2}{\nu_{1,s} / \tau} = \begin{cases}
\tau (1 - r_s)(2r_s - \overline{r}_s - 1) & \nu_{1,s} < 1 \\
\tfrac{\tau}{4}\left((r_s - \overline{r}_s)_+\right)^2 & \nu_{1,s} = 1.
\end{cases} \]
Now, $(1 - r_s)(2r_s - \overline{r}_s - 1) \leq \frac{1}{8}(1-\overline{x}_s)^2 \leq \frac{1}{8}$. When $\nu_{1,s} = 1$, we have $\tfrac{1}{4}\left((r_s - \overline{r}_s)_+\right)^2 \leq \frac{(1-\overline{r}_s)^2}{9} \leq \frac{1}{9}$ since $\nu_{1,s} = 1$ implies that $r_s \leq \frac{2 + \overline{r}_s}{3}$. Therefore, we have $\tilde{\Gamma}_s \leq \frac{\tau}{8}$.

\item If $\overline{s}$ is not null and $\nu_{1, s} = 1$, then the average information gain $I_s = I_{\overline{s}}$. Further, we have
\[ \Delta_s = \Delta_{1, s} = \tfrac{1}{2}(r_s - \overline{r}_s), \]
and
\begin{align*}
\Delta_{\overline{s}} &= \nu_{0,\overline{s}} \Delta_{0,\overline{s}} + \nu_{1,\overline{s}} \Delta_{1,\overline{s}} \\
&= \left( 1 - \frac{\Delta_{0,\overline{s}}}{\Delta_{1,\overline{s}} - \Delta_{0,\overline{s}}} \right) \Delta_{0,\overline{s}} + \frac{\Delta_{0,\overline{s}}}{\Delta_{1,\overline{s}} - \Delta_{0,\overline{s}}} \Delta_{1,\overline{s}} \\
&= 2 \Delta_{0,\overline{s}} = 1 - r_{\overline{s}}.
\end{align*}
Since $\overline{r}_s \geq r_{\overline{s}}$ by definition, we have $\Delta_s \leq \frac{1}{2}\Delta_{\overline{s}}$. Therefore, $\tilde{\Gamma}_s \leq \frac{1}{4} \tilde{\Gamma}_{\overline{s}}$, and by the induction hypothesis, $\tilde{\Gamma}_s \leq \frac{\tau}{32}$.

\item If $\overline{s}$ is not null and $\nu_{1, s} < 1$, then the average information gain $I_s = \nu_{1,s} I_{\overline{s}}$. Since,
\[ \tilde{\Gamma}_s = \frac{\Delta_s^2}{I_s}
= \frac{\Delta_s^2}{\nu_{1, s} I_{\overline{s}}}
= \frac{\Delta_s^2}{\nu_{1, s} \Delta_{\overline{s}}^2} \Gamma_{\overline{s}}, \]
we will upper bound $\frac{\Delta_s^2}{\nu_{1, s} \Delta_{\overline{s}}^2}$ and then apply the inductive hypothesis. 
We have
\[ \frac{\Delta_s^2}{\nu_{1, s} \Delta_{\overline{s}}^2}
= \frac{(1 - r_s)(2r_s - \overline{r}_s - 1)}{(1 - r_{\overline{s}})^2}
\leq \frac{(1 - r_s)(2r_s - \overline{r}_s - 1)}{(1 - \overline{r}_s)^2}. \]
Write $y = 1 - \overline{r}_s$ and define $\alpha = \frac{1-r_s}{1-\overline{r}_s}$. We have
\[ \frac{(1 - r_s)(2r_s - \overline{r}_s - 1)}{(1 - \overline{r}_s)^2}
= \frac{\alpha y (y - 2 \alpha y)}{y^2} = \alpha(1-2\alpha) \leq \frac{1}{8}. \]
Therefore, $\frac{\Delta_s^2}{\nu_{1, s} \Delta_{\overline{s}}^2} \leq \frac{1}{8}$, and $\tilde{\Gamma}_s \leq \frac{1}{8} \tilde{\Gamma}_{\overline{s}}$. By the induction hypothesis, $\tilde{\Gamma}_s \leq \frac{\tau}{64}$.

\end{enumerate}

Therefore, $\tilde{\Gamma}_s \leq \frac{\tau}{8}$ for all $0 \leq s \leq \tau-2$. This implies that for any exit rewards $r_0, \dots, r_{\tau-2} \in [0, 1)$, the $\tau$-information ratio $\Gamma_{\tau, t} \leq \frac{\tau}{8}$ for all $t$. By Theorem \ref{th:regret-bound}, this implies a regret bound
\[ \mathrm{Regret}(T) \leq \sqrt{\tfrac{1}{8} \tau T \H(\pi_*) } = \sqrt{\tfrac{1}{8} \tau T}. \]

\chapter{Convexity and Support of Value-IDS}
\label{ap:support-cardinality}

The following result and proof are based on an analysis from \citep{russo2014learning,russo2018learning}.
\begin{theorem}
\label{th:support-cardinality}
For all vectors $\alpha,\beta \in \Re^N$ and functions $\psi:\Re^N\mapsto\Re$ of the form
$\psi(\nu) = (\nu^\top \alpha)^2 / \nu^\top \beta$,
$\psi$ is convex on $\{\nu \in \Re^N: \nu^\top \beta > 0\}$, and there exists a vector $\nu^* \in \Delta_N$ such that $|\{n:\nu^*_n > 0\}| \leq 2$ and $\psi(\nu^*) = \min_{\nu \in \Delta_N} \psi(\nu)$.
\end{theorem}
\begin{proof}
Consider a function $\phi:\Re^2\mapsto\Re$ given by $\phi(x) = x_1^2 / x_2$.  We have
$$\nabla_x \phi(x) = \left[\begin{array}{c} 2 x_1/x_2 \\ - x_1^2 / x_2^2\end{array}\right]
\qquad \text{and} \qquad
\nabla^2_x \phi(x) = \left[\begin{array}{cc} 2/x_2 & -2 x_1 / x_2^2\\ - 2 x_1 / x_2^2 & 2 x_1^2 / x_2^3\end{array}\right].$$
If $x_2 > 0$ then the ${\rm tr}(\nabla^2_x \phi(x)) = 4x_1^2/x_2^4 > 0$ and $|\nabla^2_x \phi(x)| = 0$.  It follows that $\phi$ is convex.  For any $\gamma \in [0,1]$ and $\nu,\overline{\nu} \in \Re^N$ such that $\nu^\top \beta > 0$ and $\overline{\nu}^\top \beta > 0$, 
\begin{align*}
\psi(\gamma \nu + (1-\gamma) \overline{\nu}) 
=& \phi\left(\left[\begin{array}{c} (\gamma \nu + (1-\gamma) \overline{\nu})^\top \alpha \\ (\gamma \nu + (1-\gamma) \overline{\nu})^\top b\end{array}\right]\right) \\
=& \phi\left(\gamma \left[\begin{array}{c} \nu^\top \alpha \\ \nu^\top \beta\end{array}\right]
+ (1-\gamma) \left[\begin{array}{c} \overline{\nu}^\top \alpha \\ \overline{\nu}^\top b\end{array}\right] \right) \\
\leq& \gamma \phi\left(\left[\begin{array}{c} \nu^\top \alpha \\ \nu^\top \beta\end{array}\right]\right)
+ (1-\gamma) \phi\left(\left[\begin{array}{c} \overline{\nu}^\top \alpha \\ \overline{\nu}^\top \beta\end{array}\right] \right) \\
=& \gamma \psi(\nu) + (1-\gamma) \psi(\overline{\nu}).
\end{align*}
Hence, $\psi$ is convex.

Let $\nu^* \in \argmin_{\nu \in \Delta_N} \psi(\nu)$ and $\zeta(\nu) = (\nu^\top \alpha)^2 - \psi(\nu^*) \nu^\top \beta$.  Note that, for all $\nu \in \Delta_N$,
$$\zeta(\nu) = (\nu^\top \alpha)^2 - \psi(\nu^*) \nu^\top \beta \geq (\nu^\top \alpha)^2 - \psi(\nu) \nu^\top \beta = 0,$$
and $\zeta(\nu^*) = 0$, implying $\argmin_{\nu \in \Delta_N} \psi(\nu) \subseteq \argmin_{\nu \in \Delta_N} \zeta(\nu)$.
Let $\overline{\nu} \in \argmin_{\nu \in \Delta_N} \zeta(\nu)$.  Then, $\zeta(\overline{\nu}) = 0$ and
$$\psi(\overline{\nu}) = \frac{(\overline{\nu}^\top \alpha)^2}{\overline{\nu}^\top \beta} = \frac{\zeta(\overline{\nu}) + \psi(\nu^*) \overline{\nu}^\top \beta}{\overline{\nu}^\top \beta}
= \psi^*,$$
implying $\argmin_{\nu \in \Delta_N} \psi(\nu) \supseteq \argmin_{\nu \in \Delta_N} \zeta(\nu)$.  It follows that
\[ \argmin_{\nu \in \Delta_N} \psi(\nu) = \argmin_{\nu \in \Delta_N} \zeta(\nu). \]

To complete the proof, we will establish that there exists $\nu^\dagger \in \argmin_{\nu \in \Delta_N} \zeta(\nu)$ with at most two positive components.  If $\overline{\nu}$ has two or fewer positive components, we are done, we will treat the case where $\overline{\nu}$ has more than two positive components.  Note that $\nabla_\nu \zeta(\nu) = 2 \alpha \nu^\top \alpha - \psi(\nu^*) \beta$.
By the KKT conditions, $\overline{\nu} \in \argmin_{\nu \in \Delta_N} \zeta(\nu)$ if and only if there exists a vector $d \geq 0$ and a scalar $c$ such that
$$2 \alpha \overline{\nu}^\top \alpha - \psi(\nu^*) \beta - d + c \1 = 0 \qquad \text{and} \qquad d^\top \overline{\nu} = 0.$$
These conditions can equivalently be written as 
$$\overline{\nu}^\top (2 \alpha \overline{\nu}^\top \alpha - \psi(\nu^*) \beta + c \1) = 0.$$
Let $\overline{n} = \argmax_{n: \overline{\nu}_n > 0} \overline{\nu}_n$ and $\underline{n} = \argmin_{n: \overline{\nu}_n > 0} \overline{\nu}_n$.  Let $\gamma \in [0,1]$ be such that
$$\overline{\nu}^\top \alpha = \gamma \overline{\nu}_{\overline{n}} \alpha_{\overline{n}} + (1-\gamma) \overline{\nu}_{\underline{n}} \alpha_{\underline{n}},$$
and let $\tilde{\nu} = \gamma \overline{\nu}_{\overline{n}} \1_{\overline{n}} + (1-\gamma) \overline{\nu}_{\underline{n}} \1_{\underline{n}}$.  Note that $\tilde{\nu}^\top \alpha = \overline{\nu}^\top a$ and 
$$\tilde{\nu}^\top (2 \alpha \tilde{\nu}^\top \alpha - \psi(\nu^*) \beta + c \1) = 0,$$
since ${\rm support}(\tilde{\nu}) \subset {\rm support}(\overline{\nu})$.
It follows that $\tilde{\nu} \in \argmin_{\nu \in \Delta_N} \zeta(\nu)$.  Since $\tilde{\nu}$ has two positive components, the result follows.
\end{proof}

This result implies that the objective minimized by each version of value-IDS is convex and that the minimum can be attained by randomizing between at most two actions.  To see why, consider the objective of the basic version:
$$\min_{\nu \in \Delta_\actions} \frac{\E\left[V_{\pi_\target}(H_t) - Q_{\pi_\target}(H_t, \tilde{A}_t) | X_t\right]^2}{\I(\target; \tilde{A}_t, \tilde{Y}_{t+1} | X_t \leftarrow X_t)}.$$
Shortfall and information gain can be rewritten as
\begin{align*}
& \E\left[V_{\pi_\target}(H_t) - Q_{\pi_\target}(H_t, \tilde{A}_t) | X_t\right] \\
= & \sum_{a \in \actions} \nu(a) \E\left[V_{\pi_\target}(H_t) - Q_{\pi_\target}(H_t, \tilde{A}_t) | X_t, \tilde{A}_t = a\right],
\end{align*}
and
\begin{align*}
\I(\target; \tilde{A}_t, \tilde{Y}_{t+1} | X_t \leftarrow X_t)
\overset{(a)}{=}& \, \I(\target; \tilde{Y}_{t+1} | X_t \leftarrow X_t, \tilde{A}_t)  + \I(\target; \tilde{A}_t | X_t \leftarrow X_t) \\
\overset{(b)}{=}& \, \I(\target; \tilde{Y}_{t+1} | X_t \leftarrow X_t, \tilde{A}_t) \\
=& \, \sum_{a \in \actions} \nu(a) \I(\target; \tilde{Y}_{t+1} | X_t \leftarrow X_t, \tilde{A}_t = a),
\end{align*}
where (a) follows from the chain rule of mutual information and (b) follows from the fact that $\target \perp \tilde{A}_t | X_t$.  Without loss of generality, let $\actions = \{1,\ldots, |\actions|\}$.  Then, letting
$$\alpha_a = \E\left[V_{\pi_\target}(H_t) - Q_{\pi_\target}(H_t, \tilde{A}_t) | X_t, \tilde{A}_t = a\right],$$
and
$$\beta_a = \I(\target; \tilde{Y}_{t+1} | X_t \leftarrow X_t, \tilde{A}_t = a),$$
Theorem \ref{th:support-cardinality} confirms our assertion about convexity of the value-IDS objective and existence of a 2-sparse optimal solution.

\chapter{Relation Between Information Gain and Variance}
\label{ap:information-variance}

All the mutual information terms are in nats in this section.

\begin{lemma}
Let $X_t = (Z_t, S_t, P_t)$ be the agent state at timestep $t$, $\tilde{A}_t$ be a random action sampled from some distribution $\nu$ that only depends on $X_t$, $Q_{\dagger}$ be a vector of GVFs with dimension $n$, $\tilde{Y}_{t+1} = Q_{\dagger}(H_t, \tilde A_t) + W_{t+1}$ be a pseudo-observation of $Q_{\dagger}$ where $W_{t+1}$ is some zero-mean random noise, and $\pi_\target$ be a target policy. If each component of $Q_{\dagger}$ has a span of at most $M_1$ and each component of $W_{t+1}$ has a span of at most $M_2$, 
\begin{align*}
& \I(\pi_\target(\cdot|S_t); \tilde{A}_t, \tilde{Y}_{t+1} | X_t \leftarrow X_t) \\
& \geq \frac{2}{n(M_1+M_2)^2} \E \left[ {\rm tr}\left({\rm Cov}\left[ \E\left[Q_\dagger(H_t, \tilde A_t) | X_t, \tilde{A}_t, \pi_\target(\cdot|S_t)\right] \big| X_t, \tilde{A}_t \right]\right) | X_t \right].
\end{align*}
\end{lemma}

\begin{proof}
We have
\begin{align*}
& \, \I(\pi_\target(\cdot|S_t); \tilde{A}_t, \tilde{Y}_{t+1} | X_t \leftarrow X_t) \\
&= \, \I(\pi_\target(\cdot|S_t); \tilde{Y}_{t+1} | X_t \leftarrow X_t, \tilde{A}_t) + \I(\pi_\target(\cdot|S_t); \tilde{A}_t, | X_t \leftarrow X_t) \\
&\overset{(a)}{=} \, \I(\pi_\target(\cdot|S_t); \tilde{Y}_{t+1} | X_t \leftarrow X_t, \tilde{A}_t) \\
&= \, \sum_{\tilde a \in \actions} \nu(\tilde a) \I(\pi_\target(\cdot|S_t); \tilde{Y}_{t+1} | X_t \leftarrow X_t, \tilde{A}_t = \tilde a)
\end{align*}
where $(a)$ follows from $\pi_\target(\cdot | S_t) \perp \tilde{A}_t | X_t$. Then,
\begin{align*}
& n  \I(\pi_\target(\cdot|S_t); \tilde{Y}_{t+1} | X_t \leftarrow X_t, \tilde{A}_t = a) \\
& \overset{(a)}{\geq} \, \sum_{i=1}^n\I(\pi_\target(\cdot|S_t); \tilde{Y}_{t+1, i} |X_t, \tilde{A}_t = a) \\
& = \, \sum_{i=1}^n \E \bigg[ \KL\bigg(\Pr(\tilde{Y}_{t+1,i} \in \cdot|X_t, \pi_\target(\cdot|S_t), \tilde{A}_t =  a) \\
& \hspace{5cm} \| \, \Pr(\tilde{Y}_{t+1,i} \in \cdot |X_t, \tilde{A}_t = a)\bigg) | X_t, \tilde{A}_t = a \bigg] \\
& \overset{(b)}{\geq} \, \frac{2}{(M_1+M_2)^2} \sum_{i=1}^n \E \bigg[ \bigg(\E\left[\tilde{Y}_{t+1, i}|X_t, \pi_\target(\cdot|S_t), \tilde{A}_t = a\right] \\
& \hspace{5.5cm} - \E\left[\tilde{Y}_{t+1, i}|X_t, \tilde{A}_t = a\right]\bigg)^2 | X_t, \tilde{A}_t = a \bigg] \\
& = \, \frac{2}{(M_1+M_2)^2} \sum_{i=1}^n \Var\left[\E\left[\tilde{Y}_{t+1,i}| X_t, \tilde{A}_t = a, \pi_\target(\cdot|S_t) \right] | X_t, \tilde{A}_t = a\right],
\end{align*}
where (a) follows from chain rule and mutual information being non-negative, and (b) follows from Pinsker's inequality with span of each component of $\tilde Y_{t+1}$ being at most $M_1+M_2$. Since each component of $Q_{\dagger}(\cdot, \cdot)$ has a span of atmost $M_1$ and each component of $W_{t+1}$ has a span of at most $M_2$, each component of $\tilde{Y}_{t+1}=Q_{\dagger}(H_t, \tilde A_t) + W_{t+1}$ has a span of at most $M_1+M_2$.

\begin{align*}
& \I(\pi_\target(\cdot|S_t);  \tilde{A}_t, \tilde{Y}_{t+1} | X_t \leftarrow X_t)\\
& = \, \sum_{\tilde a \in \actions} \nu(\tilde a) \I(\pi_\target(\cdot|S_t); \tilde{Y}_{t+1} | X_t \leftarrow X_t, \tilde{A}_t = \tilde a) \\
& \geq \, \frac{2}{n(M_1+M_2)^2}   \sum_{\tilde a \in \actions} \nu(\tilde a)  \sum_{i=1}^n\Var\left[\E\left[\tilde{Y}_{t+1,i}| X_t, \tilde{A}_t =\tilde a, \pi_\target(\cdot|S_t)\right] | X_t, \tilde{A}_t =\tilde a\right] \\  
& = \, \frac{2}{n(M_1+M_2)^2} \sum_{i=1}^n \E \left[ \Var\left[\E\left[\tilde{Y}_{t+1,i}| X_t, \tilde{A}_t, \pi_\target(\cdot| S_t)\right] | X_t, \tilde{A}_t \right] \right] \\
& = \, \frac{2}{n(M_1+M_2)^2} \sum_{i=1}^n \E \left[ \Var\left[\E\left[Q_\dagger(H_t, \tilde A_t)_{i}| X_t, \tilde{A}_t, \pi_\target(\cdot| S_t)\right] | X_t, \tilde{A}_t \right] \right] \\
& = \, \frac{2}{n(M_1+M_2)^2} \E \left[ {\rm tr}\left({\rm Cov}\left[\E\left[Q_\dagger(H_t, \tilde A_t) | X_t, \tilde{A}_t, \pi_\target(\cdot|S_t)\right] | X_t, \tilde{A}_t \right]\right) \big| X_t \right].
\end{align*}

\end{proof}

\begin{lemma}
Let $X_t = (Z_t, S_t, P_t)$ be the agent state at timestep $t$, $\tilde{A}_t$ be a random action sampled from some distribution $\nu$ over actions, $Q_{\dagger}$ be a vector of GVFs with dimension $n$, and $\tilde{Y}_{t+1} = Q_{\dagger}(H_t, \tilde A_t) + W_{t+1}$ be a pseudo-observation of $Q_{\dagger}$ where $W_{t+1}$ is some zero-mean random noise. If each component of $Q_{\dagger}$ has a span of at most $M_1$ and each component of $W_{t+1}$ has a span of at most $M_2$, 
\begin{multline*}
\I(Q_\dagger(H_t, \tilde{A}_t); \tilde{A}_t, \tilde{Y}_{t+1} | X_t \leftarrow X_t) \\
\geq \frac{2}{n(M_1 + M_2)^2} \E \left[ {\rm tr}\left({\rm Cov}\left[Q_\dagger(H_t, \tilde A_t) | X_t, \tilde{A}_t \right]\right) | X_t \right].
\end{multline*}
\end{lemma}

\begin{proof}
We have
\begin{align*}
& \I(Q_\dagger(H_t, \tilde{A}_t); \tilde{A}_t, \tilde{Y}_{t+1} | X_t \leftarrow X_t) \\
& = \, \I(Q_\dagger(H_t, \tilde{A}_t); \tilde{A}_t | X_t \leftarrow X_t) + \I(Q_\dagger(H_t, \tilde{A}_t); \tilde{Y}_{t+1} | X_t \leftarrow X_t,  \tilde{A}_t) \\
& \overset{(a)}{\geq} \, \I(Q_\dagger(H_t, \tilde{A}_t); \tilde{Y}_{t+1} | X_t \leftarrow X_t,  \tilde{A}_t) \\
& = \, \sum_{a \in \actions}  \nu(a) (\I(Q_\dagger(H_t, \tilde{A}_t); \tilde{Y}_{t+1} | X_t \leftarrow X_t,  \tilde{A}_t = a),
\end{align*}
where (a) follows from mutual information being non-negative. 
Then,
\begin{align*}
& n \I(Q_\dagger(H_t, \tilde{A}_t) ; \tilde{Y}_{t+1}  | X_t \leftarrow X_t, \tilde{A}_t = a) \\
& \overset{(a)}{\geq} \, \sum_{i=1}^n \I(Q_\dagger(H_t, \tilde{A}_t)_i; \tilde{Y}_{t+1,i}  | X_t \leftarrow X_t, \tilde{A}_t = a) \\
& = \, \sum_{i=1}^n \E \Big[ \KL \Big( \Pr(\tilde{Y}_{t+1, i}| Q_\dagger(H_t, \tilde{A}_t)_i, X_t, \tilde{A_t}=a ) \\
& \hspace{5.5cm} \| \Pr(\tilde{Y}_{t+1, i}| X_t, \tilde{A_t}=a)  \Big) | X_t, \tilde{A_t}=a \Big] \\
& \overset{(b)}{\geq} \, \frac{2}{(M_1 + M_2)^2} \sum_{i=1}^n \E \Big[ \Big( \E \left[ \tilde{Y}_{t+1, i}| Q_\dagger(H_t, \tilde{A}_t)_i , X_t, \tilde{A_t}=a \right] \\
& \hspace{5cm} - \E \left[ \tilde{Y}_{t+1, i}| X_t, \tilde{A_t}=a \right] \Big)^2 | X_t, \tilde{A_t}=a \Big] \\
& = \, \frac{2}{(M_1 + M_2)^2} \sum_{i=1}^n\Var \left[ Q_\dagger(H_t, \tilde{A}_t)_i |  X_t, \tilde{A_t}=a \right],
\end{align*}
where $(a)$ follows from the fact that mutual information between two random vectors is not less than mutual information between any of there components, and  (b) follows from Pinsker's inequality with span of each component of $\tilde Y_{t+1}$ being at most $M_1+M_2$. Since each component of $Q_{\dagger}(\cdot, \cdot)$ has a span of atmost $M_1$ and each component of $W_{t+1}$ has a span of at most $M_2$, each component of $\tilde{Y}_{t+1}=Q_{\dagger}(H_t, \tilde A_t) + W_{t+1}$ has a span of at most $M_1+M_2$.
Therefore,
\begin{align*}
& \I(Q_\dagger(H_t, \tilde{A}_t); \tilde{A}_t, \tilde{Y}_{t+1} | X_t \leftarrow X_t) \\
& \geq \, \frac{2}{n(M_1 + M_2)^2} \sum_{a \in \actions} \nu(a) \sum_{i=1}^n\Var \left[ Q_\dagger(H_t, \tilde{A}_t)_i |  X_t, \tilde{A_t}=a \right] \\
& = \, \frac{2}{n(M_1 + M_2)^2} \sum_{i=1}^n \E \left[ \Var \left[ Q_\dagger(H_t, \tilde{A}_t)_i |  X_t, \tilde{A}_t \right] | X_t \right] \\
& = \, \frac{2}{n(M_1 + M_2)^2} \E \left[ {\rm tr}\left({\rm Cov}\left[Q_\dagger(H_t, \tilde A_t) | X_t, \tilde{A}_t \right]\right) | X_t \right].
\end{align*} 

\end{proof}


\chapter{Implementation and Computation}
\label{app:computation}

This section provides the details and parameters for our computational experiments.
We specify the agent through the agent state, and the action selection policy $\pi(\cdot | X_t).$
For the most part, these are both explained in the main body of our paper.
However, our next subsections expand on these implementational details.

\section{Agent state update}

Agent state dynamics are determined by the update rules ($f_{\rm algo}$, $f_{alea}$, $f_{epis}$) introduced in Equations \eqref{eq:algorithmic-state-dynamics}, \eqref{eq:aleatoric-state-dynamics}, and \eqref{eq:epistemic-state-dynamics}.
For the most part, these updates are already described in Section \ref{se:sample_ids_alg}, but we use this appendix to spell out some more of the details, particularly in regards to the epistemic update.

\paragraph{Algorithmic state}
The algorithmic state is null $Z_t = \emptyset$ for both IDS and $\epsilon$-greedy action selection.
For Thompson sampling the algorithmic state $Z_t = Z_{t_k}$ is resampled uniformly from the set of epistemic indices for the relevant ENN at the end of each episode.
This fully specifies the algorithmic update for all our experiments, and so we will not address this further.

\paragraph{Aleatoric state} All of the agents considered are `feed-forward' variants of DQN with aleatoric state given by the current observation $S_t = O_t$.
This fully specifies the aleatoric update for all our experiments, and so we will not address this further.

\paragraph{Epistemic state} All agents' epistemic state are given $P_t = (\theta_t, B_t)$.
Here $\theta_t$ are parameters of an ENN $f$ and $B_t$ is a FIFO experience replay buffer.
Further, all of our experiments are specialized to the case where $f$ represents a (potentially general) value function over $\actions$ finite actions.
In all of our experiments we learn via minibatch SGD, according to Equation \eqref{eq:average_enn_learn}.

For each experiment, we can therefore fully specify the epistemic update through the ENN $f$, the loss function \mbox{$\ell \left(\theta; f, \theta^-, z, (s, a, r, s')\right) \rightarrow \R$}, the initial parameters $\theta_0$, the SGD update procedure, $n_{\rm batch}$, $n_{\rm index}$.
For the replay buffer in each experiment we set a minimum replay size equal to $n_{\rm batch}$ and a maximum replay size of 10,000.

\section{Action selection}

Action selection via $\epsilon$-greedy \citep{mnih-atari-2013} and Thompson sampling \citep{osband2019deep} are relatively straightforward, and have been covered extensively in prior work.
In this subsection we expand on the sample-based implementation of IDS that approximates equations \eqref{eq:variance-ids-target-policy} and \eqref{eq:variance-ids-target-gvfs}.
Note that, since we know the solution has support on at most two actions in $\Ac$, we can approximately optimize the objective by effectively searching over a probability grid for each \textit{pair} of actions in $\Ac$.
In all of our experiments we search with granularity of $\frac{1}{100}$ in each action probability.

In all of our experiments, we use a simple sample-based approach to approximating the shortfall and variance in equations \eqref{eq:variance-ids-target-policy} and \eqref{eq:variance-ids-target-gvfs}.
The first step is to generate $n_{\rm IDS}$ samples from the ENN given the agents epistemic state $P_t$.
The expected shortfall is then calculated simply as the average of the shortfall for each action, for each sample. 
The variance-based information gain is approximated through the sample variance, as detailed below.
\begin{enumerate}[(a)]
\item {\bf Learning Target = Optimal Action.}
We use the same $n_{\rm IDS}$ samples from the ENN to approximate the information gain in Equation \eqref{eq:variance-ids-target-policy}.
In our experiments we only perform experiments with action-value learning targets, and so we can simplify the exposition significantly.
Let $Q_1,..,Q_{n_{\rm IDS}} \in \R^\Ac$ be samples of the action value generated by the agent,
 $\Qc_a := \{Q_n \mid a \in \argmax_\alpha Q_n(S_t, \alpha) \}$,
 $\overline{Q}_a := \frac{1}{|\Qc_a|} \sum_{q \in \Qc_a} q $ and $\overline{Q} = \frac{1}{n_{\rm IDS}} \sum_{n=1}^{n_{\rm IDS}} Q_n$.
Then the approximate information gain used in our experiments can be written:
\begin{multline}
    \label{eq:approx_action_variance}
    \E \left[ \mathrm{tr}\left(\mathrm{Cov}\left[\E[Q_\dagger(H_t, \tilde{A}_t) | X_t, \tilde{A}_t, \pi_\target(\cdot|S_t)])] | X_t, \tilde{A}_t \right]\right) | X_t \right] \\
    \simeq
    \frac{1}{n_{\rm IDS}} \sum_{a=1}^{n_{\Ac}} |\Qc_a| \left( \overline{Q}_a - \overline{Q} \right)^2.
\end{multline}

\item {\bf Learning Target = GVF.}
In a similar manner, we reuse the same $n_{\rm IDS}$ samples from the ENN to approximate the information gain in Equation \eqref{eq:variance-ids-target-gvfs}.
For a problem with general value function $Q_\dagger \in \R^d$, let $Q_1,..,Q_{n_{\rm IDS}} \in \R^\Ac$ be samples of the action value generated by the agent and $\overline{Q}$ the sample mean.
The approximate information gain used in our experiments is then:
\begin{multline}
    \label{eq:approx_gvf_variance}
    \E \left[ \mathrm{tr}\left(\mathrm{Cov}\left[Q_\dagger(H_t, \tilde{A}_t) | X_t, \tilde{A}_t \right]\right) | X_t \right] \\
    \simeq
    \frac{1}{n_{\rm IDS}} \sum_{i=1}^{d} \sum_{j=1}^{n_{\rm IDS}} \left(Q_{i,j} - \overline{Q}_i\right)^2.
\end{multline}
\end{enumerate}

\section{Parameters for each experiment}

In this section we list the settings used to generate the results in Section \ref{se:computation}.
It is our intention to provide some elements of our agents and code for opensource.

\paragraph{7.2.1, 7.3.1, and \bsuite\  IDS implementation}
We use the exact same settings for the experiments in these sections:
\begin{itemize}
    \item ENN = ensemble of 20 50-50-MLPs with matched prior functions initialized according to JAX standards \citep{osband2018randomized}.
    \vspace{-1mm}
    \item Loss $\ell = \ell^{Q, \gamma}$ (Equation \eqref{eq:q_learning_enn}) with $\gamma=0.99$.
    \vspace{-1mm}
    \item Action selection via Equation \eqref{eq:approx_action_variance} with $n_{\rm IDS}=40$.
    \vspace{-1mm}
    \item SGD update according with ADAM with learning rate $\alpha=0.001$. 
    \vspace{-1mm}
    \item $n_{\rm batch}=128$ for RL tasks and $n_{\rm batch}=1$ for bandit task.
\end{itemize}

\paragraph{7.3.2 Sparse bandit}
Same settings as 7.2.1 except for the ENN and loss function designed to encode the prior knowledge:
\begin{itemize}
    \item ENN = ensemble of 20 logits over $N$ possible rewarding arms.
    \vspace{-1mm}
    \item Loss $\ell$ of the cross-entropy on the posterior probabilty of the observation given the rewarding arm.
    \vspace{-1mm}
    \item Action selection via Equation \eqref{eq:approx_action_variance} with $n_{\rm IDS}=40$, given the knowledge of how to convert logits to associated action values.
    \vspace{-1mm}
    \item Vanilla SGD update with learning rate $\alpha=0.1$. 
\end{itemize}

\paragraph{7.4 Variance-IDS with general value functions}
These settings are almost identical to 7.3.2, but using a different action selection and with specialized ENNs:
\begin{itemize}
    \item ENN = ensemble of 20 logits over $N$ possible rewarding arms (rewarding states in the chain problem).
    \vspace{-1mm}
    \item Loss $\ell$ of the cross-entropy on the posterior probabilty of the observation given the ENN logits.
    \vspace{-1mm}
    \item Action selection via Equation \eqref{eq:approx_gvf_variance} with $n_{\rm IDS}=40$, given the knowledge of how to convert logits to associated action values.
    \vspace{-1mm}
    \item Vanilla SGD update with learning rate $\alpha=0.1$. 
\end{itemize}

%
%
%
%

\newpage
\onecolumn

\ifx\newgeometry\undefined\else
\newgeometry{top=20mm, bottom=20mm, left=20mm, right=20mm}
\fi

%
\bsuitetitle{Information Directed Sampling}
\label{app:bsuite-report}
\bsuiteabstract

%
\vspace{-2mm}
\subsection{Agent Definition}
\label{app:bsuite-agents}

We compare the performance of three agents as outlined in Section \ref{se:sample_ids_alg}.
In each case, the agents learn with an ensemble ENN formed of 20 50-50-MLPs and identical learning rules.
The only difference is in the action selection `planner`:
\begin{itemize}
    \vspace{-2mm}
    \item \textbf{egreedy}: $\epsilon\hspace{-1mm}=\hspace{-1mm}5\%$ greedy action selection, essentially DQN \citep{mnih-atari-2013}.
    \vspace{-2mm}
    \item \textbf{TS}: Thompson Sampling, aka \textit{bootstrapped DQN} \citep{osband2016deep}.
    \vspace{-2mm}
    \item \textbf{IDS}: Information Directed Sampling, with 40 samples per step.
\end{itemize}

\vspace{-2mm}
\subsection{Summary Scores}
\label{app:bsuite-scores}

Each \bsuite\ experiment outputs a summary score in [0,1].
We aggregate these scores by according to key experiment type, according to the standard analysis notebook.
A detailed analysis of each of these experiments will be released post review.

\begin{figure}[ht]
  \centering
  \includegraphics[width=\textwidth,height=60mm,keepaspectratio]{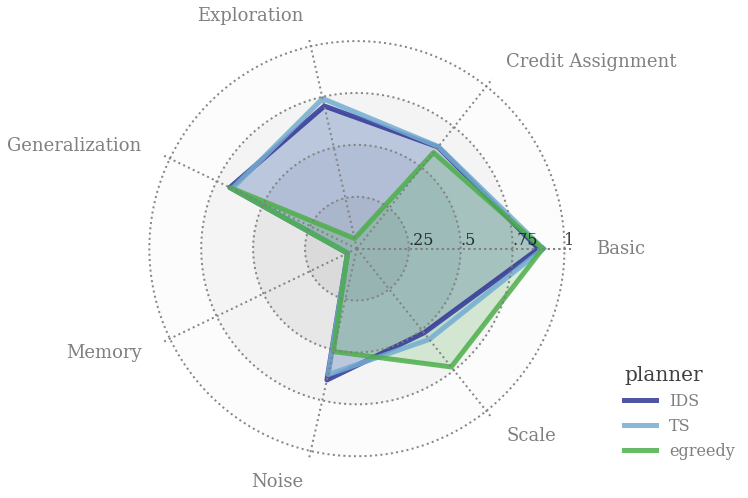}
  \captionof{figure}{Snapshot of agent behaviour.}
  \label{fig:radar}
\end{figure}

\begin{figure}[ht]
  \centering
  \includegraphics[width=\textwidth,height=60mm,keepaspectratio]{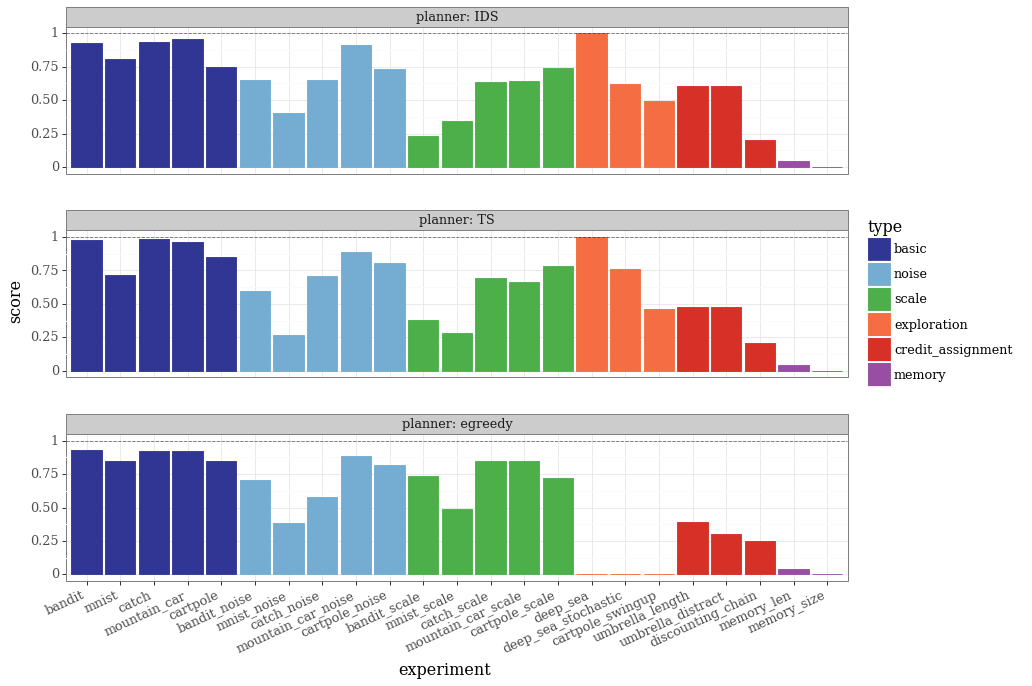}
  \captionof{figure}{Score for each \bsuite\ experiment.}
  \label{fig:bar}
\end{figure}

\subsection{Results Commentary}
\label{app:bsuite-commentary}

These results show a strong signal that action selection via IDS and TS can greatly outperform that of $\epsilon$-greedy in domains where exploration is crucial.
At least in the current collection of \bsuite\ tasks, IDS and TS perform similarly overall.
We also see some evidence that $\epsilon$-greedy action selection is more robust to changes in scale, but less robust to noise.

\backmatter  

\printbibliography

\end{document}